%% file: main.tex
\documentclass{article}

\usepackage{microtype}
\usepackage{graphicx}
\usepackage{subcaption}
\usepackage{booktabs} % for professional tables

\usepackage{hyperref}

\usepackage[accepted]{icml2023}

\usepackage{amsmath}
\usepackage{amssymb}
\usepackage{mathtools}
\usepackage{amsthm}

\usepackage[capitalize,noabbrev]{cleveref}

\usepackage{dsfont}
\input{math_commands}

\usepackage{algorithmic} 
\usepackage[algo2e]{algorithm2e} 
\SetKwComment{Comment}{$\triangleright$\ }{}  % algorithmicx style comments
\SetCommentSty{itshape}
\DontPrintSemicolon

\usepackage[toc,page,header]{appendix}
\usepackage{minitoc}
\usepackage{etoc}
\usepackage{comment}

\theoremstyle{plain}
\newtheorem{theorem}{Theorem}[section]
\newtheorem{proposition}[theorem]{Proposition}
\newtheorem{lemma}[theorem]{Lemma}
\newtheorem{corollary}[theorem]{Corollary}
\theoremstyle{definition}

\newtheorem{assumption}[theorem]{Assumption}
\theoremstyle{remark}

\newtheorem{setup}{Setup} 
\newtheorem{example}{Example}

\usepackage[textsize=tiny]{todonotes}

\icmltitlerunning{Model Transferability with Responsive Decision Subjects}

\begin{document}

%%%%%%%%%%%%%%%%%%%%%%%%%%%%%%
% ZT: for minitoc
%%%%%%%%%%%%%%%%%%%%%%%%%%%%%%
\doparttoc % tell to minitoc to generate a toc for the parts
\faketableofcontents % run a fake tableofcontents command for the partocs

\twocolumn[
\icmltitle{Model Transferability with Responsive Decision Subjects}



\icmlsetsymbol{equal}{*}

\begin{icmlauthorlist}
\icmlauthor{Yatong Chen}{ucsc}
\icmlauthor{Zeyu Tang}{cmu}
\icmlauthor{Kun Zhang}{cmu,mbz}
\icmlauthor{Yang Liu}{ucsc,br}
\end{icmlauthorlist}

\icmlaffiliation{ucsc}{Department of Computer Science and Engineering, University of California, Santa Cruz, California, United States.}
\icmlaffiliation{br}{ByteDance Research}
\icmlaffiliation{cmu}{Department of Philosophy, Carnegie Mellon University, Pennsylvania, United States.}
\icmlaffiliation{mbz}{Mohamed bin Zayed University of Artificial Intelligence, Abu Dhabi, United Arab Emirates.}

\icmlcorrespondingauthor{Yang Liu}{yangliu@ucsc.edu}

\icmlkeywords{Machine Learning, ICML}

\vskip 0.3in
]


\printAffiliationsAndNotice{}  % leave blank if no need to mention equal contribution

\begin{abstract} 
    Given an algorithmic predictor that is accurate on some source population consisting of strategic human decision subjects, will it remain accurate if the population \emph{respond} to it?
    In our setting, an agent or a user corresponds to a sample $(X,Y)$ drawn from a distribution $\D$ and will face a model $h$ and its classification result $h(X)$. Agents can modify $X$ to adapt to $h$, which will incur a distribution shift on $(X,Y)$.  
    Our formulation is motivated by applications where the deployed machine learning models are subjected to human agents, and will ultimately face \emph{responsive} and \emph{interactive} data distributions. 
    We formalize the discussions of the transferability of a model by studying how the performance of the model trained on the available source distribution (data) would translate to the performance on its \emph{induced} domain. 
    We provide both upper bounds for the performance gap due to the induced domain shift, as well as lower bounds for the trade-offs that a classifier has to suffer on either the source training distribution or the induced target distribution. We provide further instantiated analysis for two popular domain adaptation settings, including \emph{covariate shift} and \emph{target shift}.
\end{abstract}

%%%%%%%%%%%%%%%%%%%%%%%%%%%%%%%%%%%%%%%%%%%%%%%%%%%%%
%%%%%%%%%%%%%%%%%%% INTRODUCTION %%%%%%%%%%%%%%%%%%%%%
%%%%%%%%%%%%%%%%%%%%%%%%%%%%%%%%%%%%%%%%%%%%%%%%%%%%%

\section{Introduction}
\label{sec:introduction}
Decision-makers are increasingly
required to be transparent on their decision-making rules to offer the ``right to explanation'' \citep{Goodman2017explain, Selbst2018explain, ustun2019actionable}. Being transparent also invites potential adaptations from the population, leading to potential shifts. We are motivated by settings where the deployed machine learning models interact with human agents, and will ultimately face data distributions that reflect human agents' responses to the models. For instance, when a model is used to decide loan applications, candidates may adapt their features based on the model specification in order to maximize their chances of approval; thus the loan decision classifier observes a new shifted distribution caused by its own deployment (e.g., see Figure \ref{Fig:ExampleAdaptation} for a demonstration). Similar observations can be articulated for application in the insurance sector, e.g., insurance companies may develop policy such that customers' behaviors might adapt to lower premium \citep{haghtalab2020maximizing}, the education sector, e.g., teachers may want to design courses in a way that students are less incentivized to cheat \citep{kleinberg2020classifiers}, and so on.

\begin{figure}[htbp]
    \centering
    \scriptsize
    \setlength{\tabcolsep}{3pt}
    \renewcommand{\arraystretch}{1.025}
    \begin{tabular}{lc||ccc}
        \toprule
        \textsc{Feature} & \textsc{ Weight} & \textsc{Original Value} & & \textsc{Adapted Value} \\ 
        \toprule
        \text{Income} & 2 & \$ 6,000 & $\longrightarrow$ & \$ 6,000\\ \midrule
        \text{Education Level} & 3 & \text{College} & $\longrightarrow$ & \text{College} \\ 
        \midrule
        \text{Debt} & \textbf{-10} & \$40,000 & $\longrightarrow$ &  \textbf{\$20,000} \\
        \midrule
        \text{Savings} & \textbf{5} & \$20,000 & $\longrightarrow$ &  \textbf{\$0} \\
        \bottomrule
    \end{tabular}
    \caption{An example of an agent who originally has both savings and debt, observes that the classifier penalizes debt (weight -10) more than it rewards savings (weight +5), and concludes that their most efficient adaptation is to use their savings to pay down debt.}
    \label{Fig:ExampleAdaptation}
\end{figure}

In this paper, we provide a general framework for quantifying the transferability of a decision rule when facing responsive decision subjects. What we would like to achieve is some characterizations of the \emph{performance guarantee} of a classifier --- that is, given a model primarily trained on the source distribution $\D_S$, how good or bad will it perform on the distribution it induces $\D(h)$, which depends on the model $h$ itself. A key concept in our setting is the \emph{induced risk}, defined as the error a model incurs on the distribution induced by itself:
\begin{align}
    \text{Induced Risk}: ~~\errd{\D(h)}{(h)}:=\P_{\D(h)}(h(X)\neq Y)
\end{align}

Most relevant to the above formulation are the works of literature on \emph{strategic classification} \cite{hardt2016strategic}, and \emph{performative prediction} \citep{perdomo2020performative}. 
In strategic classification, agents are modeled as rational utility maximizers, and under a specific agent's response model, game theoretical solutions were proposed to model the interactions between the agents and the decision-maker. In performative prediction, a similar notion of risk called the \emph{performative prediction risk} is introduced to measure a given model’s performance on the distribution itself induces. Different from ours, one of their main focus is to find the optimal classifier that achieves minimum induced risk after a sequence of model deployments and observing the corresponding response datasets, which might be computationally expensive. 

In particular, our results are motivated by the following challenges in more general scenarios:

\squishlist
    \item \textbf{Modeling assumptions being restrictive}~ In many practical situations, it is often hard to accurately characterize the agents' utilities. Furthermore, agents might not be fully rational when they respond. All the uncertainties can lead to a far more complicated distribution change in $(X,Y)$, as compared to often-made assumptions that agents only change $X$ but not $Y$ \citep{hardt2016strategic, chen2020learning, dong2018strategic}.
    
    \item \textbf{Lack of access to response data}
    During training, machine learning practitioners may only have access to data from the source distribution, and even when they can anticipate changes in the population due to human agents’ responses, they cannot observe the newly shifted distribution until the model is actually deployed.
    
    \item \textbf{Retraining being costly}~ Even when samples from the induced data distribution are available, retraining the model from scratch may be impractical due to computational constraints, and will result in another round of agents' response at its deployment. 
\squishend

The above observations motivate us to focus on understanding the transferability of a model before diving into finding the optimal solutions that achieve the minimum induced risk -- the latter problem often requires more specific knowledge on the mapping between the model and its induced distribution, which might not be available during the training process. Another related research problem is to find models that will perform well on both the source and the induced distribution. This question might be solved using techniques from \emph{domain generalization} \cite{ zhou2021domain,sheth2022domain}. 

We add detailed discussions on how our work relates to and differs from these fields in related works (\Cref{sec:related-work}), as well as on how to use existing techniques to solve these two questions in \Cref{sec:comparison-IDA-DA}. and \Cref{sec:discussion-minimize-ida}. 
We leave the full discussions of these topics to future work.

\subsection{Our Contributions}
In this paper, we aim to provide answers to the following fundamental questions:
\squishlist
    \item \textbf{Source risk $\Rightarrow$ Induced risk}~ For a given model $h$, how different is $\err{\D(h)}{h}$, the error on the distribution induced by $h$, from $\err{\D_S}{h}:=\P_{\D_S}(h(X)\neq Y)$, the error on the source?
    \item \textbf{Induced risk $\Rightarrow$ Minimum induced risk}~ How much higher is $\err{\D(h)}{h}$, the error on the induced distribution, than $\min_{h'} ~\err{\D(h')}{h'}$, the minimum achievable induced error?
    \item \textbf{Induced risk of \emph{source optimal} $\Rightarrow$ Minimum induced risk}~Of particular interest, and as a special case of the above, how does $\err{\D(h^*_S)}{h^*_S}$, the induced error of the optimal model trained on the source distribution {$h^*_S:= \argmin_h ~\err{\D_S}{h}$}, compare to {$h^*_T:= \argmin_{h}~ \err{\D(h)}{h}$}?
    \item \textbf{Lower bound for learning tradeoffs}~What is the minimum error a model must incur on either the source distribution $\err{\D_S}{h}$ or its induced distribution $\err{\D(h)}{h}$?
\squishend

For the first three questions, we prove upper bounds on the additional error incurred when a model trained on a source distribution is transferred over to its induced domain. We also provide lower bounds for the trade-offs a classifier has to suffer on either the source training distribution or the induced target distribution. We then show how to specialize our results to two popular domain adaptation settings: \emph{covariate shift} \citep{shimodaira2000improving,zadrozny2004learning,sugiyama2007covariate,sugiyama2008direct,zhang2013domain} and \emph{target shift} \citep{lipton2018detecting,guo2020ltf,zhang2013domain}.  

\subsection{Related Work}
\label{sec:related-work}
Our work most closely relates to the fields of strategic classification, domain adaptation, and performative prediction. In particular, our work considers a setting similar to the studies of strategic classification \citep{hardt2016strategic,chen2020learning,Dong2018reveal,chen2020strategic,miller2020strategic}, which primarily focus on developing robust classifiers in the presence of strategic agents, rather than characterizing the transferability of a given model’s performance on the distribution itself induces. Our work also builds on efforts in domain adaptation \citep{jiang2008literature,ben-david2010domain,sugiyama2008direct,zhang2019bridging,kang2019contrastive,NEURIPS2020_3430095c}. The major difference between our setting and those from previous works is that the changes in distribution are not passively provided by the environment, but rather an active consequence of model deployment. We reference specific prior work in these two domains in \Cref{sec:appendix-related-work}, and here provide more detailed discussions on the existing work in performative prediction. 

\paragraph{Performative Prediction}
Performative prediction is a new type of supervised
learning problem in which the underlying data distribution shifts in response to the deployed model \citep{perdomo2020performative,Mendler2020Stochastic,brown2020performative,drusvyatskiy2020stochastic,izzo2021learn,li2022state,maheshiwar2021minmax}. In particular, \citet{perdomo2020performative} first propose the notion of the \emph{performative risk} defined as $\text{PR}(\theta):= \E_{z \sim \D(\theta)}[\ell(\theta;z)]$, where $\theta$ is the model parameter, and $\D(\theta)$ is the induced distribution as a result of the deployment of $\theta$. Similar to our definition of induced risk, performative risk also measures a given model's performance on the distribution itself induces. 

The major difference between our work and performative prediction is that we focus on different aspects of the induced domain adaptation problem. One of the primary focuses of performative prediction is to find the optimal model $\theta_{\textsf{OPT}}$ which achieves the minimum performative prediction risk, or performative stable model $\theta_{\textsf{ST}}$, which is optimal under its own induced distribution. In particular, one way to find a performative stable model $\theta_{\textsf{ST}}$ is to perform repeated retraining \cite{perdomo2020performative}. In order to get meaningful theoretical guarantees on any proposed algorithms, works in this field generally require particular assumptions on the mapping between the model parameter and its induced distribution (e.g., the smoothness of the mapping), or requires multiple rounds of deployments and observing the corresponding induced distributions, which can be costly in practice \cite{jagadeesan2022regret, Mendler2020Stochastic}. On the contrary, our work's primary focus is to study the \emph{transferability} of a particular model trained primarily on the source distribution and provide theoretical bounds on its performance on its induced distribution, which is useful for estimating the effect of a given classifier when repeated retraining is unavailable. As a result, our work does not assume the knowledge of the supervision/label information on the transferred domain. Also related are the recently developed lines of work on the multiplayer version of the performative prediction problem \cite{piliouras2022multi, narang2022multiplayer}, and the economic aspects of performative prediction \cite{hardt2022performative, mendleranticipating2022}.
The details for reproducing our experimental results can be found at {\small \texttt{\url{https://github.com/UCSC-REAL/Model_Transferability}}}.

%%%%%%%%%%%%%%%%%%%%%%%%%%%%%%%%%%%%%%%%%%%%%%%%%%%%%
%%%%%%%%%%%%%%%%%%% FORMULATION %%%%%%%%%%%%%%%%%%%%%
%%%%%%%%%%%%%%%%%%%%%%%%%%%%%%%%%%%%%%%%%%%%%%%%%%%%%

\section{Notation and Formulation}

\emph{All proofs of our results can be found in the Appendix.}

Suppose we are given a parametric model $h \in \H$ primarily trained on the training data set $S:=\{x_i,y_i\}_{i=1}^N$, which is drawn from a \emph{source} distribution $\D_S$, where $x_i \in \R^d$ and $y_i \in \{-1,+1\}$. However, $h$ will then be deployed in a setting where the samples come from a \emph{test} or \emph{target} distribution $\D_T$ that can differ substantially from $\D_S$. 
Therefore, instead of finding a classifier that minimizes the prediction error on the source distribution $\err{\D_S}{h}:=\P_{\D_S}(h(X)\neq Y)$, ideally the decision maker would like to find $h^*$ that minimizes $\err{\D_T}{h}:=\P_{\D_T}(h(X)\neq Y)$. This is often referred to as the \emph{domain adaptation problem}, where typically, the transition from $\D_S$ to $\D_T$ is assumed to be independent of the model $h$ being deployed.

We consider a setting in which the distribution shift depends on $h$, or is thought of as being \emph{induced} by $h$. We will use $\D(h)$ to denote the \emph{induced domain} by $h$:
\[
    \D_S~~ \to~~ \text{\emph{encounters model $h$}}~~\to ~~\D(h)
\]
Strictly speaking, the induced distribution is a function of both $\D_S$ and $h$ and should be better denoted by $\D_S(h)$. To ease the notation, we will stick with $\D(h)$, but we shall keep in mind its dependency of $\D_S$. For now, we do not specify the dependency of $\D(h)$ as a function of $\D$ and $h$, but later in Section \ref{sec:cs} and \ref{sec:ls} we will further instantiate $\D(h)$ under specific domain adaptation settings. 

The challenge in the above setting is that when training $h$, the learner needs to carry the thoughts that $\D(h)$ should be the distribution it will be evaluated on and that the training cares about. Formally, we define the \emph{induced risk} of a classifier $h$ as the 0-1 error on the distribution $h$ induces:
\begin{align}
    \text{Induced risk}: ~~~~\err{\D(h)}{h}:=\P_{\D(h)}(h(X)\neq Y)
\end{align}
Denote by $h^*_{T} := \argmin_{h \in \H} ~\err{\D(h)}{h}$ the classifier with minimum induced risk. More generally, when the loss may not be the 0-1 loss, we define the \emph{induced $\ell$-risk} as
\[
  \text{Induced $\ell$-risk}: ~~~~ \err{\ell, \D(h)}{h}:= \E_{z \sim \D(h)}[\ell(h;z)]
\]
The induced risks will be the primary quantities we are interested in quantifying. The following additional notation will also help present our theoretical results in the following few sections: 
\squishlist
    \item Distributions of $Y$ on a distribution $\D$:  $\D_{Y}:=\P_{\D}(Y=y)$.\footnote{The ``:=" defines the RHS as the probability measure function for the LHS.}, and in particular $\D_Y(h):=\P_{\D(h)}(Y=y),~\D_{Y\vert S}:=\P_{\D_S}(Y=y)$
    \item Distribution of $h$ on a distribution $\D$: $\D_{h}:=\P_{\D}(h(X)=y)$, and in particular $\D_h(h):=\P_{\D(h)}(h(X)=y),~\D_{h\vert S}:=\P_{\D_S}(h(X)=y)$.
    \item Marginal distribution of $X$ for a distribution $\D$: $\D_{X}:= \P_{\D}(X=x)$, and in particular $\D_{X}(h):= \P_{\D(h)}(X=x),~\D_{X\vert S}:= \P_{\D_S}(X=x)$.\footnote{For continuous $X$, the probability measure shall be read as the density function.} 
    \item Total variation distance %defined between $\D$ and $\D'$ 
    \citep{ali1966general}: $
    \dtv(\D,\D'):= \text{sup}_{\mathcal O}|\P_{\D}(\mathcal O)-\P_{\D'}(\mathcal O)|.
    $
\squishend

\subsection{Example Induced Domain Adaptation Settings} 

\label{sec:example}

We provide two example models to demonstrate the use cases of the distribution shift models described in our paper. We provide more detailed descriptions of both settings and instantiate our bounds in \Cref{sec:strategic-classification} and \Cref{sec:replicator-dynamics}, respectively.

\paragraph{Strategic Response}
As mentioned before, one example of induced distribution shift is when human agents perform \emph{strategic response} to a decision rule. 
In particular, it is natural to assume that the mapping between feature vector $X$ and the qualification $Y$ before and after the human agents' best response satisfies \emph{covariate shift}: the feature distribution $\P(X)$ will change, but $\P(Y|X)$, the mapping between $Y$ and $X$, remain unchanged. Notice that this is different from the assumption made in the classic strategic classification setting \cite{hardt2016strategic}, where \emph{any} adaptations are considered malicious, which means any changes in the feature vector $X$ \emph{do not} change the underlying true qualification $Y$. In this example, we assume that changes in feature $X$ could potentially lead to changes in the true qualification $Y$ and that the mapping between $Y$ and $X$ remains the same before and after the adaptation. This is a common assumption made in a recent line of work on incentivizing improvement behaviors from human agents (see, e.g., \citealp{chen2020strategic,shavit2020causal}). We use \Cref{subfig:CS-DAG} (Up) as a demonstration of how distribution might shift for the strategic response setting. In Section \ref{sec:strategic-classification}, we will use the strategic classification setup to verify our obtained results.

\begin{figure}[!ht]
    \centering
    \includegraphics[width=0.35\textwidth]{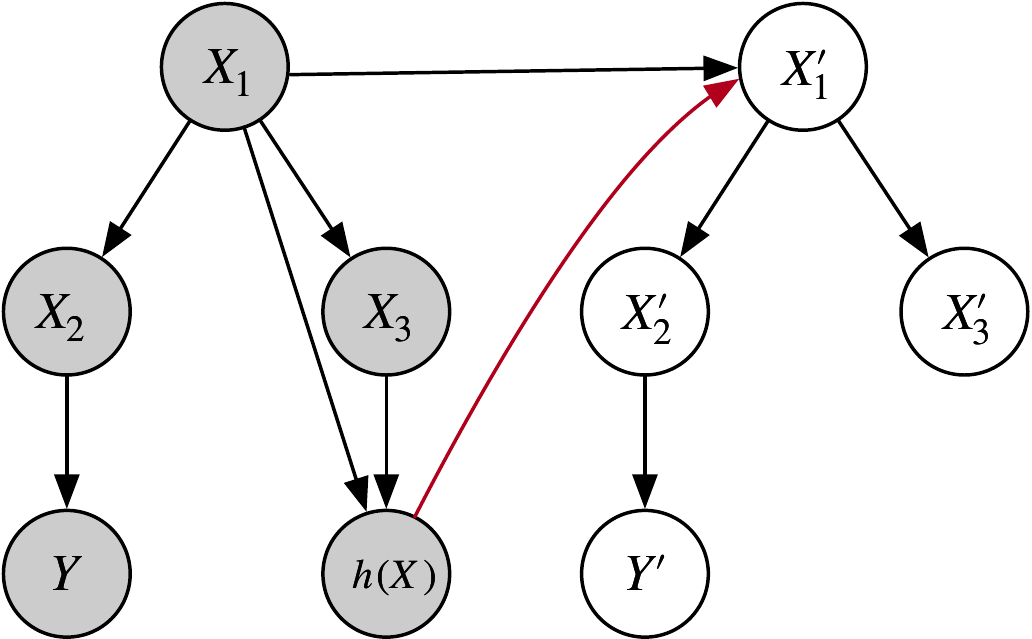}
    \hspace{0.2in}
    \includegraphics[width=0.35\textwidth]{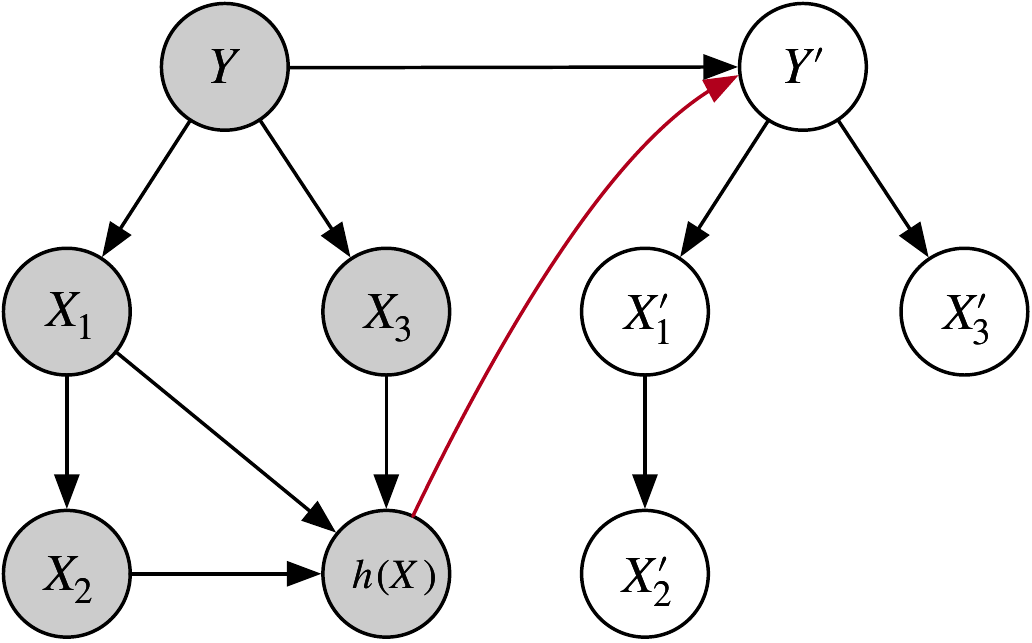}
    \caption{Example causal graph annotated to demonstrate covariate shift (\textbf{Up}) / target shift (\textbf{Down}) as a result of the deployment of $h$. Grey nodes indicate observable variables and transparent nodes are not observed at the training stage. Red arrow emphasizes $h$ induces changes in certain variables. }
    \label{subfig:CS-DAG}
\end{figure}

\paragraph{Replicator Dynamics}
Replicator dynamics is a commonly used model to study the evolution of an adopted ``strategy" in evolutionary game theory \citep{tuyls2006evolutionary,friedman2016evolutionary,Taylor1978evolutionary,raab2021unintended}. %\yc{add a few more citations.} 
The core notion of it is the growth or decline of the population of each strategy depends on its ``fitness". Consider the label $Y = \{-1,+1\}$ as the strategy, and  the following behavioral response model to capture the induced target shift:
\[
    \frac{\P_{\D(h)}(Y=+1)}{\P_{\D_S}(Y=+1)} = \frac{\textbf{Fitness}(Y=+1)}{\E_{\D_S}[\textbf{Fitness}(Y)]}
\]

The intuition behind the above equation is that the change of the $Y=+1$ population depends on how predicting $Y=+1$ ``fits" a certain utility function. For instance, the ``fitness" can take the form of the prediction accuracy of $h$ for class $+1$, namely $\textbf{Fitness}(Y=+1):=\P_{\D_S}(h(X)=+1|Y=+1)$. Intuitively speaking, a higher ``fitness" describes more success of agents who adopted a certain strategy ($Y = -1$ or $Y=+1$). Therefore, agents will imitate or replicate their successful peers by adopting the same strategy, resulting in an increase in the population ($\P_{\D(h)}(Y)$). 

With the assumption that $\P(X|Y)$ stays unchanged, this instantiates one example of a specific induced \emph{target shift}. We will provide detailed conditions for target shift in \Cref{sec:ls}. We also use \Cref{subfig:CS-DAG} (Down) as a demonstration of how distribution might shift for the replicator dynamic setting. In \Cref{sec:replicator-dynamics}, we will use a detailed replicator dynamics model to further instantiate our results. 

% \paragraph{Preview of the Empirical Justification} In Figure \Cref{subfig:CS-DAG}, we present the type of quantification we aim for using the two exemplary models we presented above. Details can be found in Section \ref{sec:exp}. In the next several sections, we will focus on deriving upper bounds and lower bounds ({\color{red} red x}) for the performance quantities (\textbf{black +}) that are of interest in our responsive agent setting.  

%%%%%%%%%%%%%%%%%%%%%%%%%%%%%%%%%%%%%%%%%%%%%%%%%%%%%
%%%%%%%%%%%%%%%%%% GENERAL BOUNDS %%%%%%%%%%%%%%%%%%%
%%%%%%%%%%%%%%%%%%%%%%%%%%%%%%%%%%%%%%%%%%%%%%%%%%%%%
\section{General Bounds}
\label{sec:general-bound}
In this section, we first provide upper and lower bounds for \emph{any} induced domain without specifying the particular type of distribution shift. In particular, we first provide upper bounds for the transfer error of any classifier $h$ (that is, the difference between $\err{\D(h)}{h}$ and $\err{\D_S}{h}$), as well as between $\err{\D(h)}{h}$ and the minimum induced risk $\err{\D(h^*_T)}{h^*_T}$. 
We then provide lower bounds for $\max\{\err{\D_S}{h}, \err{\D(h)}{h}\}$, that is, the minimum error a model $h$ must incur on either the source distribution $\D_S$ or the induced distribution $\D(h)$.

\subsection{Upper Bound}
We first investigate the upper bounds for the transfer errors. We begin by showing generic bounds and further instantiate the bound for specific domain adaptation settings in Section \ref{sec:cs} and \ref{sec:ls}. We begin by answering the following question: 

\begin{center}
\emph{
How does a model $h$ trained on its training data set fare on the induced distribution $\D(h)$?
}    
\end{center}

To that end, we define the minimum and $h$-dependent combined error of any two distributions $\D$ and $\D^\prime$ as:

\begin{align*}
    \lambda_{\D \rightarrow \D'} := \min_{h' \in \H} \err{\D'}{h'} + \err{\D}{h'}\\
    \Lambda_{\D \rightarrow \D'}(h) := \max_{h' \in \H} \err{\D'}{h} + \err{\D}{h}    
\end{align*}
and their corresponding $\H$-divergence as
$
    d_{\H \times \H}(\D, \D') = 2\sup_{h,h' \in \H}\left|\P_{\D}(h(X) \neq h'(X)) - \P_{\D'}(h(X) \neq h'(X))\right|.
$
The $\H$-divergence is a celebrated measure proposed in the domain adaptation literature \citep{ben-david2010domain} which will be useful for bounding the difference in errors of any two classifiers. Following the classical arguments from \citet{ben-david2010domain}, we can easily prove the following:

\begin{theorem}[\textbf{Source risk $\Rightarrow$ Induced risk}]
    \label{thm:transf:sh}
    The difference between  $\err{\D(h)}{h}$ and $\err{\D_S}{h}$ is upper bounded by: 
    $\err{\D(h)}{h} \leq \err{\D_S}{h} +   \lambda_{\D_S\rightarrow \D(h)} +\frac{1}{2}d_{\H \times \H}(\D_S, \D(h)). $
\end{theorem}
The transferability of a model $h$ between $\err{\D(h)}{h}$ and $\err{\D_S}{h}$ looks precisely the same as in the classical domain adaptation setting \citep{ben-david2010domain}. 

An arguably more interesting quantity in our setting to understand is the difference between the induced error of any given model $h$ and the error induced by the optimal model $h_T^*$:   
$\err{\D(h)}{h} - \err{\D(h^*_T)}{h^*_T}$. 
We get the following bound, which differs from the one in \Cref{thm:transf:sh}:

\begin{theorem}[\textbf{Induced risk $\Rightarrow$ Minimum induced risk}]
    \label{thm:transf:sh2}
    The difference between  $\err{\D(h)}{h}$ and $\err{\D(h^*_T)}{h^*_T}$ is upper bounded by:
    $\err{\D(h)}{h} - \err{\D(h^*_T)}{h^*_T} \leq \frac{\lambda_{\D(h) \rightarrow \D(h^*_T)}+\Lambda_{\D(h) \rightarrow \D(h^*_T)}(h)}{2}+ \frac{1}{2} \cdot d_{\H \times \H}(\D(h^*_T), \D(h)).$
\end{theorem}

The above theorem informs us that the induced transfer error is generally bounded by the ``average" achievable error on both distributions $\D(h)$ and $\D(h^*_T)$, as well as the $\H$ divergence between the two distributions. 

The major benefit of the results in \Cref{thm:transf:sh2} is that it provides the decision maker a way to estimate the minimum induced risk $\err{\D(h^*_T)}{h^*_T}$ even when she only has access to the induced risk of some available classifier $h$, as long as she can characterize the statistical difference between the two induced distribution. The latter, however, might not seem to be a trivial task itself. Later in \Cref{sec:how-to-use-our-bound}, we briefly discuss how our bounds can still be useful even when we do not have the exact characterizations of this quantity. 

\subsection{Lower Bound} 
\label{sec:lower-bound}
Now we provide a lower bound on the induced transfer error. We particularly want to show that at least one of the two errors $\err{\D_S}{h}$, and $\err{\D(h)}{h}$, must be lower-bounded by a certain quantity.

\begin{theorem}[\textbf{Lower bound for learning tradeoffs} ]
    \label{thm:lb}
    Any model $h$ must incur the following error on either the source or induced distribution:
    $\max\{\err{\D_S}{h}, \err{\D(h)}{h}\} \geq \frac{ \dtv(\D_{Y\vert S},\D_Y(h)) - \dtv(\D_{h\vert S}, \D_h(h))}{2}. $  
\end{theorem}

The proof leverages the triangle inequality of $\dtv$. This bound is dependent on $h$; however, by the data processing inequality of $\dtv$ (and $f$-divergence functions in general) \citep{liese2006divergences}, we have
$\dtv(\D_{h\vert S}, \D_h(h)) \leq  \dtv(\D_{X\vert S}, \D_{X}(h)).$ 
Applying this to Theorem \ref{thm:lb} yields:
\begin{corollary}
    \label{corr:general-lb}
    For any model $h$,
    \begin{align*}
        &\max\{\err{\D_S}{h}, \err{\D(h)}{h} \} \\
    \geq &\frac{\dtv(\D_{Y\vert S},\D_Y(h)) - \dtv(\D_{X\vert S}, \D_X(h))}{2} .    
    \end{align*}
\end{corollary}

The benefit of \Cref{corr:general-lb} is that the bound does not contain any quantities that are functions of the induced distribution; as a result, for any classifier $h$, we can estimate the learning tradeoffs between its source risk and its induced risk using values that are computable without actually deploying the classifier at the first place.  

\subsection{How to Use Our Bounds}
\label{sec:how-to-use-our-bound}
The upper and lower bounds we derived in the previous sections (\Cref{thm:transf:sh2} and \Cref{thm:lb}) depend on the following two quantities either explicitly or implicitly: 
1) the distribution $\D(h)$ induced by the deployment of the model $h$ in question, and 2) the optimal target classifier $h^*_T$ as well as the distribution $\D(h^*_T)$ it induces.  
The bounds may therefore seem to be of only theoretical interest since in reality we generally cannot compute $\D(h)$ without actual deployment, let alone compute $h_T^*$. Thus in general it is unclear how to compute the value of these bounds. 
 
Nevertheless, our bounds can still be useful and informative in the following ways:

\paragraph{General modeling framework with flexible hypothetical shifting models} 
The bounds can be evaluated if the decision maker has a particular shift model in mind, which specifies how the population would adapt to a  model. A common special case is when the decision maker posits an individual-level agent response model (e.g. the strategic agent \citep{hardt2016strategic} - we demonstrate how to evaluate in \Cref{sec:strategic-classification}). In these cases, the $\H$-divergence can be consistently estimated from finite samples of the population \citep{wang2005divergence}, allowing the decision maker to estimate the performance gap of a given $h$ without deploying it. The general bounds provided can thus be viewed as a framework by which specialized, computationally tractable bounds can be derived.

\paragraph{Estimate the optimal target classifier $h_T^*$ from a set of imperfect models}
Secondly, when the decision maker has access to a set of imperfect models $\tilde{h}_1, \tilde{h}_2\cdots \tilde{h}_t \in H^T$ that will predict a range of possible shifted distribution $\D(\tilde{h}_1), \cdots \D(\tilde{h}_t) \in \D^T$ and a range of possibly optimal target distribution $h_T \in \H^T$, the bounds on $h_T^*$ can be further instantiated by calculating the worst case in this predicted set :\footnote{UpperBound and LowerBound are the RHS expressions in \Cref{thm:transf:sh2} and \Cref{thm:lb}.}
    \begin{align*}
    &\err{\D(h)}{h} - \err{\D(h^*_T)}{h^*_T} \\
    &~~~~~~~~~~~~~~~~~~~~~~~\lesssim  \max_{\D' \in \D^T, h' \in \H^T} \text{UpperBound}(\D', h'), \\
    &\max\{\err{\D_S}{h}, \err{\D(h^*_T)}{h^*_T} \}  \\ 
    &~~~~~~~~~~~~~~~~~~~~~~~\gtrsim  \min_{\D' \in \D^T, h' \in \H^T} \text{LowerBound}(\D', h'). 
    \end{align*}

We provide discussions on the tightness of our bounds in \Cref{sec:discussion-tightness}.

%%%%%%%%%%%%%%%%%%%%%%%%%%%%%%%%%%%%%%%%%%%%
%%%%%%%%%%% COVARIATE SHIFT %%%%%%%%%%%%%%%%
%%%%%%%%%%%%%%%%%%%%%%%%%%%%%%%%%%%%%%%%%%%%

\section{Covariate Shift}
\label{sec:cs}
In this section, we focus on a particular distribution shift model known as \emph{covariate shift}, in which the distribution of features changes, but the distribution of labels conditioned on features remains the same:
\begin{align}
    \P_{\D(h)}(Y=y|X=x) &= \P_{\D_S}(Y=y|X=x)\\
    \P_{\D(h)}(X=x) &\neq \P_{\D_S}(X=x)
\end{align}

Thus with covariate shift, we have
\begin{align*}
    &\P_{\D(h)}(X=x,Y=y)\\
    =&\P_{\D(h)}(Y=y|X=x) \cdot  \P_{\D(h)}(X=x)\\
    =& \P_{\D_S}(Y=y|X=x) \cdot  \P_{\D(h)}(X=x) 
\end{align*}

Let $\omega_x(h) := \frac{\P_{\D(h)}(X=x)}{\P_{\D_S}(X=x)}$ be the \emph{importance weight} at $x$, which characterizes the amount of adaptation induced by $h$ at instance $x$. Then for any loss function $\ell$ we have:

\begin{proposition}[Expected Loss on $\D(h)$ Under Covariate Shift]
    \label{prop:reweight-loss-induced-distribution}
    $
    \E_{\D(h)}[\ell(h;X,Y)] = \E_{\D_S}[\omega_x(h) \cdot \ell(h;x,y)].
    $
\end{proposition}

The above derivation is a classic trick and offers the basis for performing importance reweighting when learning under covariate shift \citep{sugiyama2008direct}. The particular form informs us that $\omega_x(h)$ controls the generation of $\D(h)$ and encodes its dependency on both $\D_S$ and $h$, and is critical for deriving our results below. 

\subsection{Upper Bound}

We now derive an upper bound for transferability under covariate shift. We will particularly focus on the optimal model trained on the source data $\D_S$, which we denote as  $h^*_S:= \argmin_{h \in \H} \errd{S}{(h)}$. Recall that the classifier with minimum induced risk is denoted as $h^*_T:= \argmin_{h \in \H} \errd{\D(h)}{(h)}$. We can upper bound the difference between $h^*_S$ and $h^*_T$ as follows: 

\begin{theorem}[Suboptimality of $h^*_S$]
    \label{thm:cs-upper-bound}
    Let $X$ be distributed according to $\D_S$. We have:
    \small
    \begin{align*}
        &\err{\D(h^*_S)}{h^*_S} - \err{\D(h^*_T)}{h^*_T}\\
        \leq& \sqrt{\err{\D_S}{h^*_T}} \cdot \left(\sqrt{\var(\omega_{X}(h^*_S))} + \sqrt{\var(\omega_{X}(h^*_T))}\right) .
    \end{align*}
\end{theorem}

This result can be interpreted as follows: $h^*_T$ incurs an irreducible amount of error on the source data set, represented by $\sqrt{\err{\D_S}{h^*_T}}$. Moreover, the difference in induced risks between $h^*_S$ and $h^*_T$ is at its maximum when the two classifiers induce adaptations in ``opposite'' directions; this is represented by the sum of the standard deviations of their importance weights, $\sqrt{\var(\omega_{X}(h^*_S))} + \sqrt{\var(\omega_{X}(h^*_T ))}$.

\subsection{Lower Bound}
Recall from \Cref{thm:lb}, for the general setting, it is unclear whether the lower bound is strictly positive or not. In this section, we provide further understanding for when the lower bound $\frac{ \dtv(\D_{Y\vert S},\D_Y(h)) - \dtv(\D_{h\vert S}, \D_h(h))}{2}$ is indeed positive under covariate shift. 
Under several assumptions, our previously provided lower bound in \Cref{thm:lb} is strictly positive with covariate shift.

\begin{assumption}\label{as:cs-lb-1}
      $
    |\E_{X \in X_+(h), Y=+1}[1- \omega_X(h)]| \geq |\E_{X \in X_-(h), Y=+1}[1- \omega_X(h)]|~.
   $
\end{assumption}
{where $X_+(h) = \{x: \omega_x(h)\geq 1\}$ and $X_-(h) = \{x: \omega_x(h) < 1\}$.}

This assumption states that increased $\omega_x(h)$ value points are more likely to have positive labels.

\begin{assumption}\label{as:cs-lb-2}
        $
    |\E_{X \in X_+(h), h(X)=+1}[1- \omega_X(h)]| \geq |\E_{X \in X_-(h), h(X)=+1}[1- \omega_X(h)]|$.
\end{assumption}

This assumption states that increased $\omega_x(h)$ value points are more likely to be classified as positive.

\begin{assumption}\label{as:cs-lb-3}
    $\text{Cov}\big(\P_{\D_S}(Y = +1 | X = x)- \P_{\D_S}(h(x) = +1 | X = x), \omega_x(h)\big)> 0$.
\end{assumption}

This assumption is stating that for a classifier $h$, within all $h(X)=+1$ or $h(X)=-1$, a higher $\P_\D(Y = +1 | X = x)$ associates with a higher $\omega_x(h)$.

\begin{theorem}
    \label{thm:cs-lower-bound}
    Under \cref{as:cs-lb-1}~-~\cref{as:cs-lb-3}, the following lower bound is strictly positive under covariate shift:
    \small
    \begin{align*}
        &\max\{\err{\D_S}{h}, \err{\D(h)}{h}\}\\
        \geq& \frac{ \dtv(\D_{Y\vert S},\D_Y(h)) - \dtv(\D_{h\vert S}, \D_h(h))}{2}> 0 .
    \end{align*}
\end{theorem}

\subsection{Covariate Shift via Strategic Response}
\label{sec:strategic-classification}
As introduced in \Cref{sec:example}, we consider a setting caused by \emph{strategic response} in which agents are
classified by and adapt to a binary threshold classifier. In particular, each agent is associated with a $d$~dimensional continuous feature $x\in \R^d$ and a binary true qualification $y(x)\in \{-1,+1\}$, where $y(x)$ is a function of the feature vector $x$. 
Consistent with the literature in strategic classification \citep{hardt2016strategic}, a simple case where after seeing the threshold binary decision rule $h(x) = 2 \cdot \Indicator[x\geq \tau_h]-1$, the agents will \emph{best response} to it by maximizing the following utility function:
\begin{align*}
    u(x, x') = h(x') - h(x) - c(x, x'),    
\end{align*}

where $c(x,x')$ is the \emph{cost function} for decision subjects to modify their feature from $x$ to $x'$. We assume all agents are rational utility maximizers: they will only \emph{attempt} to change their features when the benefit of manipulation is greater than the cost (i.e. when $c(x, x')\leq 2$) and the agent will not change their feature if they are already accepted (i.e. $h(x) = +1$).
For a given threshold $\tau_h$ and manipulation budget $B$, the theoretical best response of an agent with original feature $x$ is:
\begin{align*}
  \Delta(x) = \argmax_{x'} u(x,x') ~~s.t.~ c(x, x')\leq B.  
\end{align*}
To make the problem tractable and meaningful, we further specify the following setups:

\begin{setup}(Initial Feature)
    \label{assumption:agent-initial-feature}
    Agents' initial features are uniformly distributed between $[0,1]\in \R^1$.
\end{setup}

\begin{setup}{(Agent's Cost Function)}
    \label{assumption:cost-function}
    The cost of changing from $x$ to $x'$ is proportional to the distance between them: $c(x, x') = \|x - x'\|$.
\end{setup} 
\Cref{assumption:cost-function} implies that only agents whose features are in between $[\tau_h - B, \tau_h)$ will \emph{attempt} to change their feature. 
We also assume that feature updates are \emph{probabilistic}, such that agents with
features closer to the decision boundary $\tau_h$ have a greater
\emph{chance} of updating their feature and each updated feature $x'$ is
sampled from a uniform distribution depending on $\tau_{h}$, $B$, and $x$ (see \Cref{assumption:agent-success-probability} \& \ref{assumption:agent-new-feature}):
\begin{setup}{(Agent's Success Manipulation Probability)}
    \label{assumption:agent-success-probability}
    For agents who \emph{attempt} to update their features, the probability of
    a successful feature update is
    $\P(X' \neq X) = 1 - \frac{|x - \tau_h|}{B}$.
\end{setup}
Intuitively this setup means that the closer the agent's original feature $x$ is to the decision boundary $\tau_h$, the more likely they can successfully change their feature to cross the decision boundary.
\begin{setup}[Adapted Feature's Distribution]
    \label{assumption:agent-new-feature}
    An agent's updated feature $x'$, given original $x$, manipulation budget $B$, and classification boundary $\tau_h$, is sampled as $X' \sim \text{Unif}(\tau_h, \tau_h + |B - x|)$.
\end{setup}
\Cref{assumption:agent-new-feature} aims to capture the fact that even though agent targets to change their feature to the decision boundary $\tau_h$ (i.e. the least cost action to get a favorable prediction outcome), they might end up reaching a feature that is beyond the decision boundary.

With the above setups, we can specify the bound in \Cref{thm:cs-upper-bound} for the strategic response setting as follows:
\begin{proposition} 
For our assumed setting of strategic response described above, \Cref{thm:cs-upper-bound} implies
\label{proposition:bound-strategic-response}

$
    \err{\D(h^*_S)}{h^*_S} - \err{\D(h^*_T)}{h^*_T}
    \leq  \sqrt{\frac{2B}{3}\err{\D_S}{h^*_T}} .
$
\end{proposition}
We can see that the upper bound for strategic response depends on the manipulation budget $B$, and the error the ideal classifier made on the source distribution $\err{D_S}{h^*_T}$. This aligns with our intuition that the smaller the manipulation budget is, the fewer agents will change their features, thus leading to a tighter upper bound on the difference between $\err{h^*_S}{h^*_S}$ and $\err{h^*_T}{h^*_T}$. This expression also allows us to provide bounds even without the knowledge of the mapping between $\D(h)$ and $h$, since we can directly compute $\err{\D_S}{h^*_T}$ from the source distribution and an estimated optimal classifier $h^*_T$.

%%%%%%%%%%%%%%%%%%%%%%%%%%%%%%%%%%%%%%%%%%%%%%%%%%%%%
%%%%%%%%%%%%%%%%%% TARGET SHIFT %%%%%%%%%%%%%%%%%%%%%
%%%%%%%%%%%%%%%%%%%%%%%%%%%%%%%%%%%%%%%%%%%%%%%%%%%%%

\section{Target Shift}
\label{sec:ls}
We consider another popular domain adaptation setting known as \emph{target shift}, in which the distribution of labels changes, but the distribution of features conditioned on the label remains the same:
\begin{align}
    \P_{\D(h)}(X=x|Y=y) &= \P_{\D_S}(X=x|Y=y)\\
    \P_{\D(h)}(Y=y) &\neq \P_{\D_S}(Y=y)
\end{align}

For binary classification, let
${p(h)} := \P_{\D(h)}(Y=+1)$, and $ \P_{\D(h)}(Y=-1) =1-{p(h)}$. Notice that $p(h)$ encodes the full adaptation information from $\D_S$ to $\D(h)$, since the mapping between $Y$ and $X$ $\P(X=x|Y=y)$ is known and remains unchanged during target shift. 
We have for any proper loss function $\ell$:
\begin{align*}
    &\E_{\D(h)}[\ell(h;X,Y)] \\
    =& p(h) \cdot \E_{\D(h)}[\ell(h;X,Y)|Y=+1] \\
    & ~~ ~~ + (1-p(h)) \cdot \E_{\D(h)}[\ell(h;X,Y)|Y=-1] \\
    =& p(h) \cdot \E_{\D_S}[\ell(h;X,Y)|Y=+1] \\
    & ~~ ~~ + (1-p(h)) \cdot \E_{\D_S}[\ell(h;X,Y)|Y=-1] 
\end{align*}
We will adopt the following shorthands:
$
    \err{+}{h} := \E_{\D_S}[\ell(h;X,Y)|Y=+1],~~ \err{-}{h} := \E_{\D_S}[\ell(h;X,Y)|Y=-1].
$
Note that $\err{+}{h}, \err{-}{h}$ are both defined on the conditional source distribution, which is invariant under the target shift assumption.

\subsection{Upper Bound}
We first provide characterizations of the upper bound on the transferability of $h^*_S$ under target shift. 
Denote by $\D_+$ the positive label distribution on $\D_S$ ($\P_{\D_S}(X=x|Y=+1)$) and $\D_-$ the negative label distribution on $\D_S$ ($\P_{\D_S}(X=x|Y=-1)$). Let \underline{$p:= \P_{\D_S}(Y=+1)$}. 

\begin{theorem}
    \label{thm:ls:ub}
    For target shift, the difference between $\err{\D(h^*_S)}{h^*_S}$ and $\err{\D(h^*_T)}{h^*_T}$ bounds as:
    \begin{align*}
        &\err{\D(h^*_S)}{h^*_S} - \err{\D(h^*_T)}{h^*_T}
        \leq  |\omega(h^*_S)-\omega(h^*_T)| \\
         + &(1+p)
        \cdot \left(\dtv(\D_+(h^*_S),\D_+(h^*_T))+\dtv(\D_-(h^*_S),\D_-(h^*_T)\right).
    \end{align*}
\end{theorem}

The above bound consists of two components. The first quantity captures the difference between the two induced distributions $\D(h^*_S)$ and $\D(h^*_T)$. The second quantity characterizes the difference between the two classifiers $h^*_S, h^*_T$ on the source distribution. 

\subsection{Lower Bound}

Now we discuss lower bounds. Denote by $\TPR_S(h)$ and $\FPR_S(h)$ the true positive and false positive rates of $h$ on the source distribution $\D_S$. We prove the following:
\begin{theorem}
    For target shift, any model $h$ must incur the following error on either $\D_S$ or $\D(h)$:
    \small
    \begin{align*}
        & \max\{\err{\D_S}{h}, \err{\D(h)}{h} \}    \\
        \geq &
    \frac{|p-{p(h)}|\cdot (1-|\TPR_S(h)-\FPR_S(h)|)}{2}  .  
    \end{align*}
    \label{thm:ls:lb}
\end{theorem}
The proof extends the bound of Theorem \ref{thm:lb} by further explicating each of $\dtv(\D_{Y \vert S},\D_Y(h))$, $\dtv(\D_{h|S}$, $\D_h(h))$ under the assumption of target shift. Since $|\TPR_S(h)-\FPR_S(h)| < 1$ unless we have a trivial classifier that has either $\TPR_S(h)= 1, \FPR_S(h)=0$ or $\TPR_S(h)= 0, \FPR_S(h)=1$, the lower bound is strictly positive. Taking a closer look, the lower bound is determined linearly by how much the label distribution shifts: $p-{p(h)}$. The difference is further determined by the performance of $h$ on the source distribution through $1-|\TPR_S(h)-\FPR_S(h)|$. For instance, when $\TPR_S(h)>\FPR_S(h)$, the quality becomes $\text{FNR}_S(h)+\FPR_S(h)$, that is the more error $h$ makes, the larger the lower bound will be. 

\subsection{Target Shift via Replicator Dynamics}
\label{sec:replicator-dynamics}
We now further instantiate our theoretical bound for target shift (\Cref{thm:ls:ub}) using a particular replicator dynamics model previously used in \cite{raab2021unintended}. In particular, the fitness function is specified as the prediction accuracy of $h$ for class $y$:
\begin{align}
    \textbf{Fitness}(Y=y) := \P_{\D_S}(h(X)=y|Y=y) \label{eqn:fitness}
\end{align}
Then we have $\E\left[\textbf{Fitness}(Y)\right] = 1- \err{\D_S}{h}$, and
$
   \frac{{p(h)}}{\P_{\D_S}(Y=+1)} = \frac{\Pr_{\D_S}(h(X)=+1|Y=+1)}{1-\err{\D_S}{h}}.
$
Plugging the result back into \Cref{thm:ls:ub} we get the following bound for the above replicator dynamic setting:
\begin{proposition}
    \label{prop:ls:rd} 
    Under the replicator dynamics model described in \Cref{eqn:fitness}, $ |\omega(h^*_S)-\omega(h^*_T)|$ bounds as:
    \begin{align*}
      &   |\omega(h^*_S)-\omega(h^*_T)|\leq \P_{\D_S}(Y=+1) \\
        \cdot &\frac{|\err{\D_S}{h^*_S }-\err{\D_S}{h^*_T}| \cdot |\TPR_S(h^*_S)-\TPR_S(h^*_T)|}{\err{\D_S}{h^*_S } \cdot \err{\D_S}{h^*_T}}.
    \end{align*}
\end{proposition}

The above result shows that the difference between the induced risks $\err{\D(h^*_S)}{h^*_S}$ and $\err{\D(h^*_T)}{h^*_T}$ only depends on the difference between the two classifiers' performances on the source data $\D_S$. This offers the decision maker a great opportunity to evaluate the performance gap by using their corresponding evaluations on the source data only without observing their corresponding induced distributions. 

\begin{figure}
    \centering
    \includegraphics[width=0.34\textwidth]{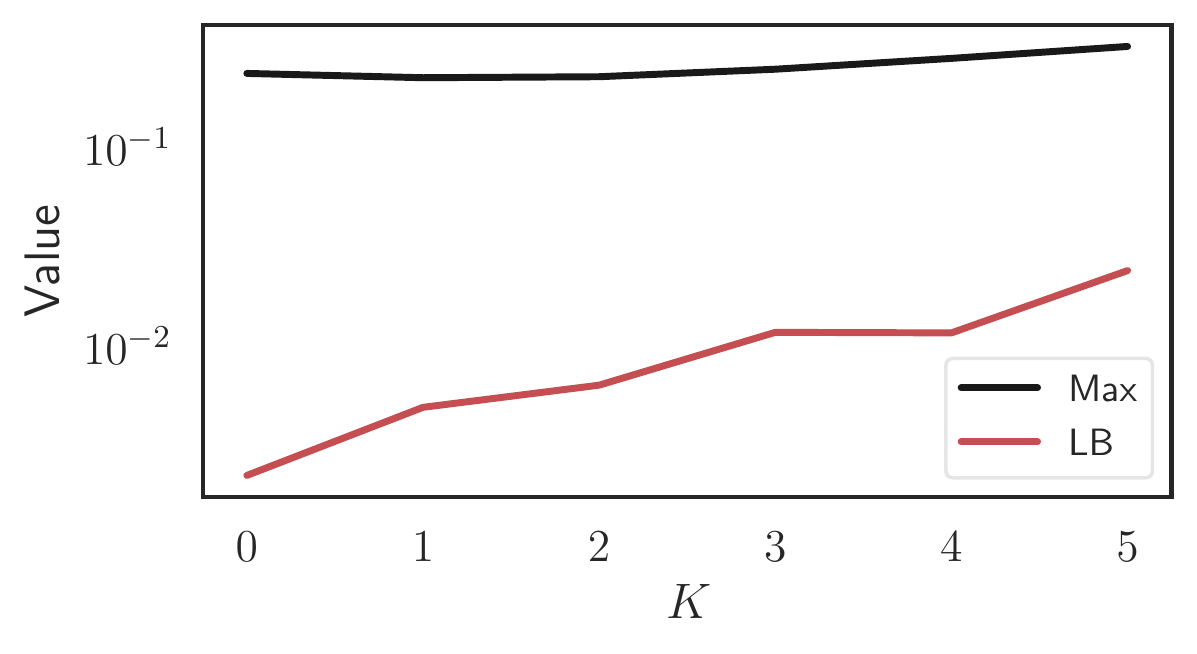}
%\hspace{0.2in}
\includegraphics[width=0.34\textwidth]{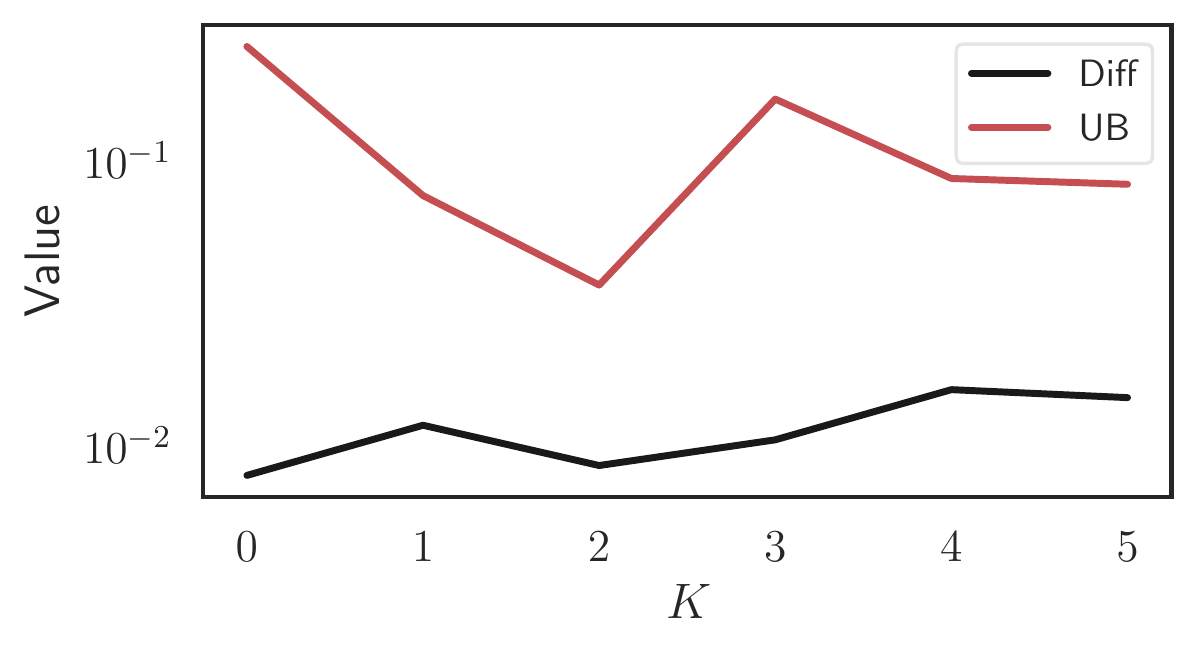}
    \caption{{
        Results for synthetic experiments on real-world data.
        $\textsf{Diff}:= \err{\D(h^*_S)}{h^*_S} - \err{\D(h^*_T)}{h^*_T}$, $\textsf{Max} := \max\{\err{\D_S}{h^*_T}, \err{\D(h^*_T)}{h^*_T}\}$, $\textsf{UB}:=$ upper bound specified in \Cref{thm:cs-upper-bound}, and $\textsf{LB}:=$ lower bound specified in \Cref{thm:cs-lower-bound}. For each time step $K = k$, we compute and deploy the source optimal classifier $h^*_S$ and update the credit score for each individual according to the received decision as the new reality for time step $K = k + 1$. Details of the data generation are deferred to  \Cref{sec:experiment-details}}.
    }
    \label{subfig:exp-DAG}
\end{figure}

%%%%%%%%%%%%%%%%%%%%%%%%%%%%%%%%%%%%%%%%%%%%%%%%%%%%%
%%%%%%%%%%%%%%%%%% EXPERIMENT %%%%%%%%%%%%%%%%%%%%%
%%%%%%%%%%%%%%%%%%%%%%%%%%%%%%%%%%%%%%%%%%%%%%%%%%%%%
\section{Experiments}\label{sec:exp} 
We present synthetic experimental results on both simulated and real-world data sets. 

\paragraph{Synthetic experiments using simulated data} 
We generate synthetic data sets from the structural equation models described on simple causal DAG in Figure \ref{subfig:CS-DAG} for covariate shift and target shift. To generate the induced distribution $\D(h)$, we posit a specific \emph{adaptation function} $\Delta: \R^d \times \H \to \R^d$, so that when an input $x$ encounters classifier $h \in \H$, its induced features are precisely $x' = \Delta(x,h)$. We provide details of the data generation processes and adaptation functions in \Cref{sec:experiment-details}.

We take our training data set $\{x_1,\ldots,x_n\}$ and learn a ``base'' logistic regression model $h(x) = \sigma(w \cdot x)$.\footnote{$\sigma(\cdot)$ is the logistic function and $w\in \R^3$ denotes the weights.} We then consider the hypothesis class $\H := \{h_\tau ~|~ \tau \in [0,1]\}$, where $h_\tau(x) := 2\cdot \Indicator[\sigma(w \cdot x) > \tau]-1$. To compute $h^*_S$, the model that performs best on the source distribution, we simply vary $\tau$ and take the $h_\tau$ with the lowest prediction error. Then, we posit a specific adaptation function $\Delta(x,h_\tau)$. Finally, to compute $h_T^*$, we vary $\tau$ from $0$ to $1$ and find the classifier $h_\tau$ that minimizes the prediction error on its induced data set $\{\Delta(x_1,h_\tau),\ldots,\Delta(x_n,h_\tau)\}$. 
We report our results in \Cref{fig:synthetic-experiments}.

For all four datasets, we do observe positive gaps $\err{D(h_S^*)}{h_S^*} - \err{D(h_T^*)}{h_T^*}$, indicating the suboptimality of training on $\D_S$. The gaps are well bounded by the theoretical results. For the lower bound, the empirical observation and the theoretical bounds are roughly within the same magnitude except for one target shift dataset, indicating the effectiveness of our theoretical result. Regarding the upper bound, for target shift, the empirical observations are well within the same magnitude of the theoretical bounds while the results for the covariate shift are relatively loose.

\begin{figure}[h!]
    \centering
    \includegraphics[width=60mm]{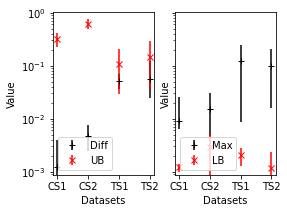}
    \caption{Results for synthetic experiments on simulated and real-world data.
  $\textsf{Diff}:= \err{\D(h^*_S)}{h^*_S} - \err{\D(h^*_T)}{h^*_T}$, $\textsf{Max} := \max\{\err{\D_S}{h^*_T}, \err{\D(h^*_T)}{h^*_T}\}$, $\textsf{UB}:=$ upper bound specified in \Cref{thm:cs-upper-bound}, and $\textsf{LB}:=$ lower bound specified in \Cref{thm:cs-lower-bound}.}
    \label{fig:synthetic-experiments}
\end{figure}

\paragraph{Synthetic experiments using real-world data}
We also perform synthetic experiments using real-world data to demonstrate our bounds.
In particular, we use the FICO credit score data set \citep{board2007report} which contains more than 300k records of TransUnion credit scores of clients from different demographic groups.
For our experiment on the preprocessed FICO data set \citep{hardt2016equality}, we convert the cumulative distribution function (CDF) of TransRisk score among different groups into group-wise credit score densities, from which we generate a balanced sample to represent a population where groups have equal representations.
We demonstrate the application of our results in a series of resource allocations.
Similar to the synthetic experiments on simulated data, we consider the hypothesis class of threshold classifiers and treat the classification outcome as the decision received by individuals.

For each time step $K = k$, we compute $h^{*}_S$, the statistical optimal classifier on the source distribution (i.e., the current reality for step $K = k$), and update the credit score for each individual according to the received decision as the new reality for time step $K = k + 1$. Details of the data generation are again deferred to  \Cref{sec:experiment-details}.  We report our results in Figure \ref{subfig:exp-DAG}.
We do observe positive gaps  $\err{\D(h^*_S)}{h^*_S} - \err{\D(h^*_T)}{h^*_T}$, indicating the suboptimality of training on $\D_S$. The gaps are well bounded by the theoretical upper bound (UB). Our lower bounds (LB) do return meaningful positive gaps, demonstrating the trade-offs that a classifier has to suffer on either the source distribution or the induced target distribution. We also provide additional experimental results using synthetic datasets generated according to causal graphs defined in \Cref{subfig:CS-DAG}. Due to page limits, we defer the detailed discussions of these results to \Cref{sec:appendix_additional_exps}.

\section{Conclusions and Future Directions}
Unawareness of the potential distribution shift might lead to unintended consequences when training a machine learning model. One goal of our paper is to raise awareness of this issue for the safe deployment of machine learning methods in high-stake scenarios. We also provide a general framework for characterizing the performance difference for a fixed-trained classifier when the decision subjects respond to it. 

Our contributions are mostly theoretical. A natural extension of our work is to collect real human experiment data to verify the usefulness and tightness of our bounds. Another potential future direction is to develop algorithms to find an optimal model that achieves minimum induced risk, which has been an exciting ongoing research problem in the field of performative prediction. Furthermore, using techniques from general domain adaptation to find robust classifiers that perform well in both the source and induced distribution is another promising direction. 

\paragraph{Ackowledgement} Y. Chen is partially supported by the National Science Foundation (NSF) under grants IIS-2143895 and IIS-2040800. 
The work is also supported in part by the NSF-Convergence Accelerator Track-D award \#2134901, by the National Institutes of Health (NIH) under Contract R01HL159805, by grants from Apple Inc., KDDI Research, Quris AI, and IBT, and by generous gifts from Amazon, Microsoft Research, and Salesforce.

\bibliographystyle{icml2023}

\bibliography{myref,da,performative_prediction}

%%%%%%%%%%%%%%%%%%%%%%%%%%%%%%%%%%%%%%%%%%%%%%%%%%%%%%%%%%%%%%%%%%%%%%%%%%%%%%%
%%%%%%%%%%%%%%%%%%%%%%%%%%%%%%%%%%%%%%%%%%%%%%%%%%%%%%%%%%%%%%%%%%%%%%%%%%%%%%%
% APPENDIX
%%%%%%%%%%%%%%%%%%%%%%%%%%%%%%%%%%%%%%%%%%%%%%%%%%%%%%%%%%%%%%%%%%%%%%%%%%%%%%%
%%%%%%%%%%%%%%%%%%%%%%%%%%%%%%%%%%%%%%%%%%%%%%%%%%%%%%%%%%%%%%%%%%%%%%%%%%%%%%%
\newpage
\onecolumn


% You can have as much text here as you want. The main body must be at most $8$ pages long.
% For the final version, one more page can be added.
% If you want, you can use an appendix like this one, even using the one-column format.

\input{proofs}
%%%%%%%%%%%%%%%%%%%%%%%%%%%%%%%%%%%%%%%%%%%%%%%%%%%%%%%%%%%%%%%%%%%%%%%%%%%%%%%
%%%%%%%%%%%%%%%%%%%%%%%%%%%%%%%%%%%%%%%%%%%%%%%%%%%%%%%%%%%%%%%%%%%%%%%%%%%%%%%

\end{document}

%% file: math_commands.tex
\usepackage{amsmath,amsfonts,bm}

\def\eqref#1{equation~\ref{#1}}
\def\1{\bm{1}}

\DeclareMathAlphabet{\mathsfit}{\encodingdefault}{\sfdefault}{m}{sl}
\SetMathAlphabet{\mathsfit}{bold}{\encodingdefault}{\sfdefault}{bx}{n}

\newcommand{\E}{\mathbb{E}}

\newcommand{\R}{\mathbb{R}}

\DeclareMathOperator*{\argmax}{arg\,max}
\DeclareMathOperator*{\argmin}{arg\,min}

\def\dtv{{d_{\text{TV}}}}
\newcommand\err[2]{\text{Err}_{#1}(#2)}
\newcommand\errd[1]{\text{Err}_{#1}}

\newcommand{\Indicator}{\mathds{1}}

\def\IR{{\textsf{IR}}}

\def\TPR{{\text{TPR}}}
\def\FPR{{\text{FPR}}}

\def\var{{\text{Var}}}
\def\cov{{\text{Cov}}}

\def\P{{\mathbb P}}

\def\R{{\mathbb R}}
\def\E{{\mathbb E}}

\def\R{{\mathbb{R}}}

\def\D{{\mathcal D}}
\def\H{{\mathcal H}}

\newcommand{\squishlist}{
\begin{list}{{{\small{$\bullet$}}}}
{\setlength{\itemsep}{3pt}      \setlength{\parsep}{1pt}
\setlength{\topsep}{1pt}       \setlength{\partopsep}{0pt}
\setlength{\leftmargin}{1em} \setlength{\labelwidth}{1em}
\setlength{\labelsep}{0.5em} } }
\newcommand{\squishend}{  \end{list}  }

%% file: proofs.tex
\appendix

\icmltitle{Supplement to ``Model Transferability with Responsive Decision Subjects''}

%%%%%%%%%%%%%%%%%%%%%%%%%%%%%%%
% ZT: minitoc
%%%%%%%%%%%%%%%%%%%%%%%%%%%%%%%
\renewcommand \thepart{}
\renewcommand \partname{}

% after \appendix, within the file where \section{} is called
\addcontentsline{toc}{section}{} % Add the appendix text to the document TOC
\part{} % Start the appendix part
% \vspace{-3em}
\parttoc % Insert the appendix TOC
% \clearpage

% end of minitoc
%%%%%%%%%%%%%%%%%%%%%%%%%%%%%%

%\yc{add an overview of the arrangement of the appendix}
We arrange the appendix as follows:
\squishlist
    \item \Cref{sec:use-transparent-model} provides some real-life scenarios where transparent models are useful or required. 
    \item \Cref{sec:appendix-related-work} provides additional related work on strategic classification and domain adaptation, as well as a detailed comparison of our setting and other sub-areas in domain adaptation.
    \item \Cref{sec:proof-thm-sh} provides proof for \Cref{thm:transf:sh}.
    \item \Cref{sec:proof-thm-sh2} provides proof for \Cref{thm:transf:sh2}.
    \item \Cref{sec:proof-thm-lb} provides proof of  \Cref{thm:lb}.
    \item \Cref{sec:proof-reweight-prop} provides proof for \Cref{prop:reweight-loss-induced-distribution}.
    \item \Cref{sec:proof-cs-ub} provides proof for \Cref{thm:cs-upper-bound}.
    \item \Cref{sec:proof-cs-lb} provides proof for \Cref{thm:cs-lower-bound}.
    \item \Cref{sec:proof-strategic-response} provides omitted assumptions and proof for \Cref{sec:strategic-classification}.
    \item \Cref{sec:proof-ls-ub} provides proof for \Cref{thm:ls:ub}.
    \item \Cref{sec:proof-ls-lb} provides proof for \Cref{thm:ls:lb}.
    \item \Cref{sec:proof-ls-rd} provides proof for \Cref{prop:ls:rd}.
    \item \Cref{sec:ts:appendix} provides additional lower bound and examples for the target shift setting. 
    \item \Cref{sec:experiment-details} provides missing experimental details.
    \item \Cref{sec:computation} discusses challenges in minimizing induced risk.
    \item \Cref{sec:discussion-minimize-ida} provides discussions on how to directly minimize the induced risk.
    \item \Cref{sec:reguarlized-training} provides discussions on adding regularization to the objective function.
    {\item \Cref{sec:discussion-tightness} provides discussions on the tightness of our theoretical bounds.}
\squishend

\section{Additional Discussions}
\subsection{Example Usages of Transparent Models}
\label{sec:use-transparent-model}
As we mentioned in \Cref{sec:introduction}, there is an increasing requirement of making the decision rule to be transparent due to its potential consequences impacts to individual decision subject. Here we provide the following reasons for using transparent models:
\begin{itemize}
    \item Government regulation may require the model to be transparent, especially in public services;
    \item In some cases, companies may want to disclose their models so users will have explanations and are incentivized to better use the provided services.
    \item Regardless of whether models are published voluntarily, model parameters can often be inferred via well-known query “attacks”.
\end{itemize}

In addition, we name some concrete examples of some real-life applications:
\begin{itemize}
    \item Consider the \emph{Medicaid health insurance program} in the United States, which serves low-income people. There is an obligation to provide transparency/disclose the rules (model to automate the decisions) that decide whether individuals qualify for the program — in fact, most public services have "terms" that are usually set in stone and explained in the documentation. Agents can observe the rules and will adapt their profiles to be qualified if needed. For instance, an agent can decide to provide additional documentation they need to guarantee approval. For more applications along these lines, please refer to this report.\footnote{\url{https://datasociety.net/library/poverty-lawgorithms/}}
    \item Credit score companies directly publish their criteria for assessing credit risk scores. In loan application settings, companies actually have the incentive to release criteria to incentivize agents to meet their qualifications and use their services.Furthermore, making decision models transparent will gain the trust of users.
    \item It is also known that it is possible to steal model parameters, if agents have incentives to do so.\footnote{\url{https://www.wired.com/2016/09/how-to-steal-an-ai/}} For instance, spammers frequently infer detection mechanisms by sending different email variants; they then adjust their spam content accordingly.
\end{itemize}

\subsection{Additional Related Work}
\label{sec:appendix-related-work}
\paragraph{Strategic Classification}
Strategic Classification focuses on the problem of how to make predictions in the presence of agents who behave strategically in order to obtain desirable outcomes \citep{hardt2016strategic,chen2020learning,Dong2018reveal,chen2020strategic,miller2020strategic}. 
In particular, \citep{hardt2016strategic} first formalizes strategic classification tasks as a two-player sequential game (i.e., a Stackelberg game) between a model designer and strategic agents. Agent best response behavior is typically viewed as malicious in the traditional setting; as a result, the model designer seeks to disincentivize this behavior or limit its impact by publishing classifiers that are robust to any agent's adaptations. In our work, the agents' strategic behaviors are not necessarily malicious; instead, we aim to provide a general framework that works for any distribution shift resulting from the human agency.  

Most existing work in strategic classification assumes that human agents are fully rational and will always perform \emph{best response} to any given classifier. As a result, their behaviors can be fully characterized based on pre-specified human response models \cite{hardt2016strategic, chen2020learning}. 
While we are also interested in settings where agents respond to a decision rule, we focus on the distribution shift of human agents at a population level and characterize the induced distribution as a function of the deployed model. Instead of specifying a particular individual-level agent's response model, we only require the knowledge of the source data $\D_S$, as well as some characterizations of the relationship between the source and the induced distribution, e.g., they satisfy some particular distribution shift models, like covariate shift (see \Cref{sec:cs}), or target shift (see \Cref{sec:ls}), or we have access to some data points from the induced distribution so we can estimate their statistical differences like H-divergence (see \cref{sec:general-bound}). 
In addition, the focus of our work is different from strategic classification. Instead of designing models robust to strategic behavior, we primarily study the \emph{transferability} of a given model's performance on the distribution itself induces. 

\paragraph{Domain Adaptation}
There has been tremendous work in domain adaptation studying different distribution shifts and learning from shifting distributions \citep{jiang2008literature,ben-david2010domain,sugiyama2008direct,zhang2019bridging,kang2019contrastive,NEURIPS2020_3430095c,xie2022gearnet}. Our results differ from these previous works since in our setting, changes in distribution are not passively provided by the environment, but rather an active consequence of model deployment. Part of our technical contributions is inspired by the transferability results in domain adaptations \cite{ben-david2010domain,zadrozny2004learning,gretton2009covariate,sugiyama2008direct,lipton2018detecting,azizzadenesheli2019regularized}.

Our work, at first sight, looks similar to several sub-areas within the literature of domain adaptation, e.g., domain generalization, adversarial attack, and test-time adaptation, to name a few. 
For instance, the notion of observing an ``induced distribution" resembles similarity of the adversarial machine learning literature \citep{Lowd05,Huang11,Vorobeychik18}. 
One of the major differences between ours and adversarial machine learning is that in adversarial machine learning, the true label $Y$ stays the same for the attacked feature, while in our paper, both $X$ and $Y$ might change in the induced distribution $\D(h)$. In \Cref{sec:comparison-IDA-DA}, we provide detailed comparisons between our setting and the other subfields in domain adaptation mentioned above. 

\paragraph{Comparisons of our setting and Some Areas in Domain Adaptation} 
\label{sec:comparison-IDA-DA}
We compare our setting (We address it as IDA, representing ``induced domain adaptation'') with the following areas:
\begin{itemize}
    \item Adversarial attack \cite{chakraborty2018adversarial,papernot2016transferability,song2019improving}: in adversarial attack, the true label $Y$ stays the same for the attacked feature, while in IDA, we allow the true label to change as well. One can think of adversarial attack as a specific form of IDA where the induced distribution has a specific target, that is to maximize the classifier’s error by only perturbing/modifying. Our transferability bound does, however, provide insights into how standard training results transfer to the attack setting.
    
    \item Domain generalization \cite{wang2021generalizing, li2017learning,muandet2013domain}: the goal of domain generalization is to learn a model that can be generalized to any unseen distribution; Similar to our setting, one of the biggest challenges in domain generalization also the lack of target distribution during training. The major difference, however, is that our focus is to understand how the performance of a classifier  trained on the source distribution degrades when evaluated on the induced distribution (which depends on how the population of decision subjects responds); this degradation depends on the classifier itself. 
    \item Test-time adaptation \cite{varsavsky2020testtime, wang2021tent,nado2021evaluating}: the issue of test-time adaptation falls into the classical domain adaptation setting where the adaptation is independent of the model being deployed. Applying this technique to solve our problem requires accessing data (either unsupervised or supervised) drawn from $\D_S(h)$ for each $h$ being evaluated during different training epochs.
\end{itemize}

\paragraph{Remark}

We believe that techniques from Domain Adaptation can potentially be applied to our setting when the decision-maker is interested in producing a classifier that performs well on the induced distribution. In general, we suspect that it will require the decision maker to know certain information about how the induced distribution and the source distribution differ, e.g. some potential characterizations of their statistical differences. Here we provide two possible directions:
\begin{itemize}
    \item Data Augmentation: The basic idea of data augmentation is to augment the original data point $(x, y)$ pairs with new pairs $(A(x), B(x)) $where $A(\cdot)$ and $B(\cdot)$ denote a pair of transformations. Then we add the new pairs to the training dataset. $A(\cdot)$ and $B(\cdot)$ are usually seen as a way of simulating domain shift, and the design of $A(\cdot)$ and $B(\cdot)$ is key to performance. In our case, A and B are functions of the classifier $h$, and they capture how the classifier influences the potential response from the decision subjects. This requires the decision maker to have a specific response model in mind when designing the augmented data points.
     \item Learning Disentangled Representations: Instead of forcing the entire model or features to be domain invariant, which is challenging, one can relax this constraint by allowing some parts to be domain-specific, essentially learning disentangled representations. In our induced domain adaptation setting, this means separating features into two sets, one set contains features that are invariant to the deployment of a classifier, and the other set contains features that will be potentially affected by. Then we can decompose a classifier into domain-specific biases and domain-agnostic weights, and only keep the latter when dealing with the unseen induced domain.
\end{itemize}

\section{Proof of Results}
\subsection{Proof of Theorem \ref{thm:transf:sh}}
\label{sec:proof-thm-sh}
\begin{proof}
We first establish two lemmas that will be helpful for bounding the errors of a pair of classifiers. Both are standard results from the domain adaption literature \cite{ben-david2010domain}.
\begin{lemma}
\label{lemma:h-divergence}
For any hypotheses $h, h^\prime\in \mathcal H$ and distributions $\D, \D'$,
\begin{align*}
    |\err{\D}{h,h'} - \err{\D'}{h,h'}|\leq \frac{d_{\H \times \H}(\D, \D')}{2}.
\end{align*}
\end{lemma}
\begin{proof}
Define the-cross prediction disagreement between two classifiers $h,h'$ on a distribution $\D$ as $\errd{\D}(h,h'):=\P_{\D}(h(X)\neq h'(X))$.
By the definition of the $\cal H-$divergence,
\begin{align*}
    d_{\H \times \H} (\D, \D') 
    &= 2 \sup_{h,h' \in \H} \left|\P_{\D} (h(X)\neq h'(X)) - \P_{\D'} (h(X)\neq h'(X))\right|\\
    &= 2 \sup_{h,h' \in \H} \left|\err{\D}{h,h'}- \err{\D'}{h,h'}\right|\\
    &\geq 2\left|\err{\D}{h,h'}- \err{\D'}{h,h'}\right|.
\end{align*}
\end{proof}

Another helpful lemma for us is the well-known fact that the 0-1 error obeys the triangle inequality (see, e.g., \cite{crammer2008learning}):
\begin{lemma}
\label{lemma:triangle-inequality-classification-error}
For any distribution $\D$ over instances and any labeling functions $f_1$, $f_2$, and $f_3$, we have $\err{\D}{f_1,f_2} \leq \err{\D}{f_1,f_3} + \err{\D}{f_2,f_3}$.
\end{lemma}

Denote by $\bar{h}^*$ the \emph{ideal joint hypothesis}, which minimizes the combined error:
\begin{align*}
    \bar{h}^* := \argmin_{h' \in \H} \err{\D(h)}{h'} + \err{\D_S}{h'}
\end{align*}
We have:
\begin{align*}
    \err{\D(h)}{h} &\leq \err{\D(h)}{\bar{h}^*} + \errd{\D(h)}(h,\bar{h}^*) \tag{Lemma \ref{lemma:triangle-inequality-classification-error}}\\
   % &=\err{\D(h)}{\bar{h}^*} +  \errd{\D_S}(h,\bar{h}^*) + \errd{\D(h)}(h,\bar{h}^*)-\errd{\D_S}(h,\bar{h}^*) \tag{\text{add and subtract} $\errd{\D_S}(h,\bar{h}^*)$}\\
    &\leq\err{\D(h)}{\bar{h}^*} +  \errd{\D_S}(h,\bar{h}^*) + \left|\errd{\D(h)}(h,\bar{h}^*)-\errd{\D_S}(h,\bar{h}^*)\right|\\
    &\leq \err{\D(h)}{\bar{h}^*} + \err{\D_S}{h} + \err{\D_S}{\bar{h}^*} + \frac{1}{2}d_{\H \times \H}(\D_S, \D(h)) \tag{\text{Lemma \ref{lemma:h-divergence}}}\\
    &= \err{\D_S}{h} + \lambda_{\D_S\rightarrow \D(h)} +\frac{1}{2}d_{\H \times \H}(\D_S, \D(h)) \tag{Definition of $\bar{h}^*$}.
\end{align*}
\end{proof}

\subsection{Proof of Theorem \ref{thm:transf:sh2}}
\label{sec:proof-thm-sh2}
\begin{proof}
Invoking Theorem \ref{thm:transf:sh}, and replacing $h$ with $h^*_T$ and $S$ with $ \D(h^*_T)$, we have
\begin{align}
    \err{\D(h)}{h^*_T} \leq \err{\D(h^*_T)}{h^*_T} + \lambda_{\D(h) \rightarrow \D(h^*_T)} +\frac{1}{2}d_{\H \times \H}(\D(h^*_T), \D(h)) \label{eqn:thm1}
\end{align}

Now observe that
\begin{align*}
    \err{\D(h)}{h} &\leq \err{\D(h)}{h^*_T} +  \errd{\D(h)}(h,h^*_T) \\
    &\leq\err{\D(h)}{h^*_T} +  \errd{\D(h^*_T)}(h,h^*_T) + \left|\errd{\D(h)}(h,h^*_T) - \errd{\D(h^*_T)}(h,h^*_T)\right| \\
    &\leq\err{\D(h)}{h^*_T} +  \errd{\D(h^*_T)}(h,h^*_T) + \frac{1}{2}d_{\H \times \H}(\D(h^*_T), \D(h)) \tag{by Lemma \ref{lemma:h-divergence}} \\
    &\leq \err{\D(h)}{h^*_T} + \err{\D(h^*_T)}{h} + \err{\D(h^*_T)}{h^*_T} + \frac{1}{2}d_{\H \times \H}(\D(h^*_T), \D(h)) \tag{by Lemma \ref{lemma:triangle-inequality-classification-error}} \\
    &\leq \err{\D(h^*_T)}{h^*_T} + \lambda_{\D(h) \rightarrow \D(h^*_T)} +\frac{1}{2}d_{\H \times \H}(\D(h^*_T), \D(h)) \tag{\text{by \eqref{eqn:thm1}}}\\
    &\quad +\err{\D(h^*_T)}{h} + \err{\D(h^*_T)}{h^*_T} + \frac{1}{2}d_{\H \times \H}(\D(h^*_T), \D(h))
\end{align*}
Adding $\err{\D(h)}{h}$ to both sides and rearranging terms yields
\begin{align*}
    2\err{\D(h)}{h} - 2\err{\D(h^*_T)}{h^*_T} &\leq \err{\D(h)}{h} + \err{\D(h^*_T)}{h}+\lambda_{\D(h) \rightarrow \D(h^*_T)} + d_{\H \times \H}(\D(h^*_T), \D(h)) \\
    &= \Lambda_{\D(h) \rightarrow \D(h^*_T)}(h)+ \lambda_{\D(h) \rightarrow \D(h^*_T)} + d_{\H \times \H}(\D(h^*_T), \D(h))
\end{align*}
Dividing both sides by 2 completes the proof.
\end{proof}
\subsection{Proof of Theorem \ref{thm:lb}}
\label{sec:proof-thm-lb}

\begin{proof}
Using the triangle inequality of $\dtv$, we have
\begin{align}
\label{eqn:triangle-inequality-dtv}
    \dtv(\D_{Y \vert S},\D_{Y}(h))
    \leq \dtv(\D_{Y \vert S},\D_{h \vert S})
        + \dtv(\D_{h \vert S},\D_{h}(h))
        + \dtv(\D_{h}(h),\D_{Y}(h))
\end{align}
and by the definition of $\dtv$, the divergence term $\dtv(\D_{Y \vert S},\D_{Y}(h))$ becomes
\begin{align*}
    \dtv(\D_{Y \vert S},\D_{h \vert S})
    &= \left|\P_{\D_S}(Y = +1 ) - \P_{\D_S}( h(x) = +1)\right|\\
    &= \left|\frac{\E_{\D_S}[Y]+1}{2} - \frac{\E_{\D_S}[h(X)]+1}{2}\right| \\
    &= \left|\frac{\E_{\D_S}[Y]}{2} - \frac{\E_{\D_S}[h(X)]}{2}\right| \\
    &\leq \frac{1}{2} \cdot \E_{\D_S}\left[|Y-h(X)|\right] \\
    &= \err{\D_S}{h}
\end{align*}
Similarly, we have
\begin{align*}
    \dtv(\D_{h}(h),\D_{Y}(h)) \leq \err{\D(h)}{h}
\end{align*}

As a result, we have
\begin{align*}
    \err{\D_S}{h}+\err{\D(h)}{h} 
    \geq & \  \dtv(\D_{Y \vert S},\D_{h \vert S}) + \dtv(\D_{h}(h),\D_{Y}(h)) \\
    \geq & \ \dtv(\D_{Y \vert S},\D_{Y}(h)) - \dtv(\D_{h \vert S},\D_{h}(h))   \tag{by \eqref{eqn:triangle-inequality-dtv}}
\end{align*}

which implies
\[
    \max\{\err{\D_S}{h}, \err{\D(h)}{h} \} \geq  \frac{\dtv(\D_{Y \vert S},\D_{Y}(h)) - \dtv(\D_{h \vert S}, \D_{h}(h))}{2}~.
\]
\end{proof}
\subsection{Proof of \Cref{prop:reweight-loss-induced-distribution}}
\label{sec:proof-reweight-prop}
\begin{proof}
\vspace{-0.05in}
\begin{align*}
     & \E_{\D(h)}[\ell(h;X,Y)] \\
     =& \int  \P_{\D(h)}(X=x,Y=y) \ell(h;x,y)~dx dy\\
     =& \int   \P_{\D_S}(Y=y|X=x) \cdot  \P_{\D(h)}(X=x) \ell(h;x,y)~dx dy\\
     =& \int   \P_{\D_S}(Y=y|X=x) \cdot  \P_{\D_S}(X=x) \cdot \frac{\P_{\D(h)}(X=x)}{\P_{\D_S}(X=x)} \cdot \ell(h;x,y)~dx dy\\
      =& \int   \P_{\D_S}(Y=y|X=x) \cdot  \P_{\D_S}(X=x) \cdot \omega_x(h) \cdot \ell(h;x,y)~dx dy\\
      =& \E_{\D_S}[\omega_x(h) \cdot \ell(h;x,y)]
\end{align*}
\end{proof}

\subsection{Proof of Theorem \ref{thm:cs-upper-bound}}
\label{sec:proof-cs-ub}
\begin{proof}
We start from the error induced by $h^*_S$. Let the \emph{average importance weight induced by $h^*_S$} be $\bar{\omega}(h^*_S) = \E_{\D_S}[\omega_{x}(h^*_S)]$; we add and subtract this from the error:
\begin{align*}
    \err{\D(h^*_S)}{h^*_S} &= \E_{\D_S}\left[\omega_{x}(h^*_S) \cdot \Indicator(h^*_S(x)\neq y)\right]\\
    &= \E_{\D_S}\left[\bar{\omega}(h^*_S)  \cdot \Indicator(h^*_S(x)\neq y)\right]+ \E_{\D_S}\left[(\omega_{x}(h^*_S) -\bar{\omega}(h^*_S) ) \cdot \Indicator(h^*_S(x)\neq y)\right]
\end{align*}
In fact, $\bar{\omega}(h^*_S)  = 1$, since
\begin{align*}
   \bar{\omega}(h^*_S)  =& \E_{\D_S}[\omega_{x}(h_S^*) ]
   = \int \omega_{x}(h_S^*)  \P_{\D_S}(X = x)dx
   \\
   &= \int \frac{\P_{\D(h)}(X = x)}{\P_{\D_S}(X = x)} \P_{\D_S}(X = x)dx
   = \int {\P_{\D(h)}(X = x)} dx
   = 1
\end{align*}

Now consider any other classifier $h$. We have
\begin{align*}
    & \err{\D(h^*_S)}{h^*_S} \\
    =& \ \E_{\D_S}\left[ \Indicator(h^*_S(x)\neq y)\right]
        + \E_{\D_S}\left[(\omega_{x}(h^*_S) - \bar{\omega}(h^*_S) ) \cdot \Indicator(h^*_S(x)\neq y)\right] \\
    \leq& \ \E_{\D_S}\left[ \Indicator(h(x)\neq y)\right]
        + \E_{\D_S}\left[(\omega_{x}(h^*_S) -\bar{\omega}(h^*_S) ) \cdot \Indicator(h^*_S(x)\neq y)\right]
        && \tag{by optimality of $h^*_{S}$ on $\D_S$} \\
    =& ~ \E_{\D_S}\left[\bar{\omega}(h) \cdot \Indicator(h(x)\neq y)\right]
        + \E_{\D_S}\left[(\omega_{x}(h^*_S) -\bar{\omega}(h^*_S) ) \cdot \Indicator(h^*_S(x)\neq y)\right]
        && \tag{multiply by $\bar{\omega}(h^*_S)  = 1$} \\
    =& \ \E_{\D_S}\left[\omega_{x}(h) \cdot \Indicator(h(x)\neq y)\right]
        + \E_{\D_S}\left[(\bar{\omega}(h)-\omega_{x}(h)) \cdot \Indicator(h(x)\neq y)\right]
        && \tag{add and subtract $\bar{\omega}(h^*_S)$} \\
        &\quad + \E_{\D_S}\left[(\omega_{x}(h^*_S) -\bar{\omega}(h^*_S) ) \cdot \Indicator(h^*_S(x)\neq y)\right] \\
    =& \ \err{\D(h)}{h}
        + \cov(\omega_{x}(h^*_S),\Indicator(h^*_S(x)\neq y))
        - \cov(\omega_{x}(h),\Indicator(h(x)\neq y)) 
\end{align*}

Moving the error terms to one side, we have
\begin{align*}
    & \ \err{\D(h^*_S)}{h^*_S} - \err{\D(h)}{h} \\
    \leq& \ \cov(\omega_{x}(h^*_S),\Indicator(h^*_S(x)\neq y)) -\cov(\omega_{x}(h),\Indicator(h(x)\neq y)) \\
    \leq& ~ \sqrt{\var(\omega_{x}(h^*_S))\cdot \var(\Indicator(h^*_S(x)\neq y))}
        && \tag{$|\cov(X,Y)| \leq \sqrt{\var(X) \cdot \var(Y)}$} \\
        &\quad + \sqrt{\var(\omega_{x}(h))\cdot \var(\Indicator(h(x)\neq y))} \\
    =& \ \sqrt{\var(\omega_{x}(h^*_S))\cdot\err{S}{h^*_S}(1-\err{S}{h^*_S})}
        + \sqrt{\var(\omega_{x}(h) )\cdot\err{\D_S}{h}(1-\err{\D_S}{h})} \\
    \leq& \ \sqrt{\var(\omega_{x}(h^*_S))\cdot\err{S}{h^*_S}}+\sqrt{\var(\omega_{x}(h) ) \cdot \err{\D_S}{h}}
        && \tag{$1-\err{\D_S}{h} \leq 1$} \\
    \leq& \ \sqrt{\err{\D_S}{h}} \cdot \left( \sqrt{\var(\omega_{x}(h^*_S))}
        + \sqrt{\var(\omega_{x}(h) )}\right)
\end{align*}
Since this holds for any $h$, it certainly holds for $h = h^*_T$.
\end{proof}

\subsection{Omitted Assumptions and Proof of Theorem \ref{thm:cs-lower-bound}}
\label{sec:proof-cs-lb}
Denote $X_+(h) = \{x: \omega_x(h)\geq 1\}$ and $X_-(h) = \{x: \omega_x(h) < 1\}$. 
First, we observe that
\begin{align*}
    &\int_{X_+(h)} \P_{\D_S} (X = x) (1- \omega_x(h))dx \\
    + & \int_{X_-(h)} \P_{\D_S} (X = x) (1-\omega_x(h))dx = 0
\end{align*}
This is simply because of $ \int_{x} \P_{\D_S} (X = x) \cdot \omega_x(h) dx =\int_{x} \P_{\D(h)} (X = x)  dx =1$.

\begin{proof}
    Notice that in the setting of binary classification, we can write the total variation distance between $\D_{Y \vert S}$ and $\D_Y(h)$ as the difference between the probability of $Y = +1$ and the probability of $Y = -1$:
    \begin{align}
    &\dtv(\D_{Y \vert S},\D_{Y}(h))\nonumber \\
    =& \left|\P_{\D_S}(Y=+1) - \P_{\D(h)}(Y=+1) \right|\nonumber \\
    =&\left|\int \P_{\D_S}(Y=+1|X=x) \P_{\D_S}(X=x) dx - \int \P_{\D_S}(Y=+1|X=x) \P_{\D_S}(X=x) \omega_{x}(h)  dx \right| \nonumber\\
    =&\left|\int \P_{\D_S}(Y=+1|X=x) \P_{\D_S}(X=x) \cdot (1-\omega_{x}(h)  ) dx \right| \label{eqn:Dtv-Y}
    \end{align}
    Similarly we have
    \begin{align}
    \dtv(\D_{h \vert S},\D_{h}(h)) 
    =\left|\int \P_{\D_S}(h(x)=+1|X=x) \P_{\D_S}(X=x) \cdot (1-\omega_{x}(h)  ) dx \right| \label{eqn:Dtv-h}
    \end{align}
    We can further expand the total variation distance between $\D_{Y \vert S}$ and $\D_{Y}(h)$ as follows:
    \begin{align*}
    % line 1
    &\dtv(\D_{Y \vert S},\D_{Y}(h)) \\
    %line2
    =&\left|\int \P_{\D_S}(Y=+1|X=x) \P_{\D_S}(X=x) \cdot (1-\omega_{x}(h)  ) dx \right|\\
    %line3
    =& \Bigl|\underbrace{\int_{X_+(h) } \P_\D(Y=+1|X=x) \P_{\D_S}(X=x) \cdot (1-\omega_{x}(h)  ) dx}_{\leq 0} \\
    %line 4
    &+ \underbrace{\int_{X_-(h) } \P_{\D_S}(Y=+1|X=x) \P_{\D_S}(X=x) \cdot (1-\omega_{x}(h)  ) dx}_{>0}\Bigr|\\
    %line5
    =& -\int_{X_+(h) } \P_{\D_S}(Y=+1|X=x) \P_{\D_S}(X=x) \cdot (1-\omega_{x}(h)  ) dx \\
    &- \int_{X_-(h) } \P_{\D_S}(Y=+1|X=x) \P_{\D_S}(X=x) \cdot (1-\omega_{x}(h)  ) dx \tag{ by Assumption \ref{as:cs-lb-1}}\\
    =& \int_{X_+(h) } \P_{\D_S}(Y=+1|X=x) \P_{\D_S}(X=x) \cdot (\omega_{x}(h) -1 ) dx\\
    &+ \int_{X_-(h) } \P_{\D_S}(Y=+1|X=x) \P_{\D_S}(X=x) \cdot ( \omega_{x}(h) -1) dx \tag{\text{by \eqref{eqn:Dtv-Y}}}\\
    =& \int \P_{\D_S}(Y=+1|X=x) \P_{\D_S}(X=x) \cdot (\omega_{x}(h) -1 ) dx 
    \end{align*}
    Similarly, by assumption \ref{as:cs-lb-2} and equation \eqref{eqn:Dtv-h}, we have
    \begin{align*}
    &\dtv(\D_{h \vert S},\D_{h}(h)) = \int \P_{\D_S}(h(x)=+1|X=x) \P_{\D_S}(X=x) \cdot (\omega_{x}(h) -1 ) dx 
    \end{align*}
    
    Thus we can bound the difference between $\dtv(\D_{Y \vert S},\D_{Y}(h))$ and $\dtv(\D_{h \vert S},\D_{h}(h))$ as follows:
    \begin{align*}
        &\dtv(\D_{Y \vert S},\D_{Y}(h)) - \dtv(\D_{h \vert S},\D_{h}(h)) \\
        =& \int \P_{\D_S}(Y=+1|X=x) \P_{\D_S}(X=x) \cdot (\omega_{x}(h) -1 ) dx \\
        & \quad - \int \P_\D(h(x)=+1|X=x) \P_{\D_S}(X=x) \cdot (\omega_{x}(h) -1 ) dx \\
        =& \int [\P_{\D_S}(Y=+1|X=x)- \P_{\D_S}(h(x)=+1|X=x)] \P_{\D_S}(X=x) \cdot (\omega_{x}(h) -1 ) dx  \\
        =& \ \E_{\D_S}  [\left(\P_{\D_S}(Y=+1|X=x) - \P_{\D_S}(h(x)=+1|X=x)\right)\left(\omega_{x}(h) -1\right)] \quad \tag{\text{by Assumption \ref{as:cs-lb-3}}}\\
        >& \ \E_{\D_S} [\P_{\D_S}(Y=+1|X=x) - \P_{\D_S}(h(x)=+1|X=x)]\E_{\D_S}[\omega_{x}(h) -1]\\
        =& \ 0
    \end{align*}
    Combining the above with Theorem \ref{thm:lb}, we have
    \begin{align*}
    \max\{\err{\D_S}{h}, \err{\D(h)}{h}\} \geq \frac{ \dtv(\D_{Y \vert S},\D_{Y}(h)) - \dtv(\D_{h \vert S}, \D_{h}(h))}{2}> 0
    \end{align*}
\end{proof}

\subsection{Omitted details for \Cref{sec:strategic-classification}}
\label{sec:proof-strategic-response}

With \Cref{assumption:cost-function}~-~\Cref{assumption:agent-new-feature}, we can further specify the important weight $w_x(h)$ for the strategic response setting:  

\begin{lemma}
  \label{lemma:strategic-classification-w_g(x)}
    Recall the definition for the covariate shift important weight coefficient $\omega_{x}(h): = \frac{\P_{D(h)}(X = x)}{\P_{D_S}(X = x)}$, for our strategic response setting, we have,
    \begin{align}
   w_x(h) = 
    \begin{cases}
    1, & x\in [0, \tau_h - B) \\
    \frac{\tau_h - x}{B}, & x\in [\tau_h - {B} ,\tau_h)\\
    \frac{1}{B}(- x + \tau_h + 2B), & x\in [\tau_h, \tau_h + B)\\
    1, & x \in [\tau_h + B, 1]
    \end{cases}
\end{align}
\end{lemma}

Proof for \Cref{lemma:strategic-classification-w_g(x)}:
\begin{proof}
We discuss the induced distribution $\D(h)$ by cases:
\begin{itemize}
    \item For the features distributed between $[0, \tau_h - B]$: since we assume the agents are rational, under assumption \ref{assumption:cost-function}, agents with feature that is smaller than $[0, \tau_h - B]$ will not perform any kinds of adaptations, and no other agents will adapt their features to this range of features either, so the distribution between $[0, \tau_h - B]$ will remain the same as before. 
    \item For the target distribution between $[\tau_h - B, \tau_h]$ can be directly calculated from assumption \ref{assumption:agent-success-probability}.
    \item For distribution between $[\tau_h, \tau_h + B]$, 
consider a particular feature $x^\star\in [\tau_h, \tau_h + B]$, under \Cref{assumption:agent-new-feature}, we know its new distribution becomes:
\begin{align*}
    \P_{\D(h)} (x = x^\star) &= 1 + \int_{x^\star - B}^{\tau_h} \frac{1 - \frac{\tau_h - z}{B}}{B - \tau_h + z} dz\\
    &= 1 +  \int_{x^\star - B}^{\tau_h} \frac{1}{B}d z\\
    & = \frac{1}{B} (-x^\star + \tau_h + 2B)
\end{align*}
\item For the target distribution between $[\tau_h + B, 1]$: under assumption \ref{assumption:cost-function} and \ref{assumption:agent-new-feature}, we know that no agents will change their feature to this feature region. So the distribution between $[\tau_h + B, 1]$ remains the same as the source distribution.
\end{itemize}

Recall the definition for the covariate shift important weight coefficient $\omega_{x}(h): = \frac{\P_{D(h)}(X = x)}{\P_{D_S}(X = x)}$, the distribution of $\omega_x(h)$ after agents' strategic responding becomes:
\begin{align}
   \omega_x(h) = 
    \begin{cases}
    1, \ \ & x\in [0, \tau_h - {B}) ~\text{and}~  x\in [\tau_h + B, 1]\\
    \frac{\tau_h - x}{B}, & x\in [\tau_h - {B} ,\tau_h)\\
    \frac{1}{B}(- x + \tau_h + 2B), & x\in [\tau_h, \tau_h + B)\\
     0,  & \text{otherwise}
    \end{cases}
\end{align}
\end{proof}

Proof for \Cref{proposition:bound-strategic-response}:
\begin{proof}
According to \Cref{lemma:strategic-classification-w_g(x)}, we can compute the variance of $w_x(h)$ as $\var(w_x(h)) = \E(w_x(h)^2) - \E(w_x(h)^2) = \frac{2}{3}B$. Then plugging it into the general bound for \Cref{thm:cs-upper-bound} gives us the desired result.
\end{proof}

\subsection{Proof of Theorem \ref{thm:ls:ub}}
\label{sec:proof-ls-ub}
\begin{proof}
Defining $p:= \P_{\D_S}(Y=+1)$, {$p(h) = \P_{\D(h)}(Y=+1)$}, we have
\begin{align}
    %\E_{\D(h^*_S)}[\ell(h^*_S;X,Y)] 
    \err{\D(h^*_S)}{h^*_S} \nonumber
    &= {p(h^*_S)} \cdot \err{+}{h^*_S} + (1-{p(h^*_S)}) \cdot \err{-}{h^*_S} \tag{by definitions of ${p(h^*_S)}$, $\err{+}{h^*_S}$, and $\err{-}{h^*_S}$} \nonumber \\
    &= \underbrace{p \cdot \err{+}{h^*_S} + (1-p) \cdot \err{-}{h^*_S}}_{\text{(I)}} + ({p(h^*_S)} - p) [\err{+}{h^*_S} - \err{-}{h^*_S}] \label{eqn:ls-ub}
\end{align}
We can expand (I) as follows:
\begin{align*}
     & \ p \cdot \err{+}{h^*_S} + (1-p) \cdot \err{-}{h^*_S} \\
     \leq& \ p \cdot \err{+}{h^*_T} + (1-p) \cdot \err{-}{h^*_T} \tag{\text{by optimality of $h^*_S$ on $\D_S$}}\\
     =& \ {p(h^*_T)} \cdot \err{+}{h^*_T}  + (1-{p(h^*_T)}) \cdot \err{-}{h^*_T} + (p-{p(h^*_T)}) \cdot [\err{+}{h^*_T}  - \err{-}{h^*_T} ]\\
     =& \ \err{\D(h^*_T)}{h^*_T} + (p-{p(h^*_T)}) \cdot [\err{+}{h^*_T}  - \err{-}{h^*_T} ]~.
\end{align*}
Plugging this back into \eqref{eqn:ls-ub}, we have
\begin{align*}
    \err{\D(h^*_S)}{h^*_S}-\err{\D(h^*_T)}{h^*_T} \leq ({p(h^*_S)} - p) [\err{+}{h^*_S}  - \err{-}{h^*_S} ]+(p-{p(h^*_T)}) \cdot [\err{+}{h^*_T}  - \err{-}{h^*_T} ]
\end{align*}
Notice that
\begin{align*}
    0.5( \err{+}{h}  - \err{-}{h})
    =& \ 0.5 \cdot 1 - 0.5 \cdot \P(h(X)=+1|Y=+1) - 0.5 \cdot \P(h(X)=+1|Y=-1)\\
    =& \ 0.5 - \P_{\D_u}(h(X)=+1)
\end{align*}
where $\D_u$ is a distribution with a uniform prior. Then
\begin{align*}
    &({p(h^*_S)} - p) [\err{+}{h^*_S}  - \err{-}{h^*_S} ] =   2({p(h^*_S)} - p) \cdot (0.5-\P_{\D_u}(h(X)=+1))\\
        &(p-{p(h^*_T)}) [\err{+}{h^*_T}  - \err{-}{h^*_T} ]=  2(p-{p(h^*_T)}) \cdot (0.5-\P_{\D_u}(h(X)=+1))
\end{align*}
Adding together these two equations yields
\begin{align}
    & \ ({p(h^*_S)} - p) [\err{+}{h^*_S}  - \err{-}{h^*_S} ] + (p-{p(h^*_T)}) \cdot [\err{+}{h^*_T}  - \err{-}{h^*_T} ] \nonumber \\
    =& \ 2 ({p(h^*_S)} - p)\cdot \left( 0.5 - \P_{\D_u}(h^*_S(X)=+1) \right) + 2(p-{p(h^*_T)}) \cdot  \left( 0.5 - \P_{\D_u}(h^*_T(X)=+1) \right) \nonumber \\
    =& \ ({p(h^*_S)}-{p(h^*_T)}) -2\left({p(h^*_S)}\P_{\D_u}(h^*_S(X)=+1)- {p(h^*_T)}\P_{\D_u}(h^*_T(X)=+1)\right) \nonumber \\
        &\quad + 2p\cdot \left (\P_{\D_u}(h^*_S(X)=+1) - \P_{\D_u}(h^*_T(X)=+1)\right) \nonumber \\
    \leq& \ |{p(h^*_S)}-{p(h^*_T)}| \cdot \left (1+2|\P_{\D_u}(h^*_S(X)=+1) - \P_{\D_u}(h^*_T(X)=+1)|\right) \nonumber \\
        &\quad + 2p \cdot |\P_{\D_u}(h^*_S(X)=+1) - \P_{\D_u}(h^*_T(X)=+1)|
        \label{eqn:label-shift-ub-1}
\end{align}
Meanwhile,
\begin{align}
    & \ |\P_{\D_u}(h^*_S(X)=+1) - \P_{\D_u}(h^*_T(X)=+1)| \nonumber \\
    \leq& \ 0.5 \cdot |\P_{\D|Y=+1}(h^*_S(X)=+1) - \P_{\D|Y=+1}(h^*_T(X)=+1)| \nonumber \\
        &\quad + 0.5 \cdot |\P_{\D|Y=-1}(h^*_S(X)=+1) - \P_{\D|Y=-1}(h^*_T(X)=+1)| \nonumber \\
    =& \ 0.5 \left (\dtv(\D_+(h^*_S),\D_+(h^*_T)) + \dtv(\D_-(h^*_S),\D_-(h^*_T)\right)
        \label{eqn:label-shift-ub-2}
\end{align}
Combining \eqref{eqn:label-shift-ub-1} and \eqref{eqn:label-shift-ub-2} gives
\begin{align*}
    & \ |{p(h^*_S)} - {p(h^*_T)}| \cdot \left(1+2 \cdot |\P_{\D_u}(h^*_S(X)=+1) - \P_{\D_u}(h^*_T(X)=+1)|\right) \\
        &\quad + 2p \cdot |\P_{\D_u}(h^*_S(X)=+1) - \P_{\D_u}(h^*_T(X)=+1)| \\
    \leq& \ |{p(h^*_S)}-{p(h^*_T)}| \cdot (1+\dtv(\D_+(h^*_S),\D_+(h^*_T))+\dtv(\D_-(h^*_S),\D_-(h^*_T))\\
        &\quad + p \cdot (\dtv(\D_+(h^*_S),\D_+(h^*_T))+\dtv(\D_-(h^*_S),\D_-(h^*_T)) \\
    \leq& \ |{p(h^*_S)}-{p(h^*_T)}| + (1+p) \cdot \left(\dtv(\D_+(h^*_S),\D_+(h^*_T))+\dtv(\D_-(h^*_S),\D_-(h^*_T)\right)~.
\end{align*}
\end{proof}

\subsection{Proof of Theorem \ref{thm:ls:lb}}
\label{sec:proof-ls-lb}
We will make use of the following fact:
\begin{lemma} \label{lemma:label-shift-TPR-FPR}
Under label shift, $\TPR_S(h) = \TPR_h(h)$ and $\FPR_S(h) = \FPR_h(h)$.
\end{lemma}
\begin{proof}
We have
\begin{align*}
    \TPR_h(h) =& \P_{\D(h)}(h(X) = +1|Y = +1)\\
    =& \int \P_{\D(h)}(h(X) = +1, X= x|Y = +1)dx\\
    =& \int \P_{\D(h)}(h(X) = +1| X= x,Y = +1) \P_{\D(h)}(X= x | Y = +1)dx\\
    =& \int \Indicator(h(x) = +1) \P_{\D(h)}(X= x | Y = +1)dx\\
    =& \int \Indicator(h(x) = +1) \P_{\D_S}(X= x | Y = +1)dx \tag{by definition of label shift}\\
    =& \int \P_{\D_S}(h(X) = +1| X= x,Y = +1) \P_{\D_S}(X= x | Y = +1)dx\\
    =& \TPR_S(h)
\end{align*}
The argument for $\TPR_h(h) = \TPR_S(h)$ is analogous.
\end{proof}

Now we proceed to prove the theorem.
\begin{proof}[Proof of Theorem \ref{thm:ls:lb}]
In section \ref{sec:lower-bound} we showed a general lower bound on the maximum of $\err{\D_S}{h}$ and $\err{\D(h)}{h}$:
\begin{align*}
    \max\{\err{\D_S}{h}, \err{\D(h)}{h} \} \geq \frac{\dtv(\D_{Y \vert S},\D_{Y}(h)) - \dtv(\D_{h \vert S}, \D_{h}(h))}{2}
\end{align*}
In the case of label shift, and by the definitions of $p$ and ${p(h)}$,
\begin{align}
    \dtv(\D_{Y \vert S},\D_{Y}(h)) 
    = |\P_{\D_S}(Y = +1) - \P_{\D(h)}(Y = +1)|
    = |p-{p(h)}| \label{eqn:dtv-ls-Y}
\end{align}

In addition, we have
\begin{align}
\label{eqn:ls-dhs}
    \D_{h \vert S} = \P_S(h(X)=+1) = p \cdot \TPR_S(h) + (1-p) \cdot \FPR_S(h)
\end{align}
Similarly
\begin{align}
    \D_{h}(h) &= \P_{\D(h)}(h(X)=+1) \nonumber \\
    &= {p(h)} \cdot \TPR_{h}(h) + (1-{p(h)}) \cdot \FPR_{h}(h) \nonumber \\
    &= {p(h)} \cdot \TPR_{S}(h) + (1-{p(h)}) \cdot \FPR_{S}(h)
        && \text{(by Lemma \ref{lemma:label-shift-TPR-FPR})}
        \label{eqn:ls-dhh}
\end{align}
Therefore 
\begin{align}
    \dtv(\D_{h \vert S}, \D_{h}(h)) 
    =& |\P_{\D_S}(h(X)=+1) - \P_{\D(h)}(h(X)=+1)| \nonumber \\
    =&  |(p - {p(h)}) \cdot \TPR_S(h) + ({p(h)}-p) \cdot \FPR_S(h)| \tag{By \eqref{eqn:ls-dhh} and \eqref{eqn:ls-dhs}}\nonumber \\
    =& |p-{p(h)}| \cdot |\TPR_S(h)-\FPR_S(h)| \label{eqn:dtv-ls-h}
\end{align}
which yields:
\begin{align}
    \dtv(\D_{Y \vert S},\D_{Y}(h)) -  \dtv(\D_{h \vert S}, \D_{h}(h)) =|p-{p(h)}|(1-|\TPR_S(h)-\FPR_S(h)|) \tag{By \eqref{eqn:dtv-ls-Y} and \eqref{eqn:dtv-ls-h}}
\end{align}
completing the proof. 
\end{proof}

\subsection{Proof of Proposition \ref{prop:ls:rd}}
\label{sec:proof-ls-rd}
\begin{proof}
\begin{align}
    &|{p(h^*_S)}-{p(h^*_T)}|  \cdot \frac{1}{ \P_{\D_S}(Y=+1)}\nonumber \\
= & \frac{|(1-\err{\D_S}{h^*_S })\TPR_S(h^*_S)-(1-\err{\D_S}{h^*_T})\TPR_S(h^*_T)| }{(1-\err{\D_S}{h^*_S}) \cdot (1-\err{\D_S}{h^*_T})}\nonumber \\
    \leq & \frac{|\err{\D_S}{h^*_S }-\err{\D_S}{h^*_T}| \cdot |\TPR_S(h^*_S)-\TPR_S(h^*_T)|}{(1-\err{\D_S}{h^*_S}) \cdot (1-\err{\D_S}{h^*_T})} \label{eqn:omega:transfer}
\end{align}
The inequality above is due to Lemma 7  of \cite{liu2015online}. 
\end{proof}

\section{Lower Bound and Example for Target Shift}
\label{sec:ts:appendix}
\subsection{Lower Bound}

Now we discuss lower bounds. Denote by $\TPR_S(h)$ and $\FPR_S(h)$ the true positive and false positive rates of $h$ on the source distribution $\D_S$. We prove the following:
\begin{theorem}
Under target shift, any model $h$ must incur the following error on either the $\D_S$ or $\D(h)$:
\begin{align*}
    & \max\{\err{\D_S}{h}, \err{\D(h)}{h} \}\\
    \geq &
\frac{|p-{p(h)}|\cdot (1-|\TPR_S(h)-\FPR_S(h)|)}{2}  .  
\end{align*}

\end{theorem}
The proof extends the bound of Theorem \ref{thm:lb} by further explicating each of $\dtv(\D_{Y \vert S},\D_Y(h))$, $\dtv(\D_{h|S}$, and $\D_h(h))$ under the assumption of target shift. Since $|\TPR_S(h)-\FPR_S(h)| < 0$ unless we have a trivial classifier that has either $\TPR_S(h)= 1, \FPR_S(h)=0$ or $\TPR_S(h)= 0, \FPR_S(h)=1$, the lower bound is strictly positive. Taking a closer look, the lower bound is determined linearly by how much the label distribution shifts: $p-{p(h)}$. The difference  is further determined by the performance of $h$ on the source distribution through $1-|\TPR_S(h)-\FPR_S(h)|$. For instance, when $\TPR_S(h)>\FPR_S(h)$, the quality becomes $\text{FNR}_S(h)+\FPR_S(h)$, that is the more error $h$ makes, the larger the lower bound will be.

\subsection{Example Using Replicator Dynamics}

Let us instantiate the discussion using a specific fitness function for the replicator dynamics model (Section \ref{sec:example}), which is the prediction accuracy of $h$ for class $+1$:
\begin{align}
\left[\textbf{Fitness of~}Y=+1\right] := \P_{\D_S}(h(X)=+1|Y=+1)
\label{eqn_supp:fitness}
\end{align}
Then we have $\E\left[\textbf{Fitness of~}Y\right] = \err{\D_S}{h}$, and
\begin{align*}
   \frac{ {p(h)}}{\P_{\D_S}(Y=+1)} = \frac{\P_{\D_S}(h(X)=+1|Y=+1)}{\err{\D_S}{h}} 
\end{align*}
Plugging the result back to our Theorem \ref{thm:ls:ub} we have
\begin{proposition}
%\yc{rearrange equations}
%For the special case of , the difference between $\err{\D(h^*_S)}{h^*_S}$ and $\err{\D(h^*_T)}{h^*_T}$ 
Under the replicator dynamics model in Eqn. (\ref{eqn_supp:fitness}), $ |{p(h^*_S)}-{p(h^*_T)}|$ further bounds as:
\begin{align*}
  &   |{p(h^*_S)}-{p(h^*_T)}|\leq \P_{\D_S}(Y=+1) \\
    \cdot &\frac{|\err{\D_S}{h^*_S }-\err{\D_S}{h^*_T}| \cdot |\TPR_S(h^*_S)-\TPR_S(h^*_T)|}{\err{\D_S}{h^*_S } \cdot \err{\D_S}{h^*_T}}.
\end{align*}

% \begin{align*}
%     &\err{\D(h^*_S)}{h^*_S} - \err{\D(h^*_T)}{h^*_T}\\
%     \leq & \P_{\D_S}(Y=+1) \\
%     \cdot &\frac{|\err{\D_S}{h^*_S }-\err{\D_S}{h^*_T}| \cdot |\TPR_S(h^*_S)-\TPR_S(h^*_T)|}{\err{\D_S}{h^*_S } \cdot \err{\D_S}{h^*_T}}\\
%     + &(1+p) \cdot \left(\dtv(\D_+(h^*_S),\D_+(h^*_T))+\dtv(\D_-(h^*_S),\D_-(h^*_T)\right)
% \end{align*}

\end{proposition}
That is, the difference between $\err{\D(h^*_S)}{h^*_S}$ and $\err{\D(h^*_T)}{h^*_T}$ is further dependent on the difference between the two classifiers' performances on the source data $\D_S$. This offers an opportunity to evaluate the possible error transferability using the source data only. 

\section{Missing Experimental Details}
\label{sec:experiment-details}

\subsection{Synthetic Experiments Using DAG}
\label{sec:synthetic-experiment-details}
Here we provide details in terms of the data-generating process for the simulated dataset.

\paragraph{Covariate Shift}
We specify the causal DAG for covariate shift setting in the following way:
\begin{align*}
    X_1 &\sim \text{Unif}(-1, 1)\\
    X_2&\sim 1.2 X_1 + \mathcal{N}(0, \sigma_2^2)\\
    X_3 &\sim - X_1^2 + \mathcal{N}(0, \sigma_3^2)\\
    Y& := 2\text{sign}(X_2 > 0) - 1
\end{align*}
where $\sigma_2^2$ and $\sigma_3^2$ are parameters of our choices. \\
\textit{Adaptation function}~
We assume the new distribution of feature $X_1'$ will be generated in the following way:
\begin{align*}
     X_1' = \Delta(X) = X_1 + c\cdot (h(X) - 1)
\end{align*}
where $c\in \R^1 > 0$ is the parameter controlling how much the prediction $h(X)$ affect the generating of $X_1'$, namely the magnitude of distribution shift.  
Intuitively, this adaptation function means that if a feature $x$ is predicted to be positive ($h(x) = +1$), then decision subjects are more likely to adapt to that feature in the induced distribution; 
Otherwise, decision subjects are more likely to be moving away from $x$ since they know it will lead to a negative prediction.

\paragraph{Target Shift} We specify the causal DAG for target shift setting in the following way:
\begin{align*}
    (Y+1)/2 &\sim \text{Bernoulli}(\alpha)\\
    X_1|Y = y &\sim \mathcal{N}_{[0,1]}(\mu_y, \sigma^2) \\
    X_2 &= -0.8 X_1 +  \mathcal{N}(0,\sigma_2^2)\\
    X_3 &= 0.2 Y +  \mathcal{N}(0,\sigma_3^2) 
\end{align*}
where $\mathcal{N}_{[0,1]}$ represents a truncated Gaussian distribution taken value between 0 and 1. $\alpha$, $\mu_y$, $\sigma^2$,$\sigma_2^2$ and $\sigma^2_3$ are parameters of our choices. \\
\textit{Adaptation function}~ We assume the new distribution of the qualification $Y'$ will be updated in the following way:
\begin{align*}
    \P(Y' = +1 |h(X) = h, Y = y) = c_{hy}, ~ \text{where} ~  \{h, y\}\in \{-1,+1\}
\end{align*}
where $0 \leq c_{hy}\in \R^1 \leq 1$ represents the likelihood for a person with original qualification $Y = y$ and get predicted as $h(X) = h$ to be qualified in the next step ($Y' = +1$).
% \paragraph{Discussion of the Results} For all four datasets, we do observe positive gaps $\err{D(h_S^*)}{h_S^*} - \err{D(h_T^*)}{h_T^*}$, indicating the suboptimality of training on $\D_S$. The gaps are well bounded by the theoretical results. For lower bound, the empirical observation and the theoretical bounds are roughly within the same magnitude except for one target shift dataset, indicating the effectiveness of our theoretical result. For upper bound, for target shift, the empirical observations are well within the same magnitude of the theoretical bounds while the results for the covariate shift are relatively loose.

\subsection{Synthetic Experiments Using Real-world Data}
On the preprocessed FICO credit score data set \citep{board2007report,hardt2016equality}, we convert the cumulative distribution function (CDF) of TransRisk score among demographic groups (denoted as $A$, including Black, Asian, Hispanic, and White) into group-dependent densities of the credit score.
We then generate a balanced sample where each group has equal representation, with credit scores (denoted as $Q$) initialized by sampling from the corresponding group-dependent density.
The value of attributes for each data point is then updated under a specified dynamics (detailed in \Cref{sec:appendix_dynamic_parameters}) to model the real-world scenario of repeated resource allocation (with decision denoted as $D$).

\subsubsection{Parameters for Dynamics}\label{sec:appendix_dynamic_parameters}
Since we are considering the dynamic setting, we further specify the data generating process in the following way (from time step $T = t$ to $T = t + 1$):
\begin{equation*}
    \begin{split}
        X_{t, 1} &\sim 1.5 Q_t + U[-\epsilon_1, \epsilon_1] \\
        X_{t, 2} &\sim 0.8 A_t + U[-\epsilon_2, \epsilon_2] \\
        X_{t, 3} &\sim A_t + \mathcal{N}(0, \sigma^2) \\
        Y_t &\sim \mathrm{Bernoulli}(q_t) \text{ for a given value of } Q_t = q_t \\
        D_t &= f_t(A_t, X_{t, 1}, X_{t, 2}, X_{t, 3}) \\
        Q_{t + 1} &= \{ Q_t \cdot [ 1 + \alpha_D(D_t) + \alpha_Y(Y_t) ] \}_{(0, 1]} \\
        A_{t + 1} &= A_t \text{ (fixed population)}
    \end{split}
\end{equation*}
where $\{\cdot\}_{(0, 1]}$ represents truncated value between the interval $(0, 1]$, $f_t(\cdot)$ represents the decision policy from input features, and $\epsilon_1, \epsilon_2, \sigma$ are parameters of choices.
{In our experiments, we set $\epsilon_1 = \epsilon_2 = \sigma = 0.1$.}

Within the same time step, i.e., for variables that share the subscript $t$, $Q_t$ and $A_t$ are root causes for all other variables ($X_{t, 1}, X_{t, 2}, X_{t, 3}, D_t, Y_t$).
{At each time step $T = t$, the institution first estimates the credit score $Q_t$ (which is not directly visible to the institution, but is reflected in the visible outcome label $Y_t$) based on $(A_t, X_{t, 1}, X_{t, 2}, X_{t, 3})$, then produces the binary decision $D_t$ according to the optimal threshold (in terms of the accuracy).}

{For different time steps, e.g., from $T = t$ to $T = t + 1$, the new distribution at $T = t + 1$ is induced by the deployment of the decision policy $D_t$.
Such impact is modeled by a multiplicative update in $Q_{t+1}$ from $Q_t$ with parameters (or functions) $\alpha_D(\cdot)$ and $\alpha_Y(\cdot)$ that depend on $D_t$ and $Y_t$, respectively.
In our experiments, we set $\alpha_D = 0.01$ and $\alpha_Y = 0.005$ to capture the scenario where one-step influence of the decision on the credit score is stronger than that for ground truth label.
}

\begin{figure*}[t]
    \centering
    \captionsetup{format=hang}
    \begin{subfigure}{.49\textwidth}
        \centering
        \includegraphics[width=.9\linewidth]{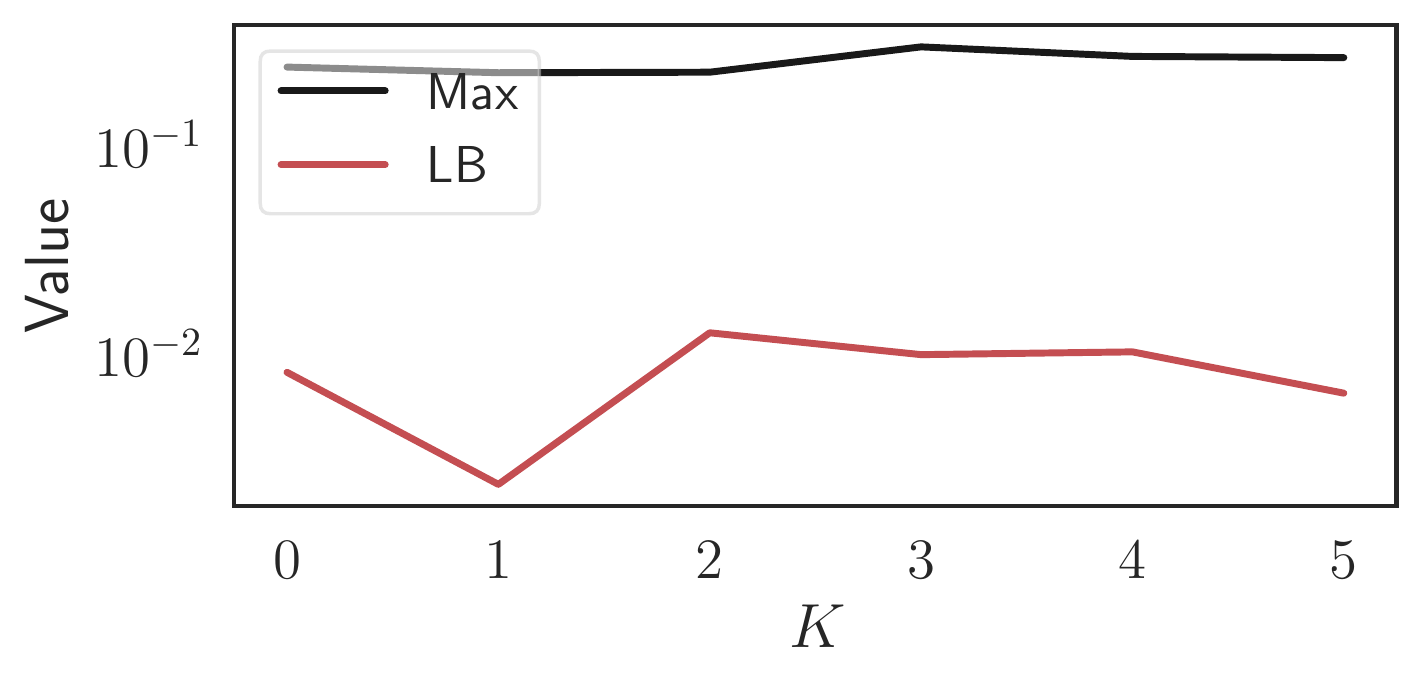}
        \caption{$L1$ penalty, strong regularization strength.}
        \label{fig:appendix_additional_exps_l1_strong_LB}
    \end{subfigure}
    \begin{subfigure}{.49\textwidth}
        \centering
        \includegraphics[width=.9\linewidth]{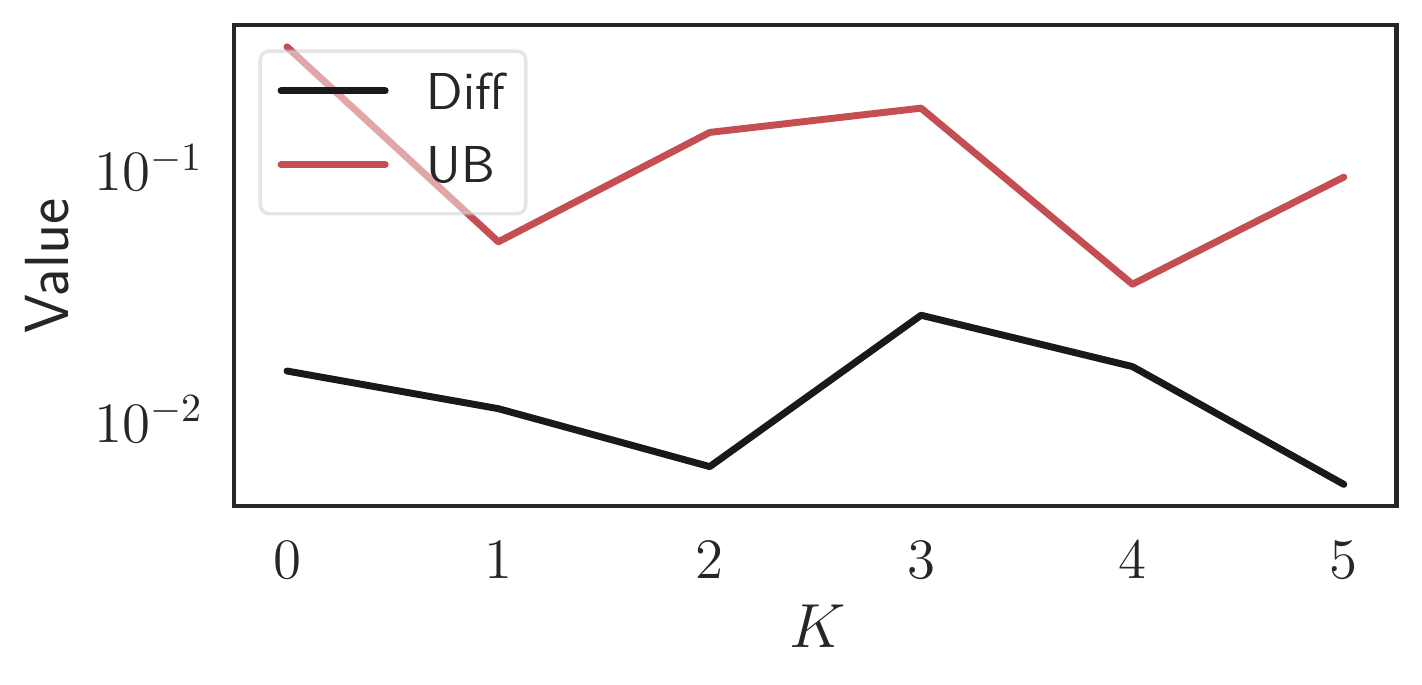}
        \caption{$L1$ penalty, strong regularization strength.}
        \label{fig:appendix_additional_exps_l1_strong_UB}
    \end{subfigure}
    \newline
    \begin{subfigure}{.49\textwidth}
        \centering
        \includegraphics[width=.9\linewidth]{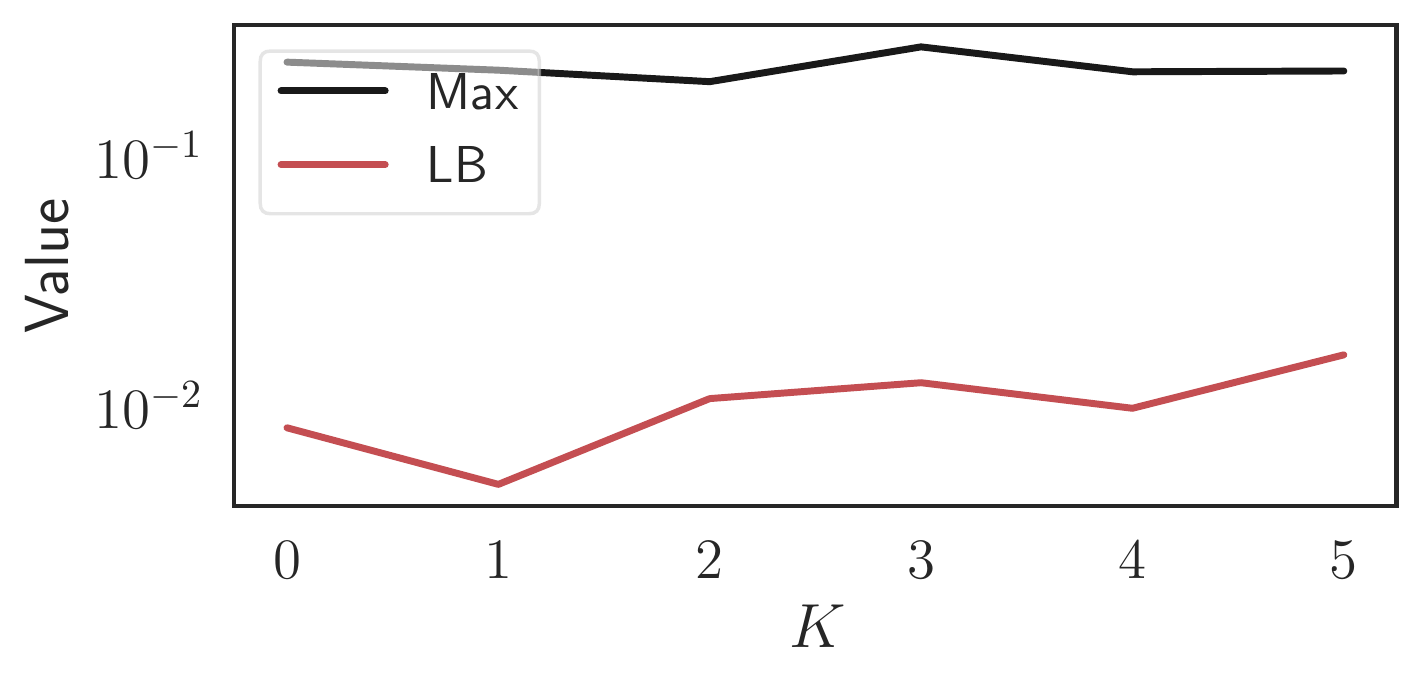}
        \caption{$L1$ penalty, medium regularization strength.}
        \label{fig:appendix_additional_exps_l1_medium_LB}
    \end{subfigure}
    \begin{subfigure}{.49\textwidth}
        \centering
        \includegraphics[width=.9\linewidth]{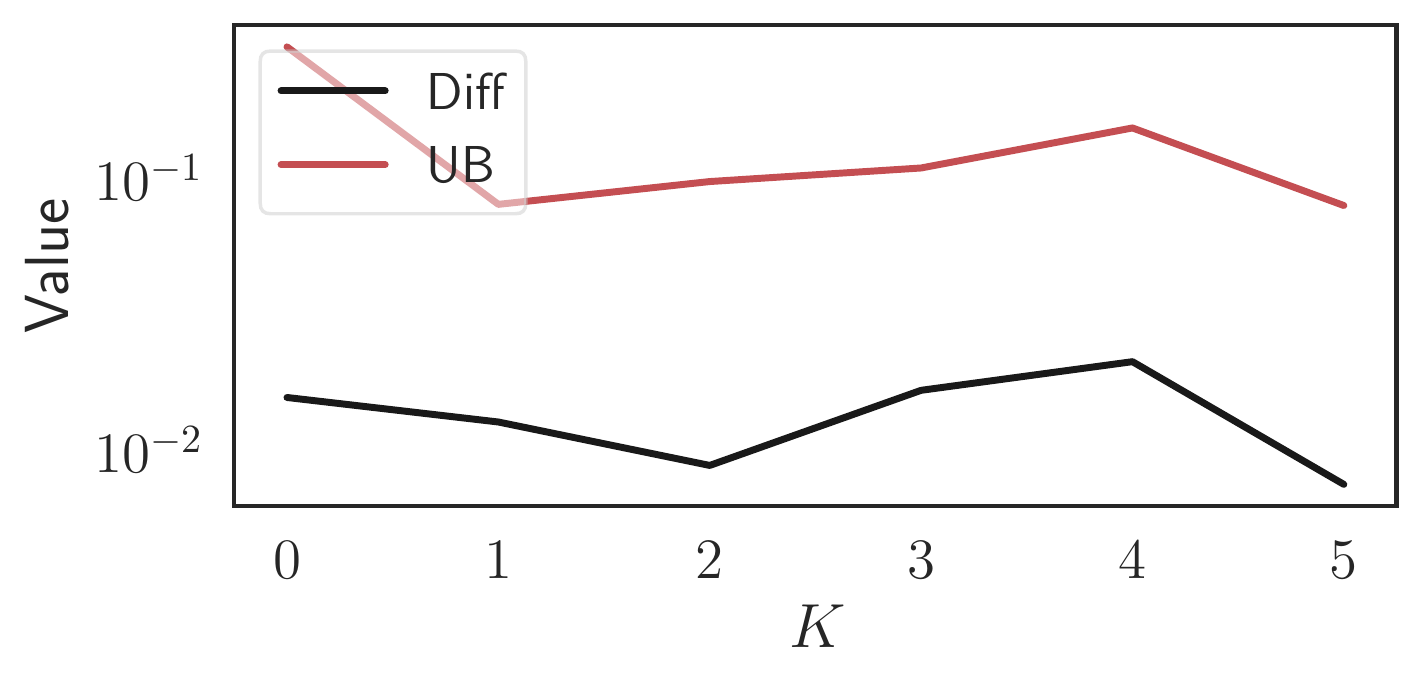}
        \caption{$L1$ penalty, medium regularization strength.}
        \label{fig:appendix_additional_exps_l1_medium_UB}
    \end{subfigure}
    \newline
    \begin{subfigure}{.49\textwidth}
        \centering
        \includegraphics[width=.9\linewidth]{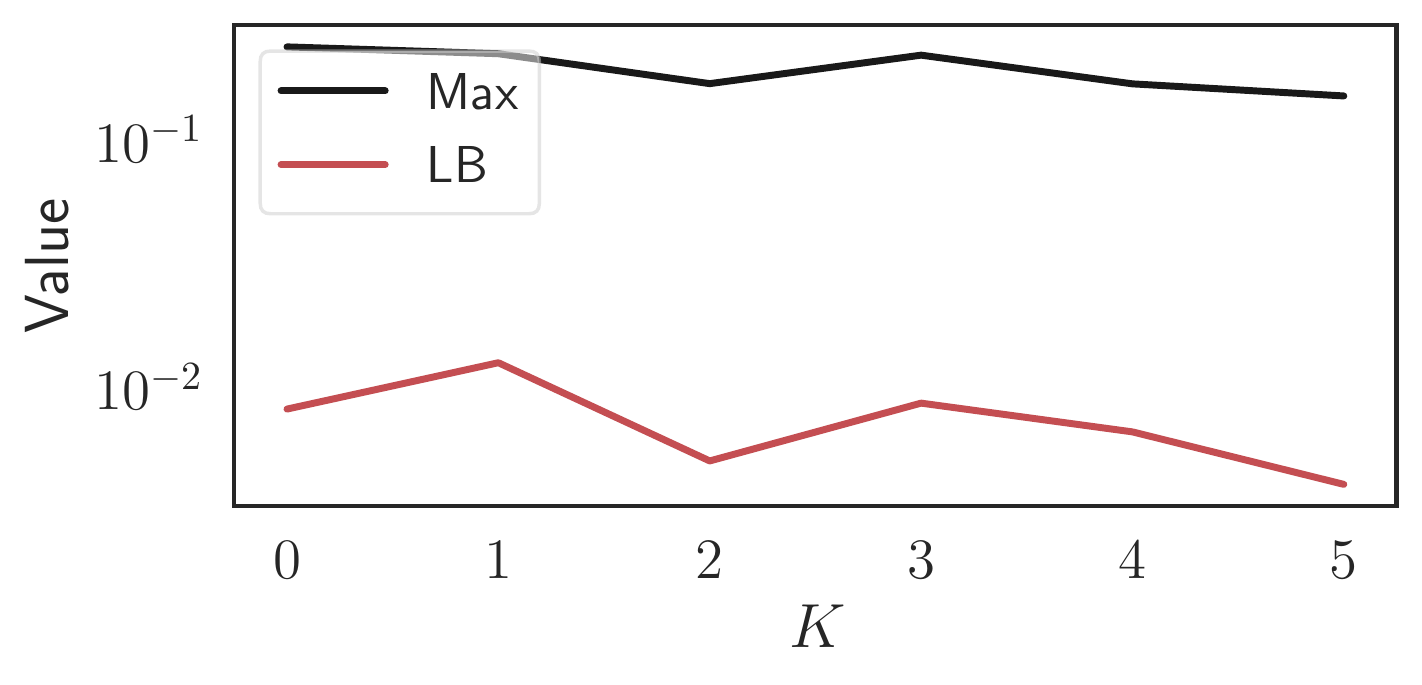}
        \caption{$L1$ penalty, weak regularization strength.}
        \label{fig:appendix_additional_exps_l1_weak_LB}
    \end{subfigure}
    \begin{subfigure}{.49\textwidth}
        \centering
        \includegraphics[width=.9\linewidth]{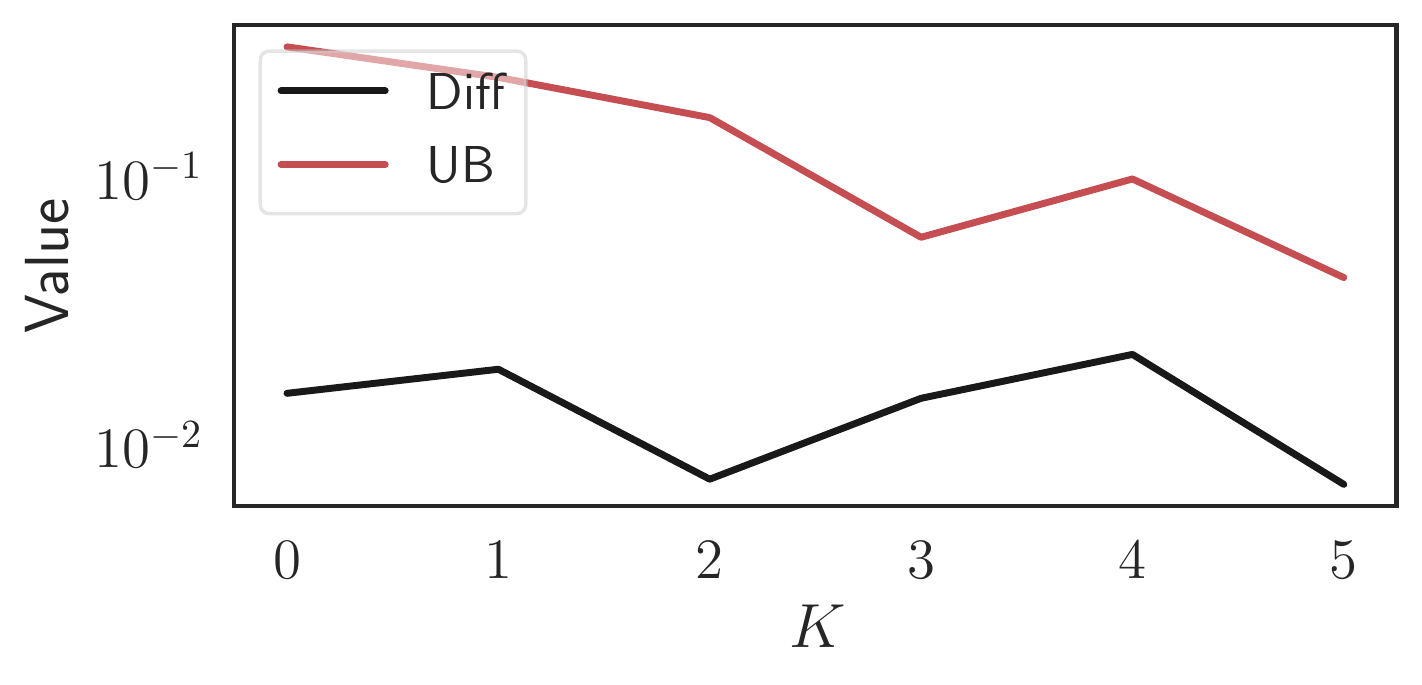}
        \caption{$L1$ penalty, weak regularization strength.}
        \label{fig:appendix_additional_exps_l1_weak_UB}
    \end{subfigure}
    \caption{
    Results of applying $L1$ penalty with different strength when constructing $h_S^*$.
    The left column consisting of panels (a), (c), and (e) compares $\textsf{Max} := \max\{\err{\D_S}{h^*_T}, \err{\D(h^*_T)}{h^*_T}\}$ and $\textsf{LB}:=$ lower bound specified in \Cref{thm:cs-lower-bound}.
    The right column consisting of panels (b), (d), and (f) compares $\textsf{Diff}:= \err{\D(h^*_S)}{h^*_S} - \err{\D(h^*_T)}{h^*_T}$ and $\textsf{UB}:=$ upper bound specified in \Cref{thm:cs-upper-bound}.
    For each time step $K = k$, we compute and deploy the source optimal classifier $h^*_S$ and update the credit score for each individual according to the received decision as the new reality for time step $K = k + 1$.
    }
    \label{fig:appendix_l1_penalty}
\end{figure*}

{\subsubsection{Additional Experimental Results}\label{sec:appendix_additional_exps}
In this section, we present additional experimental results on the real-world FICO credit score data set.
With the initialization of the distribution of credit score $Q$ and the specified dynamics, we present results comparing the influence of vanilla regularization terms in decision-making (when estimating the credit score $Q$) on the calculation of bounds for induced risks.\footnote{The regularization that involves induced risk considerations will be discussed in \Cref{sec:reguarlized-training}.}
In particular, we consider $L1$ norm (\Cref{fig:appendix_l1_penalty}) and $L2$ norm (\Cref{fig:appendix_l2_penalty}) regularization terms when optimizing decision-making policies on the source domain.
As we can see from the results, applying vanilla regularization terms (e.g., $L1$ norm and $L2$ norm) on source domain without specific considerations of the inducing-risk mechanism does not provide significant performance improvement in terms of smaller induced risk.
For example, there is no significant decrease of the term \textsf{Diff} as the regularization strength increases, for both $L1$ norm (\Cref{fig:appendix_l1_penalty}) and $L2$ norm (\Cref{fig:appendix_l2_penalty}) regularization terms.
}

\begin{figure*}[t]
    \centering
    \captionsetup{format=hang}
    \begin{subfigure}{.49\textwidth}
        \centering
        \includegraphics[width=.9\linewidth]{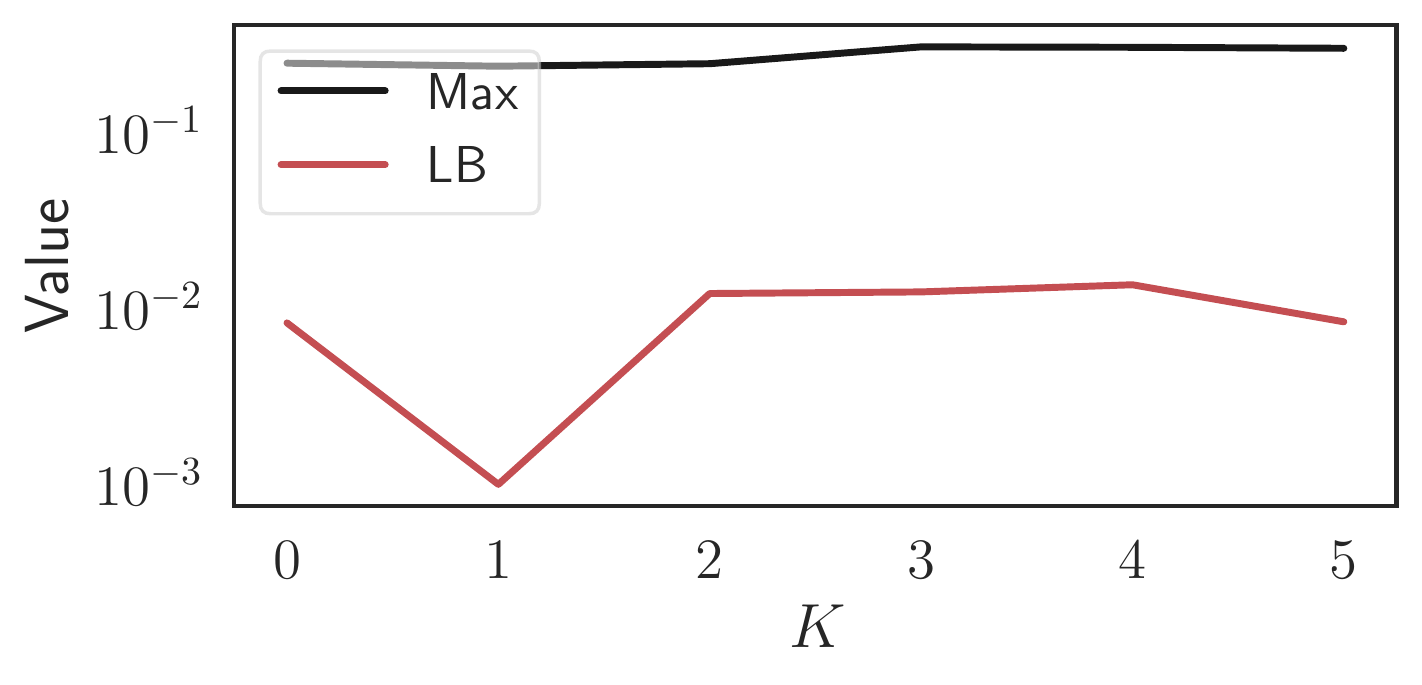}
        \caption{$L2$ penalty, strong regularization strength.}
        \label{fig:appendix_additional_exps_l2_strong_LB}
    \end{subfigure}
    \begin{subfigure}{.49\textwidth}
        \centering
        \includegraphics[width=.9\linewidth]{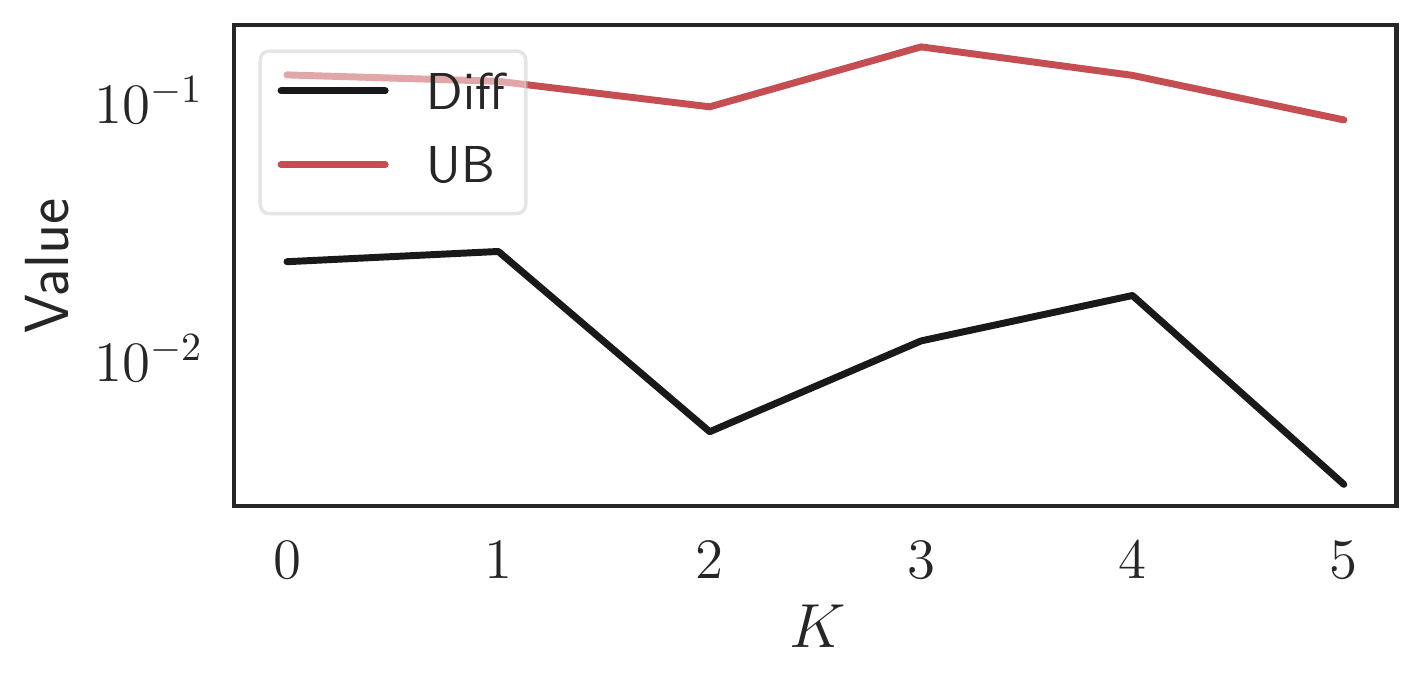}
        \caption{$L2$ penalty, strong regularization strength.}
        \label{fig:appendix_additional_exps_l2_strong_UB}
    \end{subfigure}
    \newline
    \begin{subfigure}{.49\textwidth}
        \centering
        \includegraphics[width=.9\linewidth]{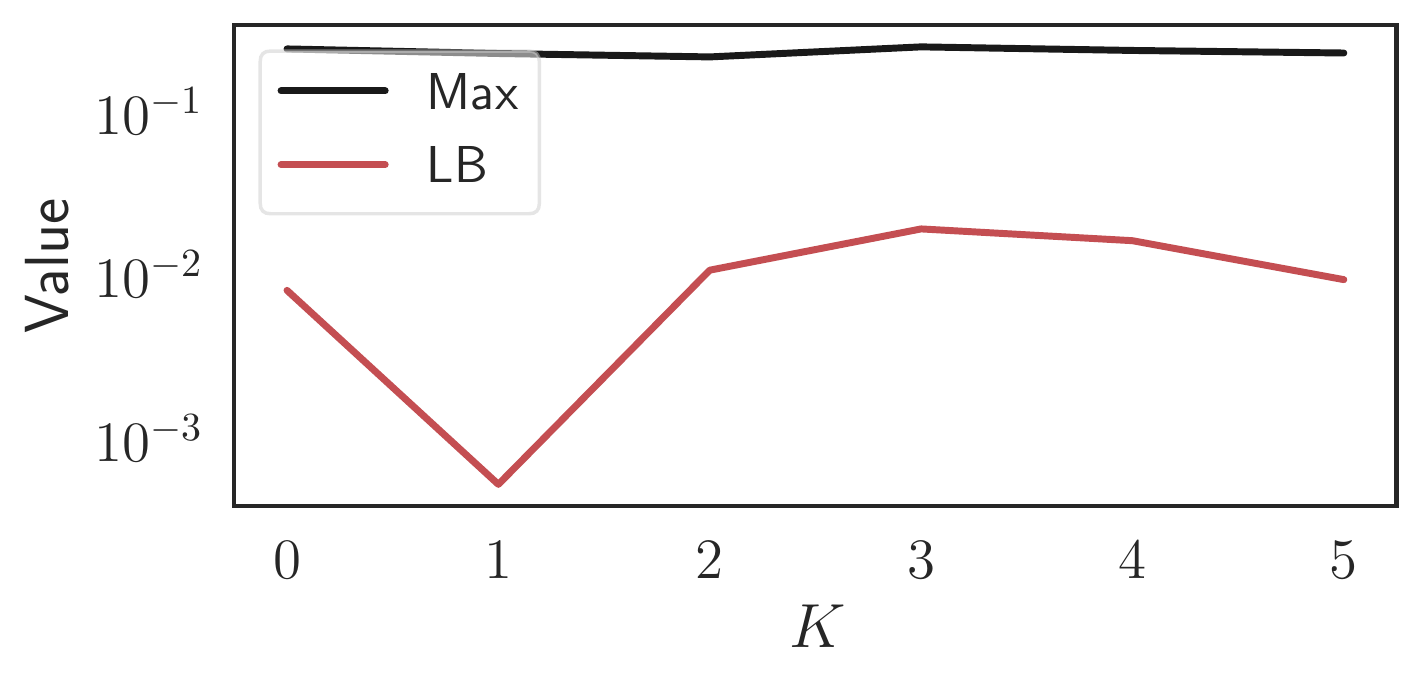}
        \caption{$L2$ penalty, medium regularization strength.}
        \label{fig:appendix_additional_exps_l2_medium_LB}
    \end{subfigure}
    \begin{subfigure}{.49\textwidth}
        \centering
        \includegraphics[width=.9\linewidth]{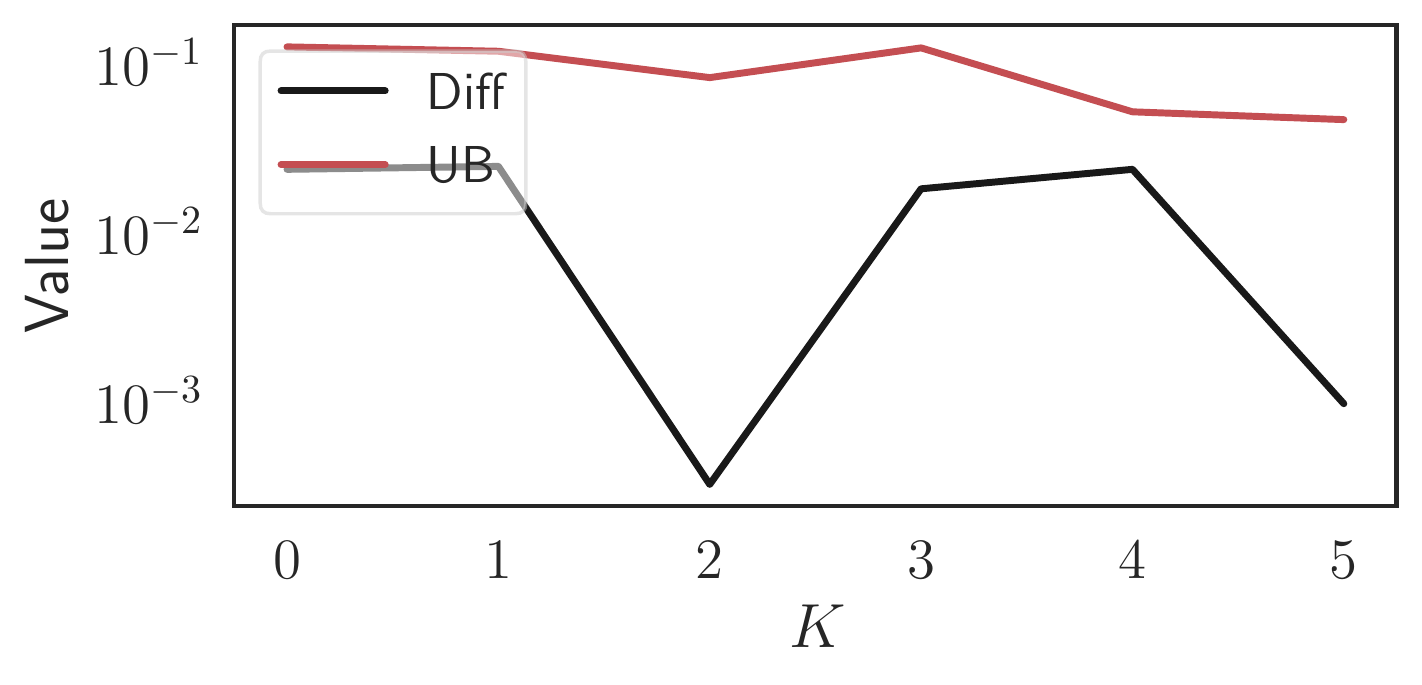}
        \caption{$L2$ penalty, medium regularization strength.}
        \label{fig:appendix_additional_exps_l2_medium_UB}
    \end{subfigure}
    \newline
    \begin{subfigure}{.49\textwidth}
        \centering
        \includegraphics[width=.9\linewidth]{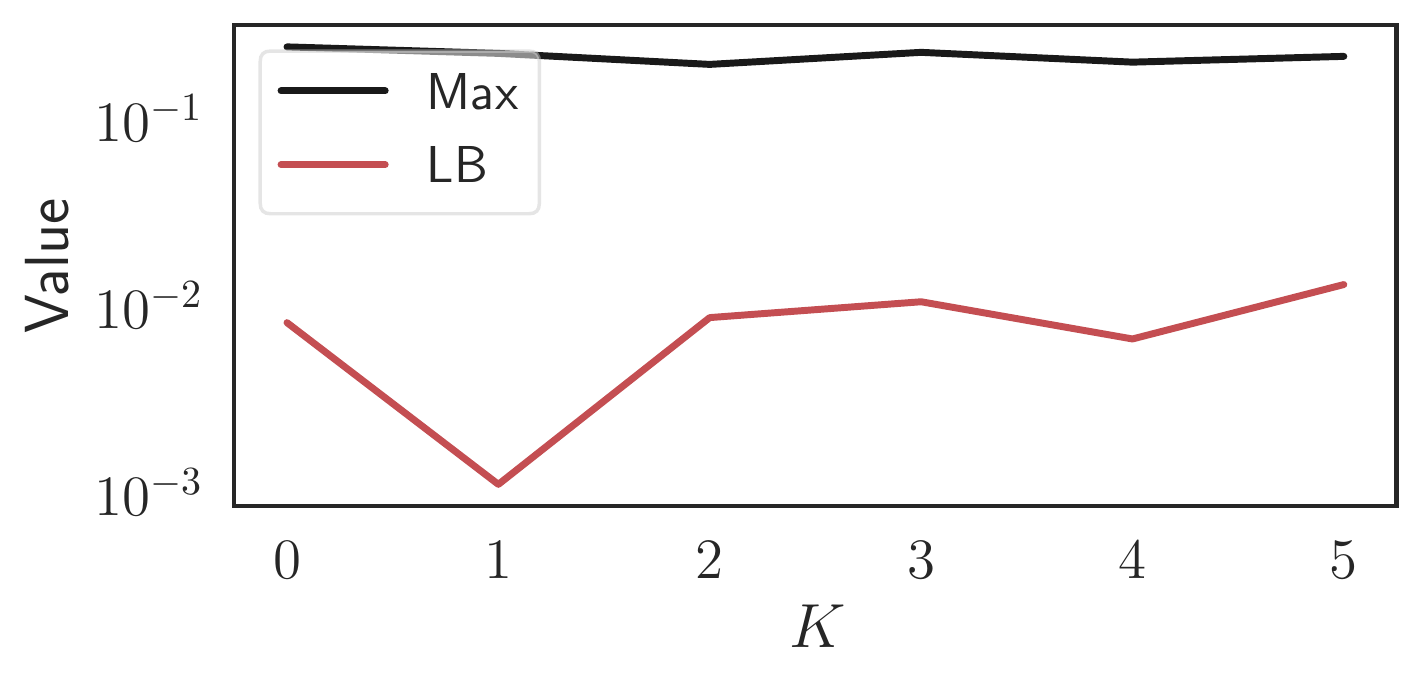}
        \caption{$L2$ penalty, weak regularization strength.}
        \label{fig:appendix_additional_exps_l2_weak_LB}
    \end{subfigure}
    \begin{subfigure}{.49\textwidth}
        \centering
        \includegraphics[width=.9\linewidth]{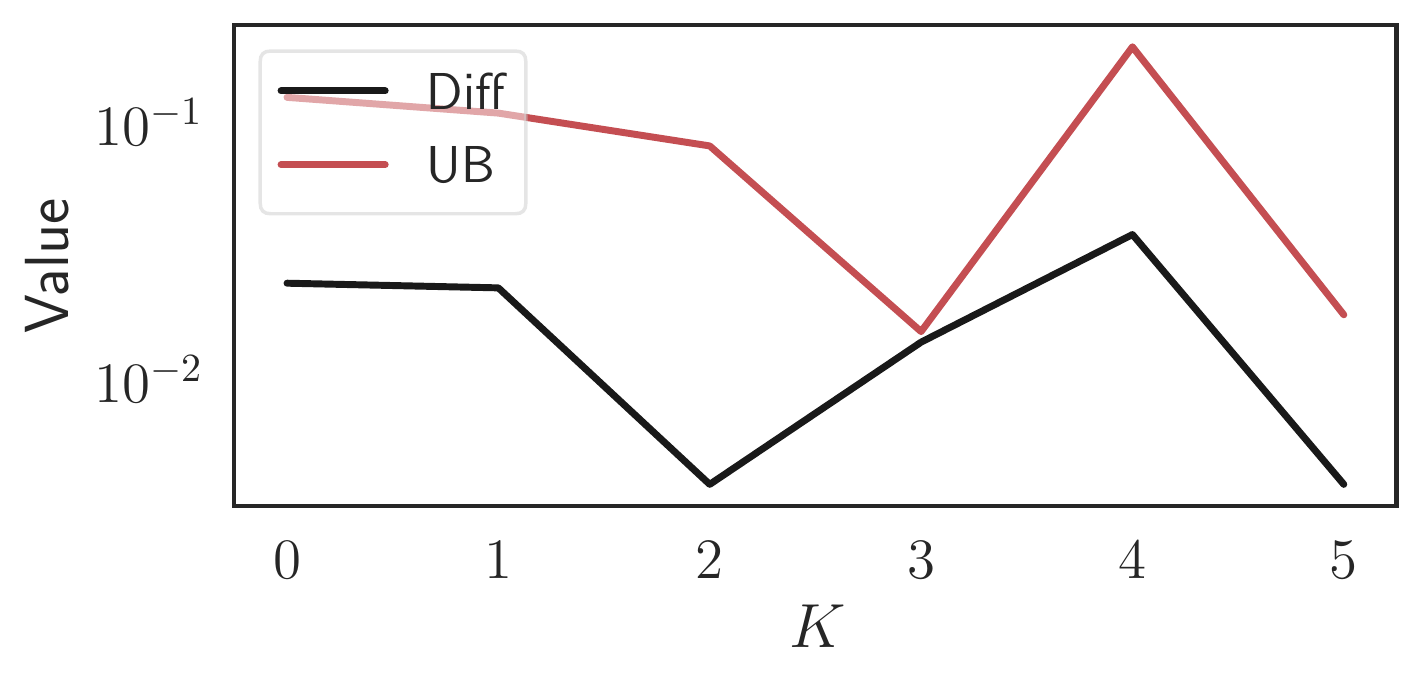}
        \caption{$L2$ penalty, weak regularization strength.}
        \label{fig:appendix_additional_exps_l2_weak_UB}
    \end{subfigure}
    \caption{
    Results of applying $L2$ penalty with different strength when constructing $h_S^*$.
    The left column consisting of panels (a), (c), and (e) compares $\textsf{Max} := \max\{\err{\D_S}{h^*_T}, \err{\D(h^*_T)}{h^*_T}\}$ and $\textsf{LB}:=$ lower bound specified in \Cref{thm:cs-lower-bound}.
    The right column consisting of panels (b), (d), and (f) compares $\textsf{Diff}:= \err{\D(h^*_S)}{h^*_S} - \err{\D(h^*_T)}{h^*_T}$ and $\textsf{UB}:=$ upper bound specified in \Cref{thm:cs-upper-bound}.
    For each time step $K = k$, we compute and deploy the source optimal classifier $h^*_S$ and update the credit score for each individual according to the received decision as the new reality for time step $K = k + 1$.
    }
    \label{fig:appendix_l2_penalty}
\end{figure*}

\section{Challenges in Minimizing Induced Risk}
\label{sec:computation}
In this section, we provide discussion on the challenges in performing induced domain adaptation. 
\subsection{Computational Challenges}
% \begin{itemize}
%     \item Use 1-D label shift example to demonstrate it is a non-convex problem
%     \item Sufficient conditions
%     \item Existing conditions too strong? Use label shift to demonstrate the impossibility results
%     \item A variational lower bound approach
% \end{itemize}
The literature of domain adaptation has provided us solutions to minimize the risk on the target distribution via a nicely developed set of results \cite{sugiyama2008direct,sugiyama2007covariate,shimodaira2000improving}. This allows us to extend the solutions  to minimize the induced risk too. Nonetheless we will highlight additional computational challenges.  

We focus on the covariate shift setting. The scenario for target shift is similar. For covariate shift, recall that earlier we derived the following fact:
\[
    %  &\text{(Importance Reweighting)}:\\
      \E_{\D(h)}[\ell(h;X,Y)] =\E_{\D}[\omega_x(h) \cdot \ell(h;x,y)]
\]
This formula informs us that a promising solution that uses $\omega_x(h) $ to perform reweighted ERM. Of course, the primary challenge that stands in the way is how do we know $\omega_x(h)$.
There are different methods proposed in the literature to estimate $\omega_x(h)$ when one has access to $\D(h)$ \cite{zhang2013domain,long2016unsupervised,gong2016domain}. How any of the specific techniques work in our induced domain adaptation setting will be left for a more thorough future study. %, and in Section \ref{sec:conclude} we will also highlight the additional challenges in our IDA setting. 
In this section, we focus on explaining the computational challenges even when such knowledge of $\omega_x(h)$ can be obtained for each model $h$ being considered during training.

%\subsection{Non-convexity}
%\paragraph{Convexity}
Though $\omega_x(h),\ell(h;x,y)$ might both be convex with respect to (the output of) the classifier $h$, their product is not necessarily convex. Consider the following example:
\begin{example}[$\omega_x(h) \cdot \ell(h;x,y)$ is generally non-convex]
Let $\mathcal{X} = (0, 1]$. Let the true label of each $x\in \mathcal{X}$ be $y(x) = \Indicator\left(x\geq \frac{1}{2}\right)$. Let $\ell(h;x, y) = \frac{1}{2}(h(x) - y)^2$, and let $h(x) = x$ (simple linear model). Notice that $\ell$ is convex in $h$. Let $\D$ be the uniform distribution, whose density function is
$
    f_\D = \begin{cases}
        1, & 0 < x \leq 1 \\
        0, & \text{otherwise}
    \end{cases}
$.
Notice that if the training data is drawn from $\D$, then $h$ is the linear classifier that minimizes the expected loss. Suppose that, since $h$ rewards large values of $x$, it induces decision subjects to shift towards higher feature values. In particular, let $\D(h)$ have density function
\begin{align*}
    f_{\D(h)} = \begin{cases}
        2x, & 0 < x \leq 1 \\
        0,  & \text{otherwise}
    \end{cases}
\end{align*}
Then for all $x \in \mathcal{X}$,
$
    \omega_{x}(h)
    = \frac{f_{\D(h)}(x)}{f_{\mathcal{D}}(x)}
    = 2x.~
$
Notice that $\omega_{x}(h) = 2x$ is convex in $h(x) = x$. Then
\begin{align*}
    \omega_{x}(h) \cdot \ell(h ; x, y)
    &= 2x \cdot \frac{1}{2}(h(x)-y)^{2} \\
    &= x(x-y)^2 = \begin{cases}
        x^3,        & 0 < x < \frac{1}{2} \\
        x(x-1)^2,   & \frac{1}{2} \leq x \leq 1
    \end{cases}
\end{align*}
which is clearly non-convex.
\end{example}

Nonetheless, we provide sufficient conditions under which $\omega_x(h) \cdot \ell(h;x,y)$ is in fact convex:
\begin{proposition}
\label{prop:convex}
Suppose $\omega_x(h)$ and $\ell(h;x,y)$ are both convex in $h$, and $\omega_x(h)$ and $\ell(h;x,y)$ satisfy $ \forall h,h',x,y$: $ (\omega_{x}(h) - \omega_{x}(h'))
        \cdot (\ell(h;x,y) - \ell(h';x,y))
    \geq 0$.
%\begin{align} \label{as:omega-ell-convexity}
%    (\omega_{x}(h) - \omega_{x}(h'))
%        \cdot (\ell(h;x,y) - \ell(h';x,y))
%    \geq 0
%    \quad
%\end{align}
Then $\omega_x(h) \cdot \ell(h;x,y)$ is convex.
\end{proposition}
\begin{proof}
Let us use the shorthand ${\omega(h)} := \omega_x(h)$ and $\ell(h) := \ell(h;x,y)$. To show that $\omega(h) \cdot \ell(h)$ is convex, it suffices to show that for any $\alpha \in [0,1]$ and any two hypotheses $h,h'$ we have
\begin{align*}
    \omega(\alpha \cdot h+(1-\alpha)\cdot h') \cdot \ell(\alpha \cdot h+(1-\alpha)\cdot h') \leq \alpha \cdot \omega(h)  \cdot \ell(h) + (1-\alpha) \cdot \omega(h') \cdot \ell(h')
\end{align*}
By the convexity of $\omega$,
\begin{align*}
    \omega(\alpha \cdot h+(1-\alpha)\cdot h')
    &\leq \alpha \cdot\omega(h)+ (1-\alpha) \cdot \omega(h')
\end{align*}
and by the convexity of $\ell$,
\begin{align*}
    \ell(\alpha \cdot h+(1-\alpha)\cdot h')
    &\leq \alpha \cdot \ell(h) +(1-\alpha)  \cdot \ell(h')
\end{align*}
Therefore it suffices to show that
\begin{align*}
    &\left[\alpha \cdot \omega(h) + (1-\alpha) \cdot \omega(h')\right] \cdot \left[\alpha \cdot \ell(h) + (1-\alpha) \cdot \ell(h')\right]
        - \alpha \cdot \omega(h)\cdot \ell(h) + (1-\alpha) \cdot \omega(h')\cdot \ell(h')
    \leq 0 \\
   & \Leftrightarrow \alpha(\alpha-1) \cdot \omega(h)\ell(h)
        - \alpha(\alpha-1) \cdot [\omega(h)\ell(h') + \omega(h')\ell(h)]
        + \alpha(\alpha-1) \cdot \omega(h')\ell(h')
    \leq 0 \\
   &\Leftrightarrow \alpha(\alpha-1) \cdot [\omega(h) - \omega(h')] \cdot [\ell(h) - \ell(h')]
    \leq 0 \\
    &\Leftrightarrow [\omega(h) - \omega(h')] \cdot [\ell(h) - \ell(h')]
    \geq 0
\end{align*}
By the assumed condition, the left-hand side is indeed non-negative, which proves the claim.
\end{proof}
 This condition is intuitive when each  $x$ belongs to a rational agent who responds to a classifier $h$ to maximize her chance of being classified as $+1$:  For $y=+1$, the higher loss point corresponds to the ones that are close to decision boundary, therefore, more $-1$ negative label points might shift to it, resulting to a larger $\omega_{x}(h)$.
  For $y=-1$, the higher loss point corresponds to the ones that are likely mis-classified as +1, which ``attracts" instances to deviate to. 
  
\subsection{Challenges due to the lack of access to data}\label{sec:conclude}
In the standard domain adaptation settings, one often assumes the access to a sample set of $X$, which already poses challenges when there is no access to label $Y$ after the adaptation. Nonetheless, the literature has observed a fruitful development of solutions \cite{sugiyama2008direct,zhang2013domain,gong2016domain}. 

One might think the above idea can be applied to our IDA setting rather straightforwardly by assuming observing samples from $\D(h)$, the induced distribution under each model $h$ during the training. However, we often do not know precisely how the distribution would shift under a model $h$ until we deploy it. This is particularly true when the distribution shifts are caused by human responding to a model.
Therefore, the ability to ``predict" accurately how samples ``react" to $h$ plays a very important role \cite{ustun2019actionable}. Indeed, the strategic classification literature enables this capability by assuming full rational human agents. For a more general setting,  building robust domain adaptation tools that are resistant to the above ``prediction error" is also going to be a crucial criterion.

\section{Discussions On Performing Direct Induced Risk Minimization}
\label{sec:discussion-minimize-ida}
In this section, we provide discussions on how to directly perform induced risk minimization for our induced domain adaptation setting. We first provide a gradient descent based method for a particular label shift setting where the underlying dynamic is replicator dynamic described in \Cref{sec:replicator-dynamics}. Then we propose a solution for a more general induced domain adaptation setting where we do not make any particular assumptions on the undelying distribution shift model.
\subsection{Gradient descent based method}
Here we provide a toy example of performing direct induced risk minimization under the assumption of label shift with underlying dynamics as the replicator dynamics described in \Cref{sec:replicator-dynamics}. 

% We first describe the setting, and then describe how we can reformulating the induced risk as a function of the classifier $h$ and perform gradient descent to find an optimal solution.   

\paragraph{Setting} Consider a simple setting in which each decision subject is associated with a 1-dimensional continuous feature $x\in \R$ and a binary true qualification $y\in \{-1, +1\}$. We assume label shift setting, and the underlying population dynamic evolves the replicator dynamic setting described in \Cref{sec:replicator-dynamics}. 
We consider a simple threshold classifier, where $\hat{Y} = h(x) = 1[X \geq \theta]$, meaning that the classifier is completely characterized by the threshold parameter $\theta$. Below we will use $\hat{Y}$ and $h(X)$ interchangeably to represent the classification outcome. Recall that the replicator dynamics is specified as follows:
\begin{align}
\label{eqn:replicator}
    \frac{\P_{\D(h)}(Y = y)}{\P_{\D_S}(Y = y)} = \frac{\textbf{Fitness}(Y = y)}{\E_{\D_S}[\textbf{Fitness}(Y)]}
\end{align}
where $\E_{\D_S}[\textbf{Fitness}(Y)] 
    = \textbf{Fitness}(Y = y) \P_{\D_S}(Y = y) + \textbf{Fitness}(Y = -y) (1 - \P_{\D_S}(Y = y))$.
 $\textbf{Fitness}(Y = y)$ is the fitness of strategy $Y = y$, which is further defined in terms of the expected utility $U_{y, \hat{y}}$ of each qualification-classification outcome pair $(y, \hat{y})$:
    \begin{align*}
        \textbf{Fitness}(Y = y):= \sum_{\hat{y}} \P[\hat{Y} = \hat{y} | Y = y] \cdot U_{y, \hat{y}}    
    \end{align*}
where $U_{y, \hat{y}}$ is the utility (or reward) for each qualification-classification outcome combination.$\P(X|Y = y)$ is sampled according to a Gaussian distribution, and will be unchanged since we consider a label shift setting.

We initialize the distributions we specify the initial qualification rate $\P_{\D_S}(Y = +1)$. To test different settings, we vary the specification of the utility matrix $U_{y, \hat{y}}$ and generate different dynamics.

\paragraph{Formulate the induced risk as a function of $h$}

To minimize the induced risk, we first formulate the induced risk as a function of the classifier $h$’s parameter $\theta$ taking into account of the underlying dynamic, and then perform gradient descent to solve for locally optimal classifier $h_T^*$.  

Recall from \Cref{sec:ls}, under label shift, we can rewrite the induced risk as the following form:
\begin{align*}
    \E_{\D(h)}[\ell(h;X,Y)]
    =& {p(h)} \cdot \E_{\D_S}[\ell(h;X,Y)|Y=+1]  + (1-{p(h)}) \cdot \E_{\D_S}[\ell(h;X,Y)|Y=-1] 
\end{align*}
where $p(h) = \P_{\D(h)}(Y = +1)$. 

\begin{figure}[t]
    \centering
    \includegraphics[width=80mm]{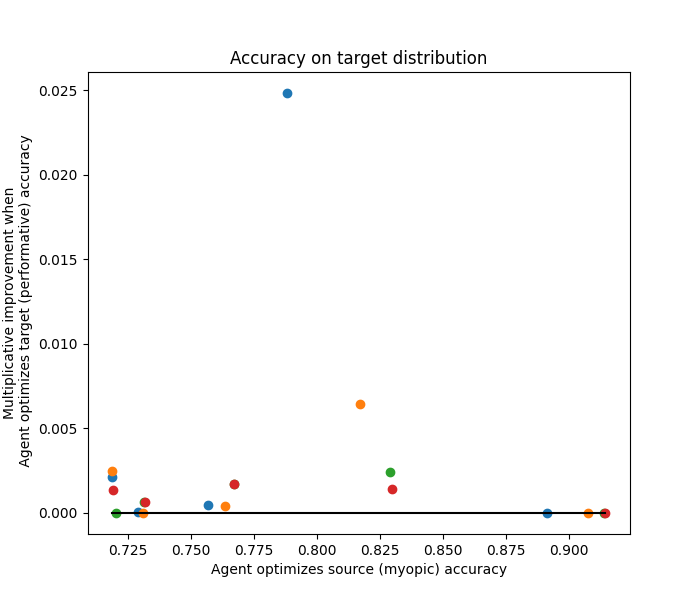}
    \caption{Experimental results of directly optimizing for the induced risk under the assumption of replicator dynamic. The X-axis denotes the prediction accuracy of $\err{D(h^*_S)}{h^*_S}$, where $h^*_S$ is the source optimal classifier under each settings. 
The Y-axis is the percent of performance improvement using the classifier that optimize for $h^*_T = \argmin \err{D(h)}{h}$, which the decision maker considers the underlying response dynamics (according to replicator dynamics in \Cref{eqn:replicator}) of the decision subjects. Different color represents different utility function, which is reflected by the specifications of values in $U_{y, \hat{y}}$; within each color, different dots represent different initial qualification rate. }
    \label{fig:gd_based_IRM}
\end{figure}

Since $\E_{\D_S}[\ell(h;X,Y)|Y=+1]$ and $\E_{\D_S}[\ell(h;X,Y)|Y=-1]$ are already functions of both $h$ and $\D_S$, it suffices to show that the accuracy on $\D(h)$, $p(h) = \P_{\D(h)}(Y = +1)$, can also be expressed as a function of $\theta$ and $\D_S$. 

To see this, recall that for a threshold classifier $\hat{Y} = 1[X > \theta]$, it means that the prediction accuracy can be written as a function of the threshold $\theta$ and target distribution $\D(h)$:
    \begin{align}
    \label{eqn: expression-p(h)}
        & ~~~~ \P_{\D(h)}(Y = +1) \nonumber \\
        &= \P_{\D(h)}(\hat{Y} = +1, Y = +1 ) +\P_{\D(h)}(\hat{Y} = -1, Y = -1 ) \nonumber \\
        &= \P_{\D(h)}( X \geq \theta, Y = +1 ) +\P_{\D(h)}( X \leq \theta,  Y = -1 )  \nonumber \\
        & = \int_{\theta}^\infty \P_{\D(h)} (Y = +1) \underbrace{\P(X = x| Y = 1)}_{\text{unchanged because of label shift}}dx \nonumber \\
        & \ \ \ \ + 
        \int_{-\infty}^\theta \P_{\D(h)} (Y = -1) \underbrace{\P(X = x| Y = -1)}_{\text{unchanged because of label shift}} dx 
    \end{align} 
    where $\P(X | Y = y)$ remains unchanged over time, and $\P_{\D(h)}(Y = y )$ evolves over time according to \Cref{eqn:replicator}, namely
    \begin{align}
        &\P_{\D(h)} (Y = y) \nonumber \\
        =& {\P_{\D_S}(Y = y)} \times \frac{\textbf{Fitness}_g(Y = y)}{\E_{\D_S}[\textbf{Fitness}_g(Y)]}\nonumber \\
        =& {\P_{\D_S}(Y = y)} \times \frac{\sum_{\hat{y}} \P_{\D_S}[\hat{Y} = \hat{y} | Y = y, G = g]\cdot U_{\hat{y}, y}}{\sum_{y} (\sum_{\hat{y}} \P_{\D_S}[\hat{Y} = \hat{y} | Y = y, G = g]\cdot U_{\hat{y}, y} )\P_{\D_S}[Y = y]}
    \end{align}

Notice that $\hat{Y}$ is only a function of $\theta$, and $U_{y, \hat{y}}$ are fixed quantities, the above derivation indicates that we can express $\P_{\D(h)} (Y = y)$ as a function of $\theta$ and $\D_S$. Plugging it back to \Cref{eqn: expression-p(h)}, we can see that the accuracy can also be expressed as a function of the classifier's parameter $\theta$, indicating that the induced risk can be expressed as a function of $\theta$. Thus we can use gradient descent using automatic differentiation w.r.t $\theta$ to find a optimal classifier $h_T^*$ that minimize the induced risk.

\paragraph{Experimental Results}
\Cref{fig:gd_based_IRM} shows the experimental results for this toy example. We can see that for each setting, compared to the baseline classifier $h^*_S$, the proposed gradient based optimization procedure returns us a classifier that achieves a better prediction accuracy (thus lower induced risk) compared to the accuracy of the source optimal classifier.

%%%%%%%%%%%%%%%%%%%%%%%%%%%%%%%%%%%%%%%%%%%%%%%

\subsection{General Setting: Induced Risk Minimization with Bandit Feedback}

In general, finding the optimal classifier that achieves the optimal induced risk $h^*_T$ is a hard problem due to the interactive nature of the problem (see, e.g. the literature of performative prediction \cite{perdomo2020performative} for more detailed discussions).
Without making any assumptions on the mapping between $h$ and $\D(h)$, one can only potentially rely on the \emph{bandit feedbacks} from the decision subjects to estimate the influence of $h$ on $\D(h)$:   
when the induced risk is a convex function of the classifier $h$’s parameter $\theta$, one possible approach is to use the standard techniques from bandit optimization \cite{flaxman2005online} to iteratively find induced optimal classifier $h^*_T$.
The basic idea is: at each step $t = 1, \cdots, T$, the decision maker deploy a classifier $h_t$, then observe data points sampled from $\D(h_t)$ and their losses, and use them to construct an approximate gradient for the induced risk as a function of the model parameter $\theta_t$. When the induced risk is a convex function in the model parameter $\theta$, the above approach guarantees to converge to $h_T^*$, and have sublinear regret in the total number of steps $T$.

The detailed description of the algorithm for finding $h_T^*$ is as follows:

\begin{algorithm}[H]
\caption{One-point bandit gradient descent for performative prediction}
\label{alg:one-point-gradient-descent-for-PP}
\SetAlgoLined
\KwResult{return $\theta_T$ after $T$ rounds}
$\theta_1 \gets 0$ \\
\ForEach{time step $t \gets 1,\ldots,T$}{
    Sample a unit vector $u_t \sim \mathrm{Unif}(\mathbf{S})$ \\
    $\theta_t^+ \gets \theta_t + \delta u_t$ \\
    Observe data points $z_1, \ldots, z_{n_t} \sim \D(\theta_t^+)$ \\
    $\widetilde{\IR}(\theta_t^+) \gets \frac{1}{n_t}\sum_{i=1}^{n_t} \ell(z_i; \theta_t^+)$ \\
    $\tilde{g}_t(\theta_t) \gets \frac{d}{\delta} \widetilde{\IR}(\theta_t^+) \cdot u_t$
        \Comment*[r]{$\tilde{g}_t(\theta_t)$ is an approximation of $\nabla_\theta \widehat{\IR}(\theta_t)$}
    $\theta_{t+1} \gets \Pi_{(1-\delta)\Theta}(\theta_{t} - \eta \tilde{g}_t(\theta_t))$
        \Comment*[r]{Take gradient step; project onto $(1-\delta)\Theta := \{(1-\delta)\theta ~|~ \theta \in \Theta\}$}
}
\end{algorithm}
% \yl{remove $|G=g$. write something like 
% \[
% L(h):= p(h) \cdot \E_{\D_S|Y=+1}[\ell(h;X,Y)] +  (1-p(h)) \cdot \E_{\D_S|Y=-1}[\ell(h;X,Y)]
% \]}
% Perform gradient descent using $\nabla_{h} L(h)$.
% 

\section{Regularized Training}
\label{sec:reguarlized-training}
In this section, we discuss the possibility that indeed minimizing regularized risk will lead to a tighter upper bound.
Consider the target shift setting. Recall that $
    p(h) := \P_{\D(h)}(Y=+1)
$ and we have for any proper loss function $\ell$:
\begin{align*}
    \E_{\D(h)}[\ell(h;X,Y)]   = p(h) \cdot \E_{\D_S}[\ell(h;X,Y)|Y=+1]  + (1-p(h)) \cdot \E_{\D_S}[\ell(h;X,Y)|Y=-1] 
\end{align*}

Suppose $p < {p(h^*_T)}$, now we claim that minimizing the following regularized/penalized risk leads to a smaller upper bound. 
\[
\E_{\D_S}[\ell(h;X,Y)] + \alpha \cdot \E_{\D_{\text{uniform}}}||\frac{h(X)+1}{2}||
\]
where in above $\D_{\text{uniform}}$ is a distribution with uniform prior for $Y$.

We impose the following assumption:
\begin{itemize}
    \item The number of predicted $+1$ for examples with $Y=+1$ and for examples with $Y=-1$ are monotonic with respect to $\alpha$.
\end{itemize}

%\begin{proof}
   % A higher $\alpha$ leads to less predicted $+1$, cause otherwise, 
   Consider the easier setting with $\ell = \text{0-1 loss}$. Then  
   \begin{align*}
   \E_{\D_{\text{uniform}}}||h(X)||
&=0.5 \cdot (\P_{X|Y=+1}[h(X)=+1] + \P_{X|Y=-1}[h(X)=+1]) - 0.5\\
   &=0.5 \cdot (\E_{X|Y=+1}[\ell(h(X),+1)] -  \E_{X|Y=-1}[\ell(h(X),-1])
\end{align*}
The above regularized risk minimization problem is equivalent to 
\[
(p+0.5\cdot \alpha) \cdot \E_{X|Y=+1}[\ell(h(X),+1)] +(p-0.5\cdot \alpha) \cdot \E_{X|Y=-1}[\ell(h(X),-1]
\]
Recall the upper bound in Theorem \ref{thm:ls:ub}:
\begin{align*}
    \err{\D(h^*_S)}{h^*_S}& - \err{\D(h^*_T)}{h^*_T}
    \leq  \underbrace{|p(h^*_S)-p(h^*_T)|}_{\text{Term 1}}\\
    &+ (1+p) \cdot \underbrace{\left(\dtv(\D_+(h^*_S),\D_+(h^*_T))+\dtv(\D_-(h^*_S),\D_-(h^*_T)\right)}_{\text{Term 2}}.
\end{align*}
With a properly specified $\alpha > 0$, this leads to a distribution with a smaller gap of $|p(\tilde{h}_S)-{p(h^*_T)}|$, where $\tilde{h}_S$ denotes the optimal classifier of the penalized risk minimization - this leads to a smaller Term 1 in the bound of Theorem 5.1. Furthermore, the induced risk minimization problem will correspond to an $\alpha$ s.t. $\alpha^* = \frac{{p(h^*_T)} - p}{0.5}$, and the original $h^*_S$ corresponds to a distribution of $\alpha = 0$. Using the monotonicity assumption, we will establish that the second term in Theorem 5.1 will also smaller when we tune a proper $\alpha$.
%\end{proof}

{\section{Discussion on the tightness of our theoretical bounds}
\label{sec:discussion-tightness}
\paragraph{General Bounds in \Cref{sec:general-bound}}
For the general bounds reported in \Cref{sec:general-bound}, it is not trivial to fully quantify the tightness without further quantifying the specific quantities of the terms, e.g. the H divergence of the source and the induced distribution, and the average error a classifier have to incur for both distribution. This part of our results adapted from the classical literature in learning from multiple domains \cite{ben-david2010domain}. The tightness of using $\mathcal H$-divergence and other terms seem to be partially validated therein. 

\paragraph{Bounds in \Cref{sec:cs} and \Cref{sec:ls}}
For more specific bounds provided in \Cref{sec:cs} (for covariate shift) and \Cref{sec:ls} (target shift), however, it is relatively easier to argue about the tightness: the proofs there are more transparent and are easier to back out the conditions where the inequalities are relaxed. For example, in Theorem 5.1, the inequalities of our bound are introduced primarily in the following two places: 1) one is using the optimiality of $h^*_S$ on the source distribution. 2) the other is bounding the statistical difference in $h^*_S$ and $h^*_T$’s predictions on the positive and negative examples. Both are saying that if the differences in the two classifiers’ predictions are bounded in a range, then the result in Theorem 5.1 is relatively tight. }

%% file: main.bbl
\begin{thebibliography}{67}
\providecommand{\natexlab}[1]{#1}
\providecommand{\url}[1]{\texttt{#1}}
\expandafter\ifx\csname urlstyle\endcsname\relax
  \providecommand{\doi}[1]{doi: #1}\else
  \providecommand{\doi}{doi: \begingroup \urlstyle{rm}\Url}\fi

\bibitem[Ali \& Silvey(1966)Ali and Silvey]{ali1966general}
Ali, S.~M. and Silvey, S.~D.
\newblock A general class of coefficients of divergence of one distribution
  from another.
\newblock \emph{Journal of the Royal Statistical Society: Series B
  (Methodological)}, 28\penalty0 (1):\penalty0 131--142, 1966.

\bibitem[Azizzadenesheli et~al.(2019)Azizzadenesheli, Liu, Yang, and
  Anandkumar]{azizzadenesheli2019regularized}
Azizzadenesheli, K., Liu, A., Yang, F., and Anandkumar, A.
\newblock Regularized learning for domain adaptation under label shifts.
\newblock \emph{arXiv preprint arXiv:1903.09734}, 2019.

\bibitem[Ben-David et~al.(2010)Ben-David, Blitzer, Crammer, Kulesza, Pereira,
  and Vaughan]{ben-david2010domain}
Ben-David, S., Blitzer, J., Crammer, K., Kulesza, A., Pereira, F., and Vaughan,
  J.
\newblock A theory of learning from different domains.
\newblock \emph{Machine Learning}, 2010.

\bibitem[{Board of Governors of the Federal Reserve System
  (US)}(2007)]{board2007report}
{Board of Governors of the Federal Reserve System (US)}.
\newblock \emph{Report to the congress on credit scoring and its effects on the
  availability and affordability of credit}.
\newblock Board of Governors of the Federal Reserve System, 2007.

\bibitem[Brown et~al.(2020)Brown, Hod, and Kalemaj]{brown2020performative}
Brown, G., Hod, S., and Kalemaj, I.
\newblock Performative prediction in a stateful world, 2020.

\bibitem[Chakraborty et~al.(2018)Chakraborty, Alam, Dey, Chattopadhyay, and
  Mukhopadhyay]{chakraborty2018adversarial}
Chakraborty, A., Alam, M., Dey, V., Chattopadhyay, A., and Mukhopadhyay, D.
\newblock Adversarial attacks and defences: A survey, 2018.

\bibitem[Chen et~al.(2020{\natexlab{a}})Chen, Liu, and
  Podimata]{chen2020learning}
Chen, Y., Liu, Y., and Podimata, C.
\newblock Learning strategy-aware linear classifiers, 2020{\natexlab{a}}.

\bibitem[Chen et~al.(2020{\natexlab{b}})Chen, Wang, and Liu]{chen2020strategic}
Chen, Y., Wang, J., and Liu, Y.
\newblock Strategic recourse in linear classification.
\newblock \emph{arXiv preprint arXiv:2011.00355}, 2020{\natexlab{b}}.

\bibitem[Crammer et~al.(2008)Crammer, Kearns, and Wortman]{crammer2008learning}
Crammer, K., Kearns, M., and Wortman, J.
\newblock Learning from multiple sources.
\newblock \emph{Journal of Machine Learning Research}, 9\penalty0 (8), 2008.

\bibitem[Dong et~al.(2018{\natexlab{a}})Dong, Roth, Schutzman, Waggoner, and
  Wu]{Dong2018reveal}
Dong, J., Roth, A., Schutzman, Z., Waggoner, B., and Wu, Z.~S.
\newblock Strategic classification from revealed preferences.
\newblock In \emph{Proceedings of the 2018 ACM Conference on Economics and
  Computation}, EC '18, New York, NY, USA, 2018{\natexlab{a}}. Association for
  Computing Machinery.

\bibitem[Dong et~al.(2018{\natexlab{b}})Dong, Roth, Schutzman, Waggoner, and
  Wu]{dong2018strategic}
Dong, J., Roth, A., Schutzman, Z., Waggoner, B., and Wu, Z.~S.
\newblock Strategic classification from revealed preferences.
\newblock In \emph{Proceedings of the 2018 ACM Conference on Economics and
  Computation}, 2018{\natexlab{b}}.

\bibitem[Drusvyatskiy \& Xiao(2020)Drusvyatskiy and
  Xiao]{drusvyatskiy2020stochastic}
Drusvyatskiy, D. and Xiao, L.
\newblock Stochastic optimization with decision-dependent distributions.
\newblock \emph{arXiv preprint arXiv:2011.11173}, 2020.

\bibitem[Flaxman et~al.(2005)Flaxman, Kalai, and McMahan]{flaxman2005online}
Flaxman, A.~D., Kalai, A.~T., and McMahan, H.~B.
\newblock Online convex optimization in the bandit setting: Gradient descent
  without a gradient.
\newblock In \emph{The Sixteenth Annual ACM-SIAM Symposium on Discrete
  Algorithms}, pp.\  385--394, 2005.

\bibitem[Friedman \& Sinervo(2016)Friedman and
  Sinervo]{friedman2016evolutionary}
Friedman, D. and Sinervo, B.
\newblock \emph{Evolutionary games in natural, social, and virtual worlds}.
\newblock Oxford University Press, 2016.

\bibitem[Gong et~al.(2016)Gong, Zhang, Liu, Tao, Glymour, and
  Sch{\"o}lkopf]{gong2016domain}
Gong, M., Zhang, K., Liu, T., Tao, D., Glymour, C., and Sch{\"o}lkopf, B.
\newblock Domain adaptation with conditional transferable components.
\newblock In \emph{International conference on machine learning}, pp.\
  2839--2848. PMLR, 2016.

\bibitem[Goodman \& Flaxman(2017)Goodman and Flaxman]{Goodman2017explain}
Goodman, B. and Flaxman, S.
\newblock European union regulations on algorithmic decision-making and a
  “right to explanation”.
\newblock \emph{AI Magazine}, 38\penalty0 (3):\penalty0 50–57, Oct 2017.

\bibitem[Gretton et~al.(2009)Gretton, Smola, Huang, Schmittfull, Borgwardt, and
  Sch{\"o}lkopf]{gretton2009covariate}
Gretton, A., Smola, A., Huang, J., Schmittfull, M., Borgwardt, K., and
  Sch{\"o}lkopf, B.
\newblock Covariate shift by kernel mean matching.
\newblock \emph{Dataset shift in machine learning}, 3\penalty0 (4):\penalty0 5,
  2009.

\bibitem[Guo et~al.(2020)Guo, Gong, Liu, Zhang, and Tao]{guo2020ltf}
Guo, J., Gong, M., Liu, T., Zhang, K., and Tao, D.
\newblock Ltf: A label transformation framework for correcting label shift.
\newblock In \emph{International Conference on Machine Learning}, pp.\
  3843--3853. PMLR, 2020.

\bibitem[Haghtalab et~al.(2020)Haghtalab, Immorlica, Lucier, and
  Wang]{haghtalab2020maximizing}
Haghtalab, N., Immorlica, N., Lucier, B., and Wang, J.~Z.
\newblock Maximizing welfare with incentive-aware evaluation mechanisms.
\newblock In \emph{Proceedings of the Twenty-Ninth International Joint
  Conference on Artificial Intelligence}. International Joint Conferences on
  Artificial Intelligence Organization, 2020.

\bibitem[Hardt et~al.(2016{\natexlab{a}})Hardt, Megiddo, Papadimitriou, and
  Wootters]{hardt2016strategic}
Hardt, M., Megiddo, N., Papadimitriou, C., and Wootters, M.
\newblock Strategic classification.
\newblock In \emph{Proceedings of the 2016 ACM Conference on Innovations in
  Theoretical Computer Science}, New York, NY, USA, 2016{\natexlab{a}}.
  Association for Computing Machinery.

\bibitem[Hardt et~al.(2016{\natexlab{b}})Hardt, Price, and
  Srebro]{hardt2016equality}
Hardt, M., Price, E., and Srebro, N.
\newblock Equality of opportunity in supervised learning.
\newblock In \emph{Advances in Neural Information Processing Systems}, pp.\
  3315--3323, 2016{\natexlab{b}}.

\bibitem[Hardt et~al.(2022)Hardt, Jagadeesan, and
  Mendler-D{\"u}nner]{hardt2022performative}
Hardt, M., Jagadeesan, M., and Mendler-D{\"u}nner, C.
\newblock Performative power.
\newblock \emph{arXiv preprint arXiv:2203.17232}, 2022.

\bibitem[Huang et~al.(2011)Huang, Joseph, Nelson, Rubinstein, and
  Tygar]{Huang11}
Huang, L., Joseph, A.~D., Nelson, B., Rubinstein, B.~I., and Tygar, J.~D.
\newblock Adversarial machine learning.
\newblock In \emph{ACM Workshop on Security and Artificial Intelligence}, 2011.

\bibitem[Izzo et~al.(2021)Izzo, Ying, and Zou]{izzo2021learn}
Izzo, Z., Ying, L., and Zou, J.
\newblock How to learn when data reacts to your model: performative gradient
  descent.
\newblock In \emph{International Conference on Machine Learning}, pp.\
  4641--4650. PMLR, 2021.

\bibitem[Jagadeesan et~al.(2022)Jagadeesan, Zrnic, and
  Mendler-D{\"u}nner]{jagadeesan2022regret}
Jagadeesan, M., Zrnic, T., and Mendler-D{\"u}nner, C.
\newblock Regret minimization with performative feedback.
\newblock In \emph{International Conference on Machine Learning}, pp.\
  9760--9785. PMLR, 2022.

\bibitem[Jiang(2008)]{jiang2008literature}
Jiang, J.
\newblock A literature survey on domain adaptation of statistical classifiers.
\newblock \emph{URL: http://sifaka. cs. uiuc.
  edu/jiang4/domainadaptation/survey}, 3:\penalty0 1--12, 2008.

\bibitem[Kang et~al.(2019)Kang, Jiang, Yang, and
  Hauptmann]{kang2019contrastive}
Kang, G., Jiang, L., Yang, Y., and Hauptmann, A.~G.
\newblock Contrastive adaptation network for unsupervised domain adaptation.
\newblock In \emph{Proceedings of the IEEE/CVF Conference on Computer Vision
  and Pattern Recognition}, pp.\  4893--4902, 2019.

\bibitem[Kleinberg \& Raghavan(2020)Kleinberg and
  Raghavan]{kleinberg2020classifiers}
Kleinberg, J. and Raghavan, M.
\newblock How do classifiers induce agents to invest effort strategically?
\newblock \emph{ACM Transactions on Economics and Computation (TEAC)},
  8\penalty0 (4):\penalty0 1--23, 2020.

\bibitem[Li et~al.(2017)Li, Yang, Song, and Hospedales]{li2017learning}
Li, D., Yang, Y., Song, Y.-Z., and Hospedales, T.~M.
\newblock Learning to generalize: Meta-learning for domain generalization,
  2017.

\bibitem[Li \& Wai(2022)Li and Wai]{li2022state}
Li, Q. and Wai, H.-T.
\newblock State dependent performative prediction with stochastic
  approximation.
\newblock In \emph{International Conference on Artificial Intelligence and
  Statistics}, pp.\  3164--3186. PMLR, 2022.

\bibitem[Liese \& Vajda(2006)Liese and Vajda]{liese2006divergences}
Liese, F. and Vajda, I.
\newblock On divergences and informations in statistics and information theory.
\newblock \emph{IEEE Transactions on Information Theory}, 52\penalty0
  (10):\penalty0 4394--4412, 2006.

\bibitem[Lipton et~al.(2018)Lipton, Wang, and Smola]{lipton2018detecting}
Lipton, Z., Wang, Y.-X., and Smola, A.
\newblock Detecting and correcting for label shift with black box predictors.
\newblock In \emph{International conference on machine learning}, pp.\
  3122--3130. PMLR, 2018.

\bibitem[Liu \& Liu(2015)Liu and Liu]{liu2015online}
Liu, Y. and Liu, M.
\newblock An online learning approach to improving the quality of
  crowd-sourcing.
\newblock \emph{ACM SIGMETRICS Performance Evaluation Review}, 2015.

\bibitem[Long et~al.(2016)Long, Zhu, Wang, and Jordan]{long2016unsupervised}
Long, M., Zhu, H., Wang, J., and Jordan, M.~I.
\newblock Unsupervised domain adaptation with residual transfer networks.
\newblock \emph{arXiv preprint arXiv:1602.04433}, 2016.

\bibitem[Lowd \& Meek(2005)Lowd and Meek]{Lowd05}
Lowd, D. and Meek, C.
\newblock Adversarial learning.
\newblock In \emph{ACM SIGKDD International Conference on Knowledge Discovery
  in Data Mining}, 2005.

\bibitem[Maheshwari et~al.(2021)Maheshwari, Chiu, Mazumdar, Sastry, and
  Ratliff]{maheshiwar2021minmax}
Maheshwari, C., Chiu, C.-Y., Mazumdar, E., Sastry, S.~S., and Ratliff, L.~J.
\newblock Zeroth-order methods for convex-concave minmax problems: Applications
  to decision-dependent risk minimization, 2021.

\bibitem[Mendler-D\"{u}nner et~al.(2020)Mendler-D\"{u}nner, Perdomo, Zrnic, and
  Hardt]{Mendler2020Stochastic}
Mendler-D\"{u}nner, C., Perdomo, J., Zrnic, T., and Hardt, M.
\newblock Stochastic optimization for performative prediction.
\newblock In \emph{Advances in Neural Information Processing Systems}, pp.\
  4929--4939. Curran Associates, Inc., 2020.

\bibitem[Mendler-D{\"u}nner et~al.(2022)Mendler-D{\"u}nner, Ding, and
  Wang]{mendleranticipating2022}
Mendler-D{\"u}nner, C., Ding, F., and Wang, Y.
\newblock Anticipating performativity by predicting from predictions.
\newblock In \emph{Advances in Neural Information Processing Systems}, 2022.

\bibitem[Miller et~al.(2020)Miller, Milli, and Hardt]{miller2020strategic}
Miller, J., Milli, S., and Hardt, M.
\newblock Strategic classification is causal modeling in disguise.
\newblock In \emph{International Conference on Machine Learning}, pp.\
  6917--6926. PMLR, 2020.

\bibitem[Muandet et~al.(2013)Muandet, Balduzzi, and
  Schölkopf]{muandet2013domain}
Muandet, K., Balduzzi, D., and Schölkopf, B.
\newblock Domain generalization via invariant feature representation, 2013.

\bibitem[Nado et~al.(2021)Nado, Padhy, Sculley, D'Amour, Lakshminarayanan, and
  Snoek]{nado2021evaluating}
Nado, Z., Padhy, S., Sculley, D., D'Amour, A., Lakshminarayanan, B., and Snoek,
  J.
\newblock Evaluating prediction-time batch normalization for robustness under
  covariate shift, 2021.

\bibitem[Narang et~al.(2022)Narang, Faulkner, Drusvyatskiy, Fazel, and
  Ratliff]{narang2022multiplayer}
Narang, A., Faulkner, E., Drusvyatskiy, D., Fazel, M., and Ratliff, L.~J.
\newblock Multiplayer performative prediction: Learning in decision-dependent
  games.
\newblock \emph{arXiv preprint arXiv:2201.03398}, 2022.

\bibitem[Papernot et~al.(2016)Papernot, McDaniel, and
  Goodfellow]{papernot2016transferability}
Papernot, N., McDaniel, P., and Goodfellow, I.
\newblock Transferability in machine learning: from phenomena to black-box
  attacks using adversarial samples, 2016.

\bibitem[Perdomo et~al.(2020)Perdomo, Zrnic, Mendler-D{\"u}nner, and
  Hardt]{perdomo2020performative}
Perdomo, J., Zrnic, T., Mendler-D{\"u}nner, C., and Hardt, M.
\newblock Performative prediction.
\newblock In \emph{International Conference on Machine Learning}, pp.\
  7599--7609. PMLR, 2020.

\bibitem[Piliouras \& Yu(2022)Piliouras and Yu]{piliouras2022multi}
Piliouras, G. and Yu, F.-Y.
\newblock Multi-agent performative prediction: From global stability and
  optimality to chaos.
\newblock \emph{arXiv preprint arXiv:2201.10483}, 2022.

\bibitem[Raab \& Liu(2021)Raab and Liu]{raab2021unintended}
Raab, R. and Liu, Y.
\newblock Unintended selection: Persistent qualification rate disparities and
  interventions.
\newblock \emph{Advances in Neural Information Processing Systems}, 2021.

\bibitem[Selbst \& Powles(2018)Selbst and Powles]{Selbst2018explain}
Selbst, A. and Powles, J.
\newblock ``meaningful information'' and the right to explanation.
\newblock In \emph{Proceedings of the 1st Conference on Fairness,
  Accountability and Transparency}, Proceedings of Machine Learning Research.
  PMLR, 2018.

\bibitem[Shavit et~al.(2020)Shavit, Edelman, and Axelrod]{shavit2020causal}
Shavit, Y., Edelman, B., and Axelrod, B.
\newblock Causal strategic linear regression.
\newblock \emph{International Conference on Machine Learning}, pp.\
  8676--8686, 2020.

\bibitem[Sheth et~al.(2022)Sheth, Moraffah, Candan, Raglin, and
  Liu]{sheth2022domain}
Sheth, P., Moraffah, R., Candan, K.~S., Raglin, A., and Liu, H.
\newblock Domain generalization--a causal perspective.
\newblock \emph{arXiv preprint arXiv:2209.15177}, 2022.

\bibitem[Shimodaira(2000)]{shimodaira2000improving}
Shimodaira, H.
\newblock Improving predictive inference under covariate shift by weighting the
  log-likelihood function.
\newblock \emph{Journal of statistical planning and inference}, 90\penalty0
  (2):\penalty0 227--244, 2000.

\bibitem[Song et~al.(2019)Song, He, Wang, and Hopcroft]{song2019improving}
Song, C., He, K., Wang, L., and Hopcroft, J.~E.
\newblock Improving the generalization of adversarial training with domain
  adaptation, 2019.

\bibitem[Sugiyama et~al.(2007)Sugiyama, Krauledat, and
  M{\"u}ller]{sugiyama2007covariate}
Sugiyama, M., Krauledat, M., and M{\"u}ller, K.-R.
\newblock Covariate shift adaptation by importance weighted cross validation.
\newblock \emph{Journal of Machine Learning Research}, 8\penalty0 (5), 2007.

\bibitem[Sugiyama et~al.(2008)Sugiyama, Suzuki, Nakajima, Kashima, von
  B{\"u}nau, and Kawanabe]{sugiyama2008direct}
Sugiyama, M., Suzuki, T., Nakajima, S., Kashima, H., von B{\"u}nau, P., and
  Kawanabe, M.
\newblock Direct importance estimation for covariate shift adaptation.
\newblock \emph{Annals of the Institute of Statistical Mathematics},
  60\penalty0 (4):\penalty0 699--746, 2008.

\bibitem[Taylor \& Jonker(1978)Taylor and Jonker]{Taylor1978evolutionary}
Taylor, P.~D. and Jonker, L.~B.
\newblock Evolutionary stable strategies and game dynamics.
\newblock \emph{Mathematical Biosciences}, 1978.

\bibitem[Tuyls et~al.(2006)Tuyls, Hoen, and
  Vanschoenwinkel]{tuyls2006evolutionary}
Tuyls, K., Hoen, P.~J., and Vanschoenwinkel, B.
\newblock An evolutionary dynamical analysis of multi-agent learning in
  iterated games.
\newblock \emph{Autonomous Agents and Multi-Agent Systems}, 2006.

\bibitem[Ustun et~al.(2019)Ustun, Spangher, and Liu]{ustun2019actionable}
Ustun, B., Spangher, A., and Liu, Y.
\newblock Actionable recourse in linear classification.
\newblock In \emph{Proceedings of the Conference on Fairness, Accountability,
  and Transparency}, pp.\  10--19, 2019.

\bibitem[Varsavsky et~al.(2020)Varsavsky, Orbes-Arteaga, Sudre, Graham, Nachev,
  and Cardoso]{varsavsky2020testtime}
Varsavsky, T., Orbes-Arteaga, M., Sudre, C.~H., Graham, M.~S., Nachev, P., and
  Cardoso, M.~J.
\newblock Test-time unsupervised domain adaptation, 2020.

\bibitem[Vorobeychik \& Kantarcioglu(2018)Vorobeychik and
  Kantarcioglu]{Vorobeychik18}
Vorobeychik, Y. and Kantarcioglu, M.
\newblock \emph{Adversarial Machine Learning}.
\newblock Morgan \& Claypool Publishers, 2018.

\bibitem[Wang et~al.(2021{\natexlab{a}})Wang, Shelhamer, Liu, Olshausen, and
  Darrell]{wang2021tent}
Wang, D., Shelhamer, E., Liu, S., Olshausen, B., and Darrell, T.
\newblock Tent: Fully test-time adaptation by entropy minimization,
  2021{\natexlab{a}}.

\bibitem[Wang et~al.(2021{\natexlab{b}})Wang, Lan, Liu, Ouyang, Qin, Lu, Chen,
  Zeng, and Yu]{wang2021generalizing}
Wang, J., Lan, C., Liu, C., Ouyang, Y., Qin, T., Lu, W., Chen, Y., Zeng, W.,
  and Yu, P.~S.
\newblock Generalizing to unseen domains: A survey on domain generalization,
  2021{\natexlab{b}}.

\bibitem[Wang et~al.(2005)Wang, Kulkarni, and Verdu]{wang2005divergence}
Wang, Q., Kulkarni, S., and Verdu, S.
\newblock Divergence estimation of continuous distributions based on
  data-dependent partitions.
\newblock \emph{IEEE Transactions on Information Theory}, 51\penalty0
  (9):\penalty0 3064--3074, 2005.
\newblock \doi{10.1109/TIT.2005.853314}.

\bibitem[Xie et~al.(2022)Xie, Wei, Feng, and An]{xie2022gearnet}
Xie, R., Wei, H., Feng, L., and An, B.
\newblock Gearnet: Stepwise dual learning for weakly supervised domain
  adaptation.
\newblock \emph{AAAI Conference on Artificial Intelligence}, 2022.

\bibitem[Zadrozny(2004)]{zadrozny2004learning}
Zadrozny, B.
\newblock Learning and evaluating classifiers under sample selection bias.
\newblock In \emph{Proceedings of the twenty-first international conference on
  Machine learning}, pp.\  114, 2004.

\bibitem[Zhang et~al.(2013)Zhang, Sch{\"o}lkopf, Muandet, and
  Wang]{zhang2013domain}
Zhang, K., Sch{\"o}lkopf, B., Muandet, K., and Wang, Z.
\newblock Domain adaptation under target and conditional shift.
\newblock In \emph{International Conference on Machine Learning}, pp.\
  819--827. PMLR, 2013.

\bibitem[Zhang et~al.(2020)Zhang, Gong, Stojanov, Huang, LIU, and
  Glymour]{NEURIPS2020_3430095c}
Zhang, K., Gong, M., Stojanov, P., Huang, B., LIU, Q., and Glymour, C.
\newblock Domain adaptation as a problem of inference on graphical models.
\newblock In Larochelle, H., Ranzato, M., Hadsell, R., Balcan, M.~F., and Lin,
  H. (eds.), \emph{Advances in Neural Information Processing Systems},
  volume~33, pp.\  4965--4976. Curran Associates, Inc., 2020.

\bibitem[Zhang et~al.(2019)Zhang, Liu, Long, and Jordan]{zhang2019bridging}
Zhang, Y., Liu, T., Long, M., and Jordan, M.
\newblock Bridging theory and algorithm for domain adaptation.
\newblock In \emph{International Conference on Machine Learning}, pp.\
  7404--7413. PMLR, 2019.

\bibitem[Zhou et~al.(2021)Zhou, Liu, Qiao, Xiang, and Loy]{zhou2021domain}
Zhou, K., Liu, Z., Qiao, Y., Xiang, T., and Loy, C.~C.
\newblock Domain generalization in vision: A survey.
\newblock \emph{arXiv preprint arXiv:2103.02503}, 2021.

\end{thebibliography}
